\documentclass[12pt]{article}
\usepackage{amsmath}
\usepackage{times}
\usepackage{graphicx}
\usepackage{color}
\usepackage{multirow}
\usepackage{cite}
\usepackage{rotating}
\usepackage{bbm}
\usepackage{latexsym}
\usepackage{url}

\usepackage{bm}
\usepackage{amssymb}
\usepackage{amsfonts}
\usepackage{amsthm}

\newtheorem{theorem}{Theorem}
\newtheorem{lemma}{Lemma}
\newtheorem{remark}{Remark}
\newtheorem{corollary}{Corollary}

\usepackage{algorithm,algorithmic}
\usepackage{setspace}

\usepackage{mcr}

\newcommand{\nold}{n^\mathrm{old}}
\newcommand{\nnew}{n^\mathrm{new}}

\newcommand{\dold}{d^\mathrm{old}}
\newcommand{\dnew}{d^\mathrm{new}}
\newcommand{\Xtest}{\bm{x}^\mathrm{test}}
\newcommand{\XtestT}{\bm{x}^\mathrm{test\top}}
\newcommand{\xtest}[1]{x^\mathrm{test}_{#1}}
\newcommand{\Xold}{X^\mathrm{old}}
\newcommand{\XoldT}{X^{\mathrm{old}\top}}
\newcommand{\Xnew}{X^\mathrm{new}}
\newcommand{\XnewT}{X^{\mathrm{new}\top}}

\newcommand{\Yold}{\bm{y}^\mathrm{old}}
\newcommand{\Ynew}{\bm{y}^\mathrm{new}}

\newcommand{\Wold}{\bm{w}^\mathrm{*old}}
\newcommand{\Woldh}{\hat{\bm{w}}^\mathrm{*old}}
\newcommand{\Wnew}{\bm{w}^\mathrm{*new}}
\newcommand{\wold}[1]{w^\mathrm{*old}_{#1}}
\newcommand{\woldh}[1]{\hat{w}^\mathrm{*old}_{#1}}
\newcommand{\wnew}[1]{w^\mathrm{*new}_{#1}}

\newcommand{\Alphaold}{\bm{\alpha}^\mathrm{*old}}
\newcommand{\Alphaoldh}{\hat{\bm{\alpha}}^\mathrm{*old}}
\newcommand{\Alphanew}{\bm{\alpha}^\mathrm{*new}}
\newcommand{\alphaold}[1]{\alpha^\mathrm{*old}_{#1}}
\newcommand{\alphaoldh}[1]{\hat{\alpha}^\mathrm{*old}_{#1}}
\newcommand{\alphanew}[1]{\alpha^\mathrm{*new}_{#1}}
\newcommand{\Prim}{\mathrm{P}}
\newcommand{\PrimL}[1]{\mathrm{PL}_{#1}}
\newcommand{\PrimR}{\mathrm{PR}}
\newcommand{\Dual}{\mathrm{D}}
\newcommand{\Pold}{\mathrm{P}_\mathrm{old}}
\newcommand{\Dold}{\mathrm{D}_\mathrm{old}}
\newcommand{\Pnew}{\mathrm{P}_\mathrm{new}}
\newcommand{\Dnew}{\mathrm{D}_\mathrm{new}}
\newcommand{\Gnew}{\mathrm{G}_\mathrm{new}}

\newcommand{\MinLin}[2]{{\rm MinLin}[{#1},{#2}]}
\newcommand{\MaxLin}[2]{{\rm MaxLin}[{#1},{#2}]}
\newcommand{\argmax}{\mathop{\rm argmax}\limits}
\newcommand{\argmin}{\mathop{\rm argmin}\limits}

\newcommand{\targmin}{{\rm argmin}}

\newcommand{\minpartial}{\underline{\partial}}
\newcommand{\maxpartial}{\overline{\partial}}
\newcommand{\minranpartial}[1]{\underline{\underline{\partial{#1}}}}
\newcommand{\maxranpartial}[1]{\overline{\overline{\partial{#1}}}}

\newcommand{\ccell}[1]{\multicolumn{1}{c}{#1}} 

\newcommand{\captionfonts}{\normalsize}

\makeatletter  
\long\def\@makecaption#1#2{%
  \vskip\abovecaptionskip
  \sbox\@tempboxa{{\captionfonts #1: #2}}%
  \ifdim \wd\@tempboxa >\hsize
    {\captionfonts #1: #2\par}
  \else
    \hbox to\hsize{\hfil\box\@tempboxa\hfil}%
  \fi
  \vskip\belowcaptionskip}
\makeatother   

\begin{document}
\hspace{13.9cm}1

\ \vspace{20mm}\\

{\LARGE Generalized Low-Rank Update: Model Parameter Bounds for Low-Rank Training Data Modifications}

\ \\
{\bf \large
	Hiroyuki Hanada$^{\displaystyle 1}$,
	Noriaki Hashimoto$^{\displaystyle 1}$,
	\\
	Kouichi Taji$^{\displaystyle 2}$,
	Ichiro Takeuchi$^{\displaystyle 2, \displaystyle 1, \dagger}$
}\\
{$^{\displaystyle 1}$ Center for Advanced Intelligence Project, RIKEN, Tokyo, Japan, 103-0027.}\\
{$^{\displaystyle 2}$ Department of Mechanical Systems Engineering, Nagoya University, Nagoya, Japan, 464-8603.}\\
{$^{\dagger}$ Corresponding author. e-mail: ichiro.takeuchi@mae.nagoya-u.ac.jp}\\
%


\thispagestyle{empty}
\markboth{}{NC instructions}
\ \vspace{-0mm}\\
%
\begin{center} {\bf Abstract} \end{center}

In this study, we have developed an incremental machine learning (ML) method that efficiently obtains the optimal model when a small number of instances or features are added or removed.
This problem holds practical importance in model selection, such as cross-validation (CV) and feature selection.
Among the class of ML methods known as linear estimators, there exists an efficient model update framework called the \emph{low-rank update} that can effectively handle changes in a small number of rows and columns within the data matrix.
However, for ML methods beyond linear estimators, there is currently no comprehensive framework available to obtain knowledge about the updated solution within a specific computational complexity.
In light of this, our study introduces a method called the \emph{Generalized Low-Rank Update (GLRU)} which extends the low-rank update framework of linear estimators to ML methods formulated as a certain class of regularized empirical risk minimization, including commonly used methods such as SVM and logistic regression.
The proposed GLRU method not only expands the range of its applicability but also provides information about the updated solutions with a computational complexity proportional to the amount of dataset changes.
To demonstrate the effectiveness of the GLRU method, we conduct experiments showcasing its efficiency in performing cross-validation and feature selection compared to other baseline methods.

\clearpage

\section{Introduction} \label{sec:introduction}

In the selection of machine learning (ML) models, it is often necessary to retrain the model parameters multiple times for slightly different datasets.
However, retraining the model from scratch whenever there are minor changes in the dataset is computationally expensive.
To address this, several approaches have been explored to achieve efficient model updates across different problems.
In this study, we focus on the process of adding or removing instances or features from a training dataset, which is known as \emph{incremental learning}.

Within a simple class of ML methods known as \emph{linear estimators}, the techniques for updating the model when instances or features are added or removed are well-established~\cite{orr1996radial,Nocedal99,Boyd04a}.
This updating technique is commonly referred to as a \emph{low-rank update} since it involves updating a small number of rows and/or columns in the matrix used for the computation of the linear estimator.
Although this low-rank update framework is significantly more efficient than retraining the model from scratch, the class of ML methods to which these efficient updates are applicable is limited to linear estimators.
In this study, we introduce a method called \emph{Generalized Low-Rank Update (GLRU)} as a means to generalize the low-rank update framework for a wider range of ML methods.

To provide a concrete explanation, let us introduce some notations.
Consider a training dataset denoted as $(X, \bm y)$ for a supervised learning problem, such as regression and classification, with $n$ instances and $d$ features.
Here, $X$ represents an $n \times d$ input matrix, where the $(i, j)$-th element representes the value of the $j$-th feature in the $i$-th instance, while $\bm y$ is an $n$-dimensional vector, where the $i$-th element represents the output value of the $i$-th instance.
Furthermore, let us denote the training datasets \emph{before} and \emph{after} the update as $(X^{\rm old}, \bm y^{\rm old})$ and $(X^{\rm new}, \bm y^{\rm new})$, respectively, where the former has $n^{\rm old}$ instances and $d^{\rm old}$ features, while the latter has $n^{\rm new}$ instances and $d^{\rm new}$ features. 
Moreover, we denote the optimal model parameters for the old and the new datasets as $\bm w^{*{\rm old}}$ and $\bm w^{*{\rm new}}$, respectively.

To provide a concrete explanation, let us introduce some notations.
Consider a training dataset denoted as $(X, \bm y)$ for a supervised learning problem, such as regression and classification, with $n$ instances and $d$ features.
Here, $X$ is an $n \times d$ input matrix, where the element at position $(i, j)$ represents the value of the $j$-th feature in the $i$-th instance, while $\bm y$ is an $n$-dimensional vector, where the element at index $i$ represents the output value of the $i$-th instance.
Furthermore, let us denote the training datasets \emph{before} and \emph{after} the update as $(X^{\rm old}, \bm y^{\rm old})$ and $(X^{\rm new}, \bm y^{\rm new})$, respectively, where the former has $n^{\rm old}$ instances and $d^{\rm old}$ features, while the latter has $n^{\rm new}$ instances and $d^{\rm new}$ features.
Moreover, we denote the optimal model parameters for the old and the new datasets as $\bm w^{*{\rm old}}$ and $\bm w^{*{\rm new}}$, respectively.

With these notations, incremental learning refers to the problem of efficiently obtaining $\bm w^{*{\rm new}}$ from $\bm w^{*{\rm old}}$ when the new dataset $(X^{\rm new}, \bm y^{\rm new})$ is obtained by adding or removing a small number of instances and/or features from the original dataset $(X^{\rm old}, \bm y^{\rm old})$.
To discuss computational complexity, let us consider an example scenario where $\Delta n$ instances are added or removed, while the number of features remains the same ($d = d^{\rm old} = d^{\rm new}$).
In the case of the aforementioned linear estimator, it is known that the computational cost required for this incremental learning problem is $\cO(d^2 \Delta n)$, which is significantly smaller compared to the computational cost of updating from scratch, which is $\cO(nd^2 + d^3)$.
Similarly, considering the case where $\Delta d$ features are added or removed, while the number of instances remains unchanged ($n = n^{\rm old} = n^{\rm new}$), the computational cost of updating a linear estimator is $\cO(n d \Delta d)$, which is significantly smaller compared to the individual costs of updating from scratch, which is $\cO(n d^2 + d^3)$.

The main contributions of this study are to generalize the low-rank update framework of linear estimator in two key aspects.
The first contribution is to broaden the applicability of efficient row-rank update computation to a wider range of ML methods.
Specifically, the GLRU method is applicable to ML methods that can be formulated as regularized empirical risk minimization, where the loss and regularization are represented as a certain class of convex functions (see Section \ref{sec:learning} for the details).
This class includes several commonly used ML methods, such as Support Vector Machine (SVM) and Logistic Regression (LR), to which the aforementioned row-rank update technique for linear estimators cannot be applied.
The second contribution is to provide information about the updated optimal solution with the minimum computational cost.
Here, the minimum computational cost refers to the computational complexity proportional to the number of elements that differ between $X^{\rm old}$ and $X^{\rm new}$.
For example, if $\Delta n$ instances are added/removed as discussed above, the minimum computational complexity is $\mathcal{O}(d \Delta n)$, and if $\Delta d$ features are added/removed, it becomes $\mathcal{O}(n \Delta d)$.
While the GLRU method has advantages in terms of applicability and computational complexity, it cannot provide the updated optimal solution $\bm w^{*{\rm new}}$ itself --- instead, it provides the range of the updated optimal solution $\cW$ so that $\Wnew\in\cW$.
%
Figure~\ref{fg:GLRU-concept} schematically illustrates how the GLRU method provides information on the updated optimal solution.

While some readers may argue that knowing only the range of the updated optimal solution is insufficient, in practical machine learning (ML) problems, having knowledge of the optimal solution's range alone can enable valuable decision-making.
For example, in the case of leave-one-out CV (LOOCV), let the dataset that includes all training instances be $(X^{\rm old}, \bm y^{\rm old})$, while let the new dataset obtained by removing the $i$-th instance be $(X^{\rm new}, \bm y^{\rm new})$.
In this scenario, the GLRU method indicates that the optimal solution $\bm w^{*{\rm new}}$ is located somewhere in $\cW$ with a computational complexity of $\cO(d)$.
Using the $\cW$, it is possible to calculate the lower and upper bounds of $\bm x_i^\top \bm w^{*{\rm new}}$ for the $i$-th training instance. 
This means that in a binary classification problem with a threshold of zero, if the lower bound is greater than zero, the instance can be classified as the positive class, while if the upper bound is smaller than zero, it can be classified as the negative class.
By repeating this process for all instances, we can obtain the upper and lower bounds of the LOOCV error, which is useful for decision-making in model selection.
%
We can also enjoy similar advantages in feature selection problems (see Section \ref{sec:GLRU-stepwise} for the details).
Figure \ref{fg:GLRU-LOOCV} schematically illustrates how the GLRU method efficiently performs LOOCV.

This paper is an expanded version of a conference paper \cite{Hanada2018AAAI-Modification}.
The conference paper focused on cases where the size of the training data matrix remained the same, but specific elements changed.
Therefore, with the method described in the conference paper, we need to adopt an alternative approach to handle changes in size, for example, such as filling a row or a column of $X$ with zeros.
In this paper, we specifically consider cases where the dataset size changes due to additions or removals of instances or features.
Additionally, this paper addresses the problem of low-rank updates in ML methods, which is an important practical issue.
To our knowledge, the low-rank update framework in incremental learning settings has mainly been limited to linear estimators.
The main contribution of this extended version is to expand the scope of the row-rank update framework in a more general and comprehensive way.
We also provide a theoretical discussion on the computational complexity of adding or removing instances and/or features for both primal and dual optimal solutions.
These complexity analyses allow us to confirm the theoretical and practical advantages of the GLRU method over several existing approaches developed for specific settings and models (refer to the related works section for details and Section \ref{sec:experiment} for numerical comparisons).

\begin{figure}[tp]
\begin{center}
\includegraphics[width=0.7\hsize]{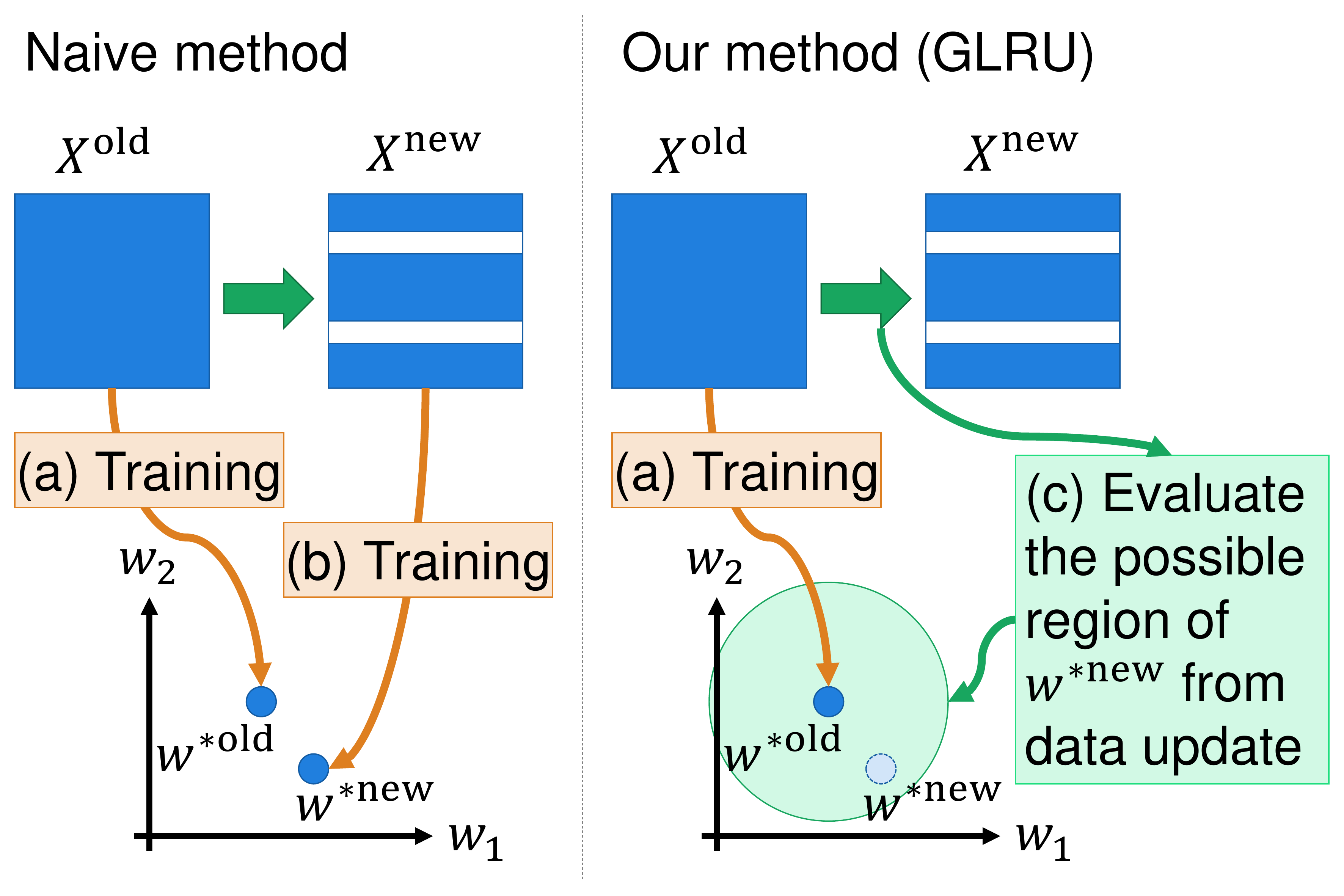}
\end{center}
\caption{
Concept of GLRU.
(a) Suppose that we have already trained the model parameter $\Wold$ for the original training set $\Xold$.
(b) However, even if $\Wold$ is known, computing the model parameter $\Wnew$ for the modified training set $\Xnew$ is costly in general.
(c) GLRU instead aims to evaluate the possible region of $\Wnew$ from $\Wold$ and the data modification, which can be computed in cost only depending on the size of the modification.}
\label{fg:GLRU-concept}
\end{figure}

\begin{figure}[tp]
\begin{center}
\includegraphics[width=0.7\hsize]{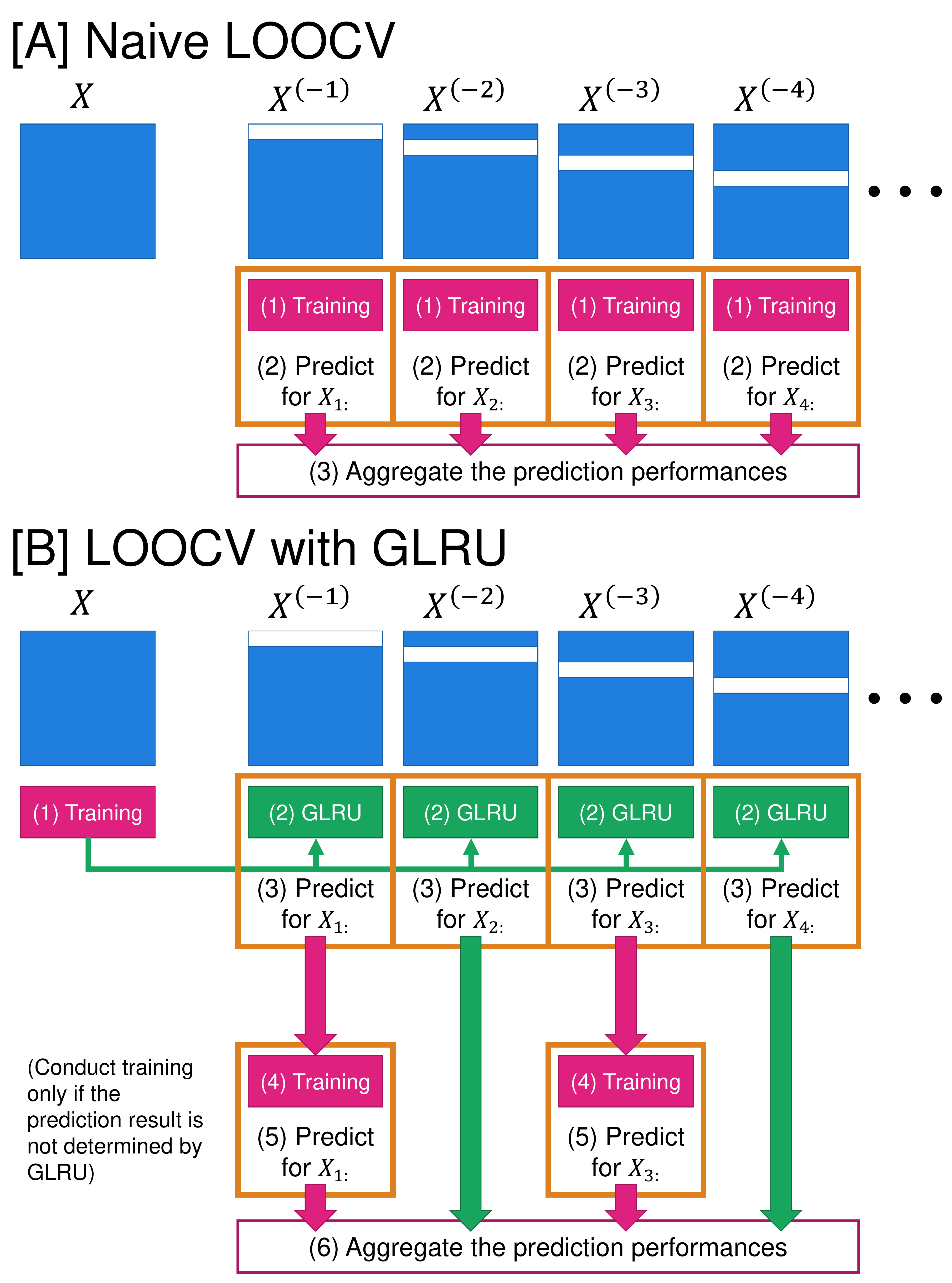}
\end{center}
\caption{Concept of LOOCV with GLRU compared to the naive one. $X_{i:}$ denotes the $i^\mathrm{th}$ row (instance) of $X$, and $X^{(-i)}$ denotes the dataset of $i^\mathrm{th}$ instance being removed. [A] With naive LOOCV, we need to conduct training for $n$ times, where $n$ is the number of instances. [B] However, with GLRU, we conduct training computations only when the prediction result is not determined by GLRU.}
\label{fg:GLRU-LOOCV}
\end{figure}

\subsection{Related works} \label{sec:related}

In general, methods for re-training the model parameter when the training dataset can be modified is referred to as {\em incremental learning}~\cite{solomonoff1989incremental,gepperth2016incremental}.
For complex models such as deep neural network, incremental learning simply refers to fine-tuning the model on a new dataset using the previous model as the initial value for optimization. However, the computational cost of incremental learning through fine-tuning is unknown until the actual retraining process is conducted, and the theoretical properties related to computational complexity have not been explored.
To our knowledge, computational complexity analysis of incremantal learning is limited to the class of linear estimators.
In the case of linear estimator, incremental learning is conducted in the framework called low-rank update~\cite{orr1996radial,hager1989updating,pan1990least,davis2005row}.
In the context of machine learning, the low-rank update of a linear estimator is applied to enhance the efficiency of cross-validation, and various specialized extensions for specific problems are also being performed.
Furthermore, in addition to the CV, low-rank update framework is also used in other issues such as label assignment in semi-supervised learning~\cite{gong2017label} and dictionary learning~\cite{wohlberg2016convolutional}.

Approximate methods are commonly employed to efficiently update ML models when low-rank update frameworks are not available.
For instance, in learning algorithms that require iterative optimization, \cite{bertsekas2015incremental} proposed the use of approximate solutions obtained through fewer iterations, assuming that the old and new models do not significantly differ.
Another example is the algorithm proposed by \cite{schlimmer1986case} for approximate incremental learning of decision trees, where only the tree nodes close to the updated instance are updated instead of updating all the tree nodes.
However, in our proposed GLRU method, we can provide a range within which an optimal solution may exist.
This enables decision-making that does not have to rely on approximation.

Fast methods for cross-validation and stepwise feature selection are also valuable topics.
However, most fast cross-validation methods employ approximate methods, that is, the results may differ slightly from those computed exactly.
A well-known method is to apply the Newton's method for only one step; see, for example, \cite{giordano2019swiss,rad2020scalable}.
Note that such methods may not be fast for datasets with large number of features, since the Newton's method requires the hessian (discussed in the experiment in Section \ref{sec:exp-loocv}).
There have been limited number of fast and exact cross-validation methods \cite{okumura2015quick}, and cross-validation with GLRU can also provide exact results.
Fast methods for stepwise feature selection have been less discussed; we found some fast methods with approximations (e.g., AIC-/BIC-based method \cite{an1989stepwise}, p-value-based method \cite{lin2011vif} and a method using the nearest-neighbor search of features \cite{zogalasiudem2014fast}), however, we could not find fast and non-approximate methods.

In contrast, GLRU can provide non-approximate results even for stepwise feature selections.
GLRU borrows techniques to provide a possible region of model parameters originally developed for {\em safe screening} methods~\cite{ghaoui2012safe,wang2013lasso,wang2013scaling,ogawa2013safe,liu2014safe,wang2014safe,xiang2017screening,fercoq2015mind,ndiaye2015gap,Zimmert2015,shibagaki2015regularization,nakagawa2016safe,shibagaki2016simultaneous}.
Safe screening aims to detect and remove model parameters that are surely zero at the optimum when the model parameter is expected to be sparse (containing many zeros) for fast computation.
Similar to the proposed method, some methods that provide a possible region of model parameters have been proposed based on the safe screening method \cite{okumura2015quick,gabel2015monitoring}.

\section{Problem setups}

\subsection{Definitions} \label{sec:definitions}

Let $\mathbb{R}$ and $\mathbb{N}$ be the set of all real numbers and positive integers, respectively.
Moreover, for $k\in\mathbb{N}$, we define $[k] := \{1, 2, \dots, k\}$.
For the two sets $A$ and $B$, $A\equiv B$ (resp. $A\not\equiv B$) denotes that $A$ and $B$ are identical (resp. not identical) set.

In this paper, a matrix is denoted by a capital letter. Given a matrix, its submatrix is denoted by accompanying indices (here we use ``:'' for ``all elements of the dimension''), while its element is denoted by corresponding {\em small} letter accompanied by indices. For example, for a matrix $X$, $X_{(\{2,3\}):}$ denotes the submatrix of $X$ consisting of the second and the third row of all columns. Additionally, $x_{24}$ denotes the element of $X$ in the second row of the fourth column.

In this study we often use the following minimization:
\begin{align*}
\MinLin{\bm{a}}{\bm{b}}(\bm{c}) := \min_{\bm{v}\in\mathbb{R}^k,~a_j\leq v_j\leq b_j} \bm{c}^\top\bm{v},
\end{align*}
where $\bm{a}, \bm{b}, \bm{c}\in\mathbb{R}^k$. This minimization can be easily solved as
\begin{align*}
\MinLin{\bm{a}}{\bm{b}}(\bm{c}) = \sum_{j:~c_j < 0} c_j b_j + \sum_{j:~c_j > 0} c_j a_j.
\end{align*}
Similarly, the corresponding maximization is given by
\begin{align*}
\MaxLin{\bm{a}}{\bm{b}}(\bm{c}) & := \max_{\bm{v}\in\mathbb{R}^k,~a_j\leq v_j\leq b_j} \bm{c}^\top\bm{v}
	= \sum_{j:~c_j < 0} c_j a_j + \sum_{j:~c_j > 0} c_j b_j.
\end{align*}

\subsection{Learning algorithm: Regularized empirical risk minimization} \label{sec:learning}

In this study, we consider the {\em regularized empirical risk minimization} (regularized ERM) for linear predictions.

We consider the following linear prediction:
For a $d$-dimensional input $\bm{x}\in\mathbb{R}^d$,
its scalar outcome $y\in\mathbb{R}$
and a {\em loss function} $\ell_y: \mathbb{R}\to\mathbb{R}$,
we assume that there exists a parameter vector $\bm{w}^*\in\mathbb{R}^d$
such that $\ell_y(\bm{x}^\top\bm{w}^*)$ is small for any pair of $(\bm{x}, y)$.

We then estimate $\bm{w}^*$ from an $n$-instance training dataset of inputs $X\in\mathbb{R}^{n\times d}$ and outcomes $\bm{y}\in\mathbb{R}^n$.
The regularized ERM estimates $\bm{w}^*$ as the optimal parameter for the following problem:
\begin{align}
& \bm{w}^* := \argmin_{\bm{w}\in\mathbb{R}^d} \Prim_X(\bm{w}),
	\quad\text{where}\quad
	\Prim_X(\bm{w}) := \frac{1}{n}\sum_{i\in[n]} \ell_{y_i}(X_{i:}\bm{w}) + \sum_{j\in[d]} \rho_j(w_j).
	\label{eq:primal}
\end{align}
Here, $\rho_j: \mathbb{R}\to\mathbb{R}$ is a {\em regularization function}\footnote{
	This formulation assumes that the regularization function
	must have the form $\sum_{j\in[d]} \rho_j(w_j)$, that is,
	must be defined as the sum of the values separately computed for $d$ features.
	The following discussions do not consider regularization functions that cannot be
	defined as the form above, e.g., fused lasso regularization
	$\kappa\sum_{j\in[d-1]} |w_j - w_{j+1}|$ \cite{tibshirani2005sparsity}.
} which imposes a ``good'' property on the model parameters such as stability and sparsity.
Common examples of loss and regularization functions are presented in Appendix \ref{app:losses-regularizations}.
In this study, we assume that $\ell_y$ and $\rho_j$ are both convex and continuous.
(As a consequence, $\Prim_X(\bm{w})$ is also convex and continuous with respect to $\bm{w}$.)

In this paper, as we discuss the modifications of training dataset, we introduce additional notations.
Under the same definitions of $\Xold$, $\Yold$, $\Xnew$ and $\Ynew$ as in Section \ref{sec:introduction},
we write $\Pold := \Prim_{\Xold}$ and $\Pnew := \Prim_{\Xnew}$ in short.
Then $\Wold$ and $\Wnew$ in Section \ref{sec:introduction} are defined as:
\begin{align*}
\Wold := \argmin_{\bm{w}\in\mathbb{R}^{\dold}} \Pold(\bm{w}),
\quad
\Wnew := \argmin_{\bm{w}\in\mathbb{R}^{\dnew}} \Pnew(\bm{w}).
\end{align*}

\section{Mathematical backgrounds}

\subsection{Convex optimization} \label{sec:convex}

Unless otherwise specified, we consider minimizations of convex functions
$f: \mathbb{R}^k\to\mathbb{R}\cup\{+\infty\}$.\footnote{
	Another common formulation of convex functions is that the domain is a {\em convex subset} of
	$\mathbb{R}^k$ and the range is a subset of $\mathbb{R}$ ($+\infty$ cannot be included).
	However, for simplicity of treating convex conjugates as explained later,
	we adopted the formulation here.
}
For function $f: \mathbb{R}^k\to\mathbb{R}\cup\{+\infty\}$,
we define the domain $\mathrm{dom}f := \{ \bm{v} \mid \bm{v}\in\mathbb{R}^k, f(\bm{v}) < +\infty \}$
as the ``feasible'' set for $f$, that is, the set when we consider minimizing $f$.

For function $f: \mathbb{R}^k\to\mathbb{R}\cup\{+\infty\}$,
we denote the gradient at $\bm{v}$ by $\nabla f(\bm{v})$.
For convex function $f: \mathbb{R}^k\to\mathbb{R}\cup\{+\infty\}$,
its {\em subgradient} at point $\bm{v}\in\mathbb{R}^k$ is defined as
\begin{align}
& \partial f(\bm{v}) := \{ \bm{g} \mid \bm{g}\in\mathbb{R}^k,~\forall \bm{v}^\prime\in\mathbb{R}^k:
	f(\bm{v}^\prime) - f(\bm{v}) \geq \bm{g}^\top(\bm{v}^\prime - \bm{v}) \}.
\label{eq:subgradient}
\end{align}
Note that, if $f$ is convex and $\nabla f(\bm{v})$ exists,
then $\partial f(\bm{v}) = \{ \nabla f(\bm{v}) \}$.
For example, if $f(t) = |t|$, $\nabla f(0)$ does not exist, whereas $\partial f(0) = \{ s \mid -1 \leq s \leq 1 \}$.

For any univariate convex function $f: \mathbb{R}\to\mathbb{R}\cup\{+\infty\}$,
it is known that $\partial f(t)$ must be a closed interval if it exists\footnote{
	In the multivariate case, $\partial f(\bm{v})$ must be a closed convex set if it exists \cite{rockafellar1970convex}.
}.
Thus, we define
$\minpartial f(t) := \inf \{ s \mid s\in\partial f(t) \}$ and
$\maxpartial f(t) := \sup \{ s \mid s\in\partial f(t) \}$.
Moreover, we define
$\minranpartial{f} := \inf_{t\in\mathbb{R}} \minpartial f(t)$ and
$\maxranpartial{f} := \sup_{t\in\mathbb{R}} \maxpartial f(t)$.

For convex function $f: \mathbb{R}^k\to\mathbb{R}\cup\{+\infty\}$,
its {\em convex conjugate} $f^*(\bm{v}): \mathbb{R}^k\to\mathbb{R}\cup\{+\infty\}$ is defined as follows:
\begin{align*}
f^*(\bm{v}) := \sup_{\bm{v}^\prime\in\mathbb{R}^k} [ \bm{v}^\top\bm{v}^\prime - f(\bm{v}^\prime) ].
\end{align*}
For univariate convex function $f: \mathbb{R}\to\mathbb{R}\cup\{+\infty\}$,
\begin{align}
\mathrm{dom}f^* = [\minranpartial{f}, \maxranpartial{f}]
\label{eq:conjugate-finite}
\end{align}
holds \cite{rockafellar1970convex}.
Furthermore, it is known that, if $f$ is convex and lower semi-continuous, then $f = f^{**}$ holds (Fenchel-Moreau theorem).

Function $f: \mathbb{R}^k\to\mathbb{R}\cup\{+\infty\}$ is called $\lambda$-{\em strongly convex} ($\lambda > 0$) if
\begin{align}
& \forall \bm{u}\in\mathbb{R}^k, \forall \bm{v}\in\mathrm{dom}f, \forall \bm{g}\in \partial f(\bm{v}): \nonumber\\
& f(\bm{u}) - f(\bm{v}) \geq \bm{g}^\top (\bm{u} - \bm{v}) + \frac{\lambda}{2}\|\bm{u} - \bm{v}\|_2^2. \label{eq:strongly-convex}
\end{align}
A function $f: \mathbb{R}^k\to\mathbb{R}$ is called $\nu$-{\em smooth} ($\nu > 0$) if its gradient exists in its domain and is $\nu$-Lipschitz continuous with respect to the L2-norm, that is,
\begin{align*}
\forall \bm{u}, \bm{v}\in\mathrm{dom}f:
\quad
\|\nabla f(\bm{u}) - \nabla f(\bm{v})\|_2 \leq \nu\|\bm{u} - \bm{v}\|_2.
\end{align*}
The smoothness turns into the strong convexity after taking the convex conjugate, and vice versa.
Specifically, for any convex function $f$, $f^*$ is $(1/\nu)$-strongly convex if $f$ is $\nu$-smooth;
whereas $f^*$ is $(1/\lambda)$-smooth if $f$ is $\lambda$-strongly convex,
under certain regularity conditions (see Section X.4.2 of \cite{hiriart1993convex}).

\subsection{Dual problem of regularized ERM} \label{sec:dual-problem}

To identify the possible region of the model parameters after data modification,
we use the {\em dual problem} of \eqref{eq:primal}.
This is derived from {\em Fenchel's duality theorem} in the literature on convex optimization \cite{rockafellar1970convex}.

For notational simplicity, we rewrite $\Prim_X$ defined in \eqref{eq:primal}
using $\PrimL{n}(\bm{u}): \mathbb{R}^n\to\mathbb{R}$ and $\PrimR(\bm{v}): \mathbb{R}^d\to\mathbb{R}$ as
\begin{align*}
& \Prim_X(\bm{w}) = \PrimL{n}(X\bm{w}) + \PrimR(\bm{w}),\quad\text{where} \\
& \PrimL{n}(\bm{u}) := \frac{1}{n}\sum_{i\in[n]} \ell_{y_i}(u_i),
	\quad
	\PrimR(\bm{v}) := \sum_{j\in[d]} \rho_j(v_j).
\end{align*}

First, note that the convex conjugates of $\PrimL{n}$ and $\PrimR$ are computed as follows
(the details are presented in Appendix \ref{app:primal-part-conjugate}):
\begin{align*}
\PrimL{n}^*(\bm{u}) = \frac{1}{n}\sum_{i\in[n]} \ell^*_{y_i}(n u_i),
\quad
\PrimR^*(\bm{v}) = \sum_{j\in[d]} \rho_j^*(v_j).
\end{align*}
Then, the {\em dual problem} of \eqref{eq:primal} is defined as follows:
\begin{align}
\bm{\alpha}^* :=& \argmax_{\bm{\alpha}\in\mathbb{R}^n} \Dual_X(\bm{\alpha}),
	\quad\text{where} \nonumber \\
\Dual_X(\bm{\alpha}) &:= -\PrimL{n}^*\Bigl(-\frac{1}{n}\bm{\alpha}\Bigr) - \PrimR^*\Bigl(\frac{1}{n}X^\top\bm{\alpha}\Bigr) \nonumber \\
	& = - \frac{1}{n} \sum_{i\in[n]} \ell^*_{y_i}(- \alpha_i)
		- \sum_{j\in[d]} \rho_j^*\Bigl( \frac{1}{n}X_{:j}^\top\bm{\alpha} \Bigr).
	\label{eq:dual}
\end{align}
Fenchel's duality theorem provides the following relationship:
\begin{align}
& \Prim_X(\bm{w}^*) = \Dual_X(\bm{\alpha}^*), \label{eq:strong-duality}\\
& \forall j\in[d]:\quad w^*_j \in \partial\rho_j^*\Bigl(\frac{1}{n}X_{:j}^\top\bm{\alpha}^*\Bigr), \label{eq:KKT-dual2primal}\\
& \forall i\in[n]:\quad -\alpha^*_i \in \partial\ell_{y_i}(X_{i:}\bm{w}^*). \label{eq:KKT-primal2dual}
\end{align}
The relationship \eqref{eq:strong-duality} is known as {\em strong duality}\footnote{
	The strong duality does not hold in general, however, it does hold for this setup
	because we assume that $\ell_y$ and $\rho_j$ are finite, convex, and continuous.
	The detailed condition under which a strong duality holds is discussed in \cite{rockafellar1970convex}.
},
whereas \eqref{eq:KKT-dual2primal} and \eqref{eq:KKT-primal2dual} as the {\em Karush-Kuhn-Tucker condition}.

For any $\bm{w}\in\mathbb{R}^d$ and $\bm{\alpha}\in\mathbb{R}^n$, we refer to $\Prim_X(\bm{w}) - \Dual_X(\bm{\alpha})$ as the {\em duality gap} of $\Prim_X$ and $\Dual_X$ for $\bm{w}$ and $\bm{\alpha}$.
According to \eqref{eq:strong-duality}, the duality gap must be nonnegative because $\Prim_X(\bm{w}^*)$ is the minimum of $\Prim_X$ whereas $\Dual_X(\bm{\alpha}^*)$ is the maximum of $\Dual_X$.
This means that, as $\bm{w}$ and $\bm{\alpha}$ approach $\bm{w}^*$ and $\bm{\alpha}^*$, respectively, the duality gap approaches zero. 

We define $\Dold$, $\Dnew$, $\Alphaold$ and $\Alphanew$ as like in Section \ref{sec:learning}, that is,
\begin{align*}
& \Dold := \Dual_{\Xold}, \qquad \Dnew := \Dual_{\Xnew}, \\
& \Alphaold := \argmax_{\bm{\alpha}\in\mathbb{R}^{\nold}} \Dold(\bm{\alpha}),
	\qquad
	\Alphanew := \argmax_{\bm{\alpha}\in\mathbb{R}^{\nnew}} \Dnew(\bm{\alpha}).
\end{align*}

\section{Construction of GLRU for regularized ERM} \label{ch:GLRU-method}

To compute the region ${\cal W}$ such that $\Wnew\in{\cal W}$ (Section \ref{sec:introduction})
for regularized ERM,
first we present a method for identifying the possible regions of the (unknown) model parameters
for modified dataset
$\Wnew\in\mathbb{R}^{\dnew}$ and $\Alphanew\in\mathbb{R}^{\nnew}$, using other known parameters
$\hat{\bm{w}}\in\mathbb{R}^{\dnew}$ and $\hat{\bm{\alpha}}\in\mathbb{R}^{\nnew}$
explained in Section \ref{ch:bound-common}.
Then we present a method for calculating the region with computational cost
proportional to the size of the data modification,
which is achieved by plugging in $\hat{\bm{w}}\gets\Wold$ and $\hat{\bm{\alpha}}\gets\Alphaold$,
respectively, in Section \ref{ch:bound-fast}.

\subsection{Possible regions of $\Wnew$ and $\Alphanew$} \label{ch:bound-common}

\subsubsection{Main result}

First we describe how the possible regions of $\Wnew$ and $\Alphanew$ are obtained.
The following theorem is a refinement of \cite{Hanada2018AAAI-Modification}, however, we give the proof in Appendix \ref{app:proof-th:bounds} for completeness.
Similar discussions are found in
\cite{ndiaye2015gap}, 
\cite{Zimmert2015} 
and
\cite{shibagaki2016simultaneous}.

\begin{theorem} \label{th:bounds}
We define $\Pnew$ and $\Wnew$ as described in Section \ref{sec:learning},
and $\Dnew$ and $\Alphanew$ as described in Section \ref{sec:dual-problem}.
For any
$\hat{\bm{w}}\in\mathbb{R}^{\dnew}$ and
$\hat{\bm{\alpha}}\in\mathrm{dom}\Dnew\subseteq\mathbb{R}^{\nnew}$,
define
\begin{align*}
\Gnew(\hat{\bm{w}}, \hat{\bm{\alpha}}) := \Pnew(\hat{\bm{w}}) - \Dnew(\hat{\bm{\alpha}}).
\end{align*}
Then, we can identify the possible regions for the model parameters for a modified dataset $\Wnew$ or $\Alphanew$ as follows:
\begin{description}
\item [(i-P)]
	If $\rho_j(t)$ is $\lambda$-strongly convex ($\lambda > 0$) for all $j\in[d]$, then:
	\begin{align}
	& \|\Wnew - \hat{\bm{w}}\|_2 \leq r_P := \sqrt{\frac{2}{\lambda}\Gnew(\hat{\bm{w}}, \hat{\bm{\alpha}})}.
\label{eq:primal-bound}
	\end{align}
\item [(i-D)]
	Under the same assumption as in (i-P), we have
	\begin{align}
	& \forall i\in[\nnew]:\nonumber\\
	& \alphanew{i} \in [
		-\maxpartial \ell_{y_i}(\Xnew_{i:}\hat{\bm{w}} + r_P\|\Xnew_{i:}\|_2),
		-\minpartial \ell_{y_i}(\Xnew_{i:}\hat{\bm{w}} - r_P\|\Xnew_{i:}\|_2)]. \label{eq:primal-dual-bound}
	\end{align}
\item [(ii-D)]
	If $\ell_{y_i}(t)$ is $\mu$-smooth ($\mu > 0$), then:
	\begin{align}
	& \|\Alphanew - \hat{\bm{\alpha}}\|_2 \leq r_D := \sqrt{2\nnew\mu\Gnew(\hat{\bm{w}}, \hat{\bm{\alpha}})}.
		\label{eq:dual-bound}
	\end{align}
\item [(ii-P)]
	Under the same assumption as in (ii-D), we have
	\begin{align}
	& \forall j\in[\dnew]:
	\qquad
	\wnew{j} \in \Bigl[
		\minpartial\rho_j^*\Bigl(\frac{\underline{F}_j(\hat{\bm{\alpha}})}{\nnew} \Bigr), 
		\maxpartial\rho_j^*\Bigl(\frac{\overline{F}_j(\hat{\bm{\alpha}})}{\nnew} \Bigr) \Bigr],
			\label{eq:dual-primal-bound}
	\end{align}
	where $\underline{F}_j$ and $\overline{F}_j$ are defined as follows:
	\begin{align}
	\underline{F}_j(\hat{\bm{\alpha}}) := \XnewT_{:j}\hat{\bm{\alpha}} - r_D \|\XnewT_{:j}\|_2,
	\quad
	\overline{F}_j(\hat{\bm{\alpha}}) := \XnewT_{:j}\hat{\bm{\alpha}} + r_D \|\XnewT_{:j}\|_2.
	\label{eq:dual2primal-input-naive}
	\end{align}
\end{description}
\end{theorem}

\begin{remark} \label{rm:finite-subderiv-conj}
In (ii-P) of Theorem \ref{th:bounds}, the upper (resp. the lower) bound \eqref{eq:dual-primal-bound} may be a positive infinity (resp. a negative infinity).
In fact, for $j\in[d]$ and $t\in\mathbb{R}$ such that
\begin{align}
\rho_j^*(t) < +\infty ~~\text{and}~~ (-\infty\in\partial\rho_j^*(t)~~\text{or}~~+\infty\in\partial\rho_j^*(t)),
\label{eq:finite-subderiv-conj}
\end{align}
the bounds become infinite if $\underline{F}_j(\hat{\bm{\alpha}})$ or $\overline{F}_j(\hat{\bm{\alpha}})$ includes such $t$.
\begin{itemize}
\item In case $\rho_j(t)$ is $\lambda$-strongly convex ($\lambda > 0$) for all $j\in[d]$, then \eqref{eq:finite-subderiv-conj} cannot hold since $\rho_j^*(t)$ is a $(1/\lambda)$-smooth (Section \ref{sec:convex}) and therefore $\partial\rho_j^*(t)$ must be finite.
\item For the L1-regularization $\rho_j(t) = \kappa|t|$, the expression \eqref{eq:finite-subderiv-conj} holds for $t = \pm\kappa$: we have $\rho_j^*(\pm\kappa) = 0 < +\infty$ but $\partial\rho_j^*(\pm\kappa) = [0, +\infty]$.
\end{itemize}
\end{remark}

\begin{remark} \label{rm:dual2primal-input-feasible}
In (ii-P) of Theorem \ref{th:bounds}, $\underline{F}_j(\hat{\bm{\alpha}})$ and $\overline{F}_j(\hat{\bm{\alpha}})$
may be tighter by imposing the constraint $\hat{\bm{\alpha}}\in\mathrm{dom}\Dual_{\Xnew}$.
See Appendix \ref{app:dual2primal-input-feasible} for further details.
\end{remark}

\subsubsection{Bounds of prediction results from those of model parameters} \label{ch:bound-predict}

Given a test instance $\Xtest\in\mathbb{R}^{\dnew}$, we present the lower and upper bounds of $\XtestT\Wnew$ derived from the possible regions of $\Wnew$ or $\Alphanew$ in Theorem \ref{th:bounds}.
\begin{corollary} \label{co:bounds-predict}
For any $\Xtest\in\mathbb{R}^{\dnew}$,
\begin{enumerate}
\item {\em Primal Strong Convexity Bound (Primal-SCB)}: If the assumption (i-P) in Theorem \ref{th:bounds} holds,
	\begin{align}
	& \XtestT\hat{\bm{w}} - r_P \|\Xtest\|_2
	\leq
	\XtestT\Wnew
	\leq
	\XtestT\hat{\bm{w}} + r_P \|\Xtest\|_2 \label{eq:prediction-bound-primal}
	\end{align}
\item {\em Dual Strong Convexity Bound (Dual-SCB)}: If the assumption (ii-P) in Theorem \ref{th:bounds} holds,
\end{enumerate}
\begin{align}
	& \MinLin{\Bigl\{\frac{\minpartial\rho_j^*(\underline{F}_j(\hat{\bm{\alpha}}))}{\nnew}\Bigr\}_j}{\Bigl\{\frac{\maxpartial\rho_j^*(\overline{F}_j(\hat{\bm{\alpha}}))}{\nnew}\Bigr\}_j}(\Xtest) \nonumber\\
	\leq & \XtestT\Wnew \nonumber\\
	\leq & \MaxLin{\Bigl\{\frac{\minpartial\rho_j^*(\underline{F}_j(\hat{\bm{\alpha}}))}{\nnew}\Bigr\}_j}{\Bigl\{\frac{\maxpartial\rho_j^*(\overline{F}_j(\hat{\bm{\alpha}}))}{\nnew}\Bigr\}_j}(\Xtest). \label{eq:prediction-bound-dual}
\end{align}
\end{corollary}

\subsection{Fast computation of GLRU for data modifications} \label{ch:bound-fast}

Our aim is not only identifying the possible regions of $\Wnew$ and $\Alphanew$ from $\Wold$ and $\Alphaold$
but also computing them quickly, specifically, proportional to the size of the update of the dataset.
To do this, we show that the computational complexity of computing $\Gnew(\hat{\bm{w}}, \hat{\bm{\alpha}})$
is proportinal to the modified size of the dataset.

To reduce its cost, $\hat{\bm{w}}$ and $\hat{\bm{\alpha}}$ in Theorem \ref{th:bounds} should be derived from $\Wold$ and $\Alphaold$, respectively.

Additionally, when we compute $\Wold$ and $\Alphaold$,
we retain the following values computed together with $\Wold$ and $\Alphaold$:
\begin{subequations}
\label{eq:precomputed}
\begin{align}
& \Xold\Wold,\quad\XoldT\Alphaold, \label{sec:precomputed-ip}\\
& \PrimL{\nold}(X\Wold),\quad\PrimR(\Wold), \label{sec:precomputed-primal}\\
& \PrimL{\nold}^*\Bigl(-\frac{1}{\nold}\Alphaold\Bigr),\quad\PrimR^*\Bigl(\frac{1}{\nold}\XoldT\Alphaold\Bigr). \label{sec:precomputed-dual}
\end{align}
\end{subequations}

\begin{theorem} \label{th:bound-computations}
Under the same assumption as in Theorem \ref{th:bounds},
suppose that $\Wold$, $\Alphaold$, and the values \eqref{eq:precomputed} are already calculated. Then,
\begin{itemize}
\item If $(\nold - \nnew)$ instances are removed, then the bound \eqref{eq:bound-instance-removal} can be computed in $O((\nold - \nnew)d)$ time.
\item If $(\nnew - \nold)$ instances are added, then the bound \eqref{eq:bound-instance-addition} can be computed in $O((\nnew - \nold)d)$ time.
\item If $(\dold - \dnew)$ features are removed, then the bound \eqref{eq:bound-feature-removal} can be computed in $O((\dold - \dnew)n)$ time.
\item If $(\dnew - \dold)$ features are added, then the bound \eqref{eq:bound-feature-addition} can be computed in $O((\dnew - \dold)n)$ time.
\end{itemize}
\end{theorem}

The results are presented in the following sections, and the proofs are presented in Appendix \ref{app:computation-dualitygap}.

\subsubsection{Bound computations for instance removals/additions} \label{sec:bound-inst-change}

First we consider a case in which some instances are removed.
Without loss of generality, we assume that the last $(\nold-\nnew)$ instances are removed ($d:=\dnew=\dold$, $\nnew < \nold$, and $\Xold_{i:} = \Xnew_{i:}~\forall i\in[\nnew]$).
In this case, we simply set $\hat{\bm{w}}\gets\Wold$ in Theorem \ref{th:bounds},
whereas we construct $\hat{\bm{\alpha}}\in\mathbb{R}^{\nnew}$ as
$\hat{\bm{\alpha}}\gets\Alphaoldh := \{ \alphaold{i} \}_{i\in[\nnew]}$.
Then, using \eqref{eq:precomputed},
the duality gap $\Gnew(\Wold, \Alphaoldh)$ can be computed in $O((\nold-\nnew)d)$ time as follows:
\begin{align}
& \Gnew(\Wold, \Alphaoldh) := \Pnew(\Wold) - \Dnew(\Alphaoldh) \nonumber\\
= & \frac{1}{\nnew} \biggl[ -\sum_{i=\nnew+1}^{\nold} [\ell_{y_i}(\Xold_{i:}\Wold) + \ell^*_{y_i}(-\alphaold{i})] \nonumber\\
	& + \nold\PrimL{\nold}(\Xold\Wold) + \nnew\PrimR(\Wold) \nonumber\\
	& + \nold\PrimL{\nold}^*\Bigl( -\frac{1}{\nold}\Alphaold \Bigr) \nonumber\\
	& + \nnew \PrimR^*\Bigl(\frac{1}{\nnew}\Bigl[\XoldT\Alphaold - \sum_{i=\nnew+1}^{\nold}\alphaold{i}\Xold_{i:}\Bigr]\Bigr)
	\biggr].
	\label{eq:bound-instance-removal}
\end{align}

We then consider the case in which some instances are added ($d:=\dnew=\dold$, $\nnew > \nold$, and $\Xold_{i:} = \Xnew_{i:}~\forall i\in[\nold]$).
In this case, we set $\hat{\bm{w}}\gets\Wold$ in Theorem \ref{th:bounds},
whereas we construct $\hat{\bm{\alpha}}\in\mathbb{R}^{\nnew}$ as
\begin{align*}
& \hat{\bm{\alpha}}\gets\Alphaoldh, \quad\text{where} \\
& \alphaoldh{i} = \alphaold{i}, & (i\in[\nold]) \\
& -\alphaoldh{i} \in \partial\ell_{y_i}(\Xnew_{i:}\Wold). & (i\in[\nnew]\setminus[\nold])
\end{align*}
The latter is based on the KKT condition \eqref{eq:KKT-primal2dual}.
Then, using \eqref{eq:precomputed},
the duality gap $\Gnew(\Wold, \Alphaoldh)$ can be computed in $O((\nnew-\nold)d)$ time as follows:
\begin{align}
& \Gnew(\Wold, \Alphaoldh) := \Pnew(\Wold) - \Dnew(\Alphaoldh) \nonumber\\
= & \frac{1}{\nnew} \biggl[ \sum_{i=\nold+1}^{\nnew} [\ell_{y_i}(\Xnew_{i:}\Wold) + \ell^*_{y_i}(-\alphaoldh{i})] \nonumber\\
	& + \nold\PrimL{\nold}(\Xold\Wold) + \nnew\PrimR(\Wold) \nonumber\\
	& + \nold\PrimL{\nold}^*\Bigl( -\frac{1}{\nold}\Alphaold \Bigr) \nonumber\\
	& + \nnew \PrimR^*\Bigl(\frac{1}{\nnew}\Bigl[\XoldT\Alphaold + \sum_{i=\nold+1}^{\nnew}\alphaoldh{i}\Xnew_{i:}\Bigr]\Bigr)
	\biggr].
	\label{eq:bound-instance-addition}
\end{align}

\subsubsection{Bound computations for feature removals/additions} \label{sec:bound-feat-change}

We can conduct a similar procedure to Section \ref{sec:bound-inst-change} for feature additions and removals.

First, we consider that the last $\dold-\dnew$ features are removed
($n:=\nnew=\nold$, $\dnew < \dold$ and $\Xold_{j:} = \Xnew_{j:}~\forall j\in[\dnew]$).
In this case, we simply set $\hat{\bm{\alpha}}\gets\Alphaold$ in Theorem \ref{th:bounds},
whereas we construct $\hat{\bm{w}}\in\mathbb{R}^{\dnew}$ as
$\hat{\bm{w}} = \Woldh := \{ \wold{j} \}_{j\in[\dnew]}$.
Then, using \eqref{eq:precomputed},
the duality gap $\Gnew(\Wold, \Alphaoldh)$ can be computed in $O((\dold-\dnew)n)$ time as follows:
\begin{align}
& \Gnew(\Woldh, \Alphaold) := \Pnew(\Woldh) - \Dnew(\Alphaold) \nonumber\\
= & - \sum_{j = \dnew+1}^{\dold} \Bigl[ \rho_j(\wold{j}) + \rho_j^*\Bigl( \frac{1}{n}\XoldT_{:j}\Alphaold \Bigr) \Bigr] \nonumber\\
	& + \PrimL{n}\Bigl(\Xold\Wold - \sum_{j = \dnew+1}^{\dold} \wold{j}\Xold_{:j} \Bigr) + \PrimR(\Wold) - \Dold(\Alphaold)
	\label{eq:bound-feature-removal}
\end{align}

Then we consider that the last $\dnew-\dold$ features are added
($n:=\nnew=\nold$, $\dnew > \dold$ and $\Xold_{j:} = \Xnew_{j:}~\forall j\in[\dold]$).
In this case, we set $\hat{\bm{\alpha}}\gets\Alphaold$ in Theorem \ref{th:bounds},
whereas we construct $\hat{\bm{w}}\in\mathbb{R}^{\dnew}$ as
\begin{align*}
& \hat{\bm{w}}\gets\Woldh, \quad\text{where} \\
& \woldh{j} = \wold{j}, & (j\in[\dold]) \\
& \woldh{j} \in \partial\rho_j^*((1/n)\Xnew_{:j}\Alphaold), & (j\in[\dnew]\setminus[\dold])
\end{align*}
The latter is based on the KKT condition \eqref{eq:KKT-dual2primal}.
Then, with \eqref{eq:precomputed},
the duality gap $\Gnew(\Wold, \Alphaoldh)$ can be computed in $O((\dnew-\dold)n)$ time as follows:
\begin{align}
& \Gnew(\Woldh, \Alphaold) := \Pnew(\Woldh) - \Dnew(\Alphaold) \nonumber\\
= & \sum_{j = \dold+1}^{\dnew} \Bigl[ \rho_j(\woldh{j}) + \rho_j^*\Bigl( \frac{1}{n}\XnewT_{:j}\Alphaold \Bigr) \Bigr] \nonumber\\
	& + \PrimL{n}\Bigl(\Xold\Wold + \sum_{j = \dold+1}^{\dnew} \woldh{j}\Xnew_{:j} \Bigr) + \PrimR(\Wold) - \Dold(\Alphaold).
	\label{eq:bound-feature-addition}
\end{align}

\section{Application to fast LOOCV} \label{sec:GLRU-LOOCV}

\begin{algorithm}[tp]
\caption{Naive LOOCV.}
\label{alg:ordinary-loocv}
\begin{algorithmic}[1]
\REQUIRE $X\in\mathbb{R}^{n\times d}, \bm{y}\in\{-1, +1\}^n$.
\STATE $e \gets 0$. \COMMENT{number of classification errors}
\FOR{$i\in[n]$}
	\STATE \label{alg:ordinary-loocv-train} $\bm{w}^{*(-i)} \gets \targmin_{\bm{w}\in\mathbb{R}^d}\Prim_{X^{(-i)}}(\bm{w})$.
		\\ \COMMENT{Here, if we know $\bm{w}^* := \targmin_{\bm{w}\in\mathbb{R}^d}\Prim_{X}(\bm{w})$,}
		\\ \COMMENT{then we can start training from $\bm{w}^*$ for faster}
		\\ \COMMENT{training (known as ``warm-start'').}
	\STATE {\bf if}~$y_i \neq \mathrm{sign}(X_{i:}\bm{w}^{*(-i)})$~{\bf then}~$e\gets e+1$.
\ENDFOR
\STATE {\bf return}~$e$.
\end{algorithmic}
\end{algorithm}

\begin{algorithm}[tp]
\caption{LOOCV-GLRU: Fast LOOCV by GLRU. We primarily present the processes when the precondition (i-P) in Theorem \ref{th:bounds} is satisfied; we present the different processes for the case when the precondition (ii-P) is satisfied with brackets $\langle\quad\rangle$.}
\label{alg:loocv-GLRU}
\begin{algorithmic}[1]
\REQUIRE $X\in\mathbb{R}^{n\times d}, \bm{y}\in\{-1, +1\}^n$.
\STATE $e \gets 0$. \COMMENT{number of classification errors}
\STATE \label{alg:loocv-GLRU:initial-training} $\bm{w}^* \gets \targmin_{\bm{w}\in\mathbb{R}^d}\Prim_{X}(\bm{w})$.
	\\ \COMMENT{Compute the model parameter with all instances}
\FOR{$i\in[n]$}
	\STATE Compute $\Gnew(\Wold, \Alphaoldh)$ by \eqref{eq:bound-instance-removal} with $\nold\gets n$, $\nnew\gets n-1$, and the removed instance being the $i^\mathrm{th}$.
		\\ \COMMENT{Here we assume $\Wold = \bm{w}^*$ (already computed)}
		\\ \COMMENT{and $\Wnew = \bm{w}^{*(-i)}$ (not computed).}
		\\ \COMMENT{$\Alphaoldh = \bm{\alpha}^*$ can be computed from $\bm{w}^*$ by \eqref{eq:KKT-primal2dual}.}
	\STATE Compute $r_P$ by \eqref{eq:primal-bound}~$\langle${}$r_D$ by \eqref{eq:dual-primal-bound}$\rangle$.
	\STATE Compute the lower and the upper bounds of $X_{i:}\bm{w}^{*(-i)}$ by \eqref{eq:prediction-bound-primal} $\langle$by \eqref{eq:prediction-bound-dual}$\rangle$, denoted by $L_i$ and $U_i$, respectively.
	\IF[Classified positive]{$L_i > 0$} \label{alg:loocv-GLRU:check-label-positive}
		\STATE {\bf if}~$y_i = -1$~{\bf then}~$e\gets e+1$.
	\ELSIF[Classified negative]{$U_i < 0$} \label{alg:loocv-GLRU:check-label-negative}
		\STATE {\bf if}~$y_i = +1$~{\bf then}~$e\gets e+1$.
	\ELSE[Not determined]
		\STATE \label{alg:loocv-GLRU-train} $\bm{w}^{*(-i)} \gets \targmin_{\bm{w}\in\mathbb{R}^d}\Prim_{X^{(-i)}}(\bm{w})$.
		\STATE {\bf if}~$y_i \neq \mathrm{sign}(X_{i:}\bm{w}^{*(-i)})$~{\bf then}~$e\gets e+1$.
	\ENDIF \label{alg:loocv-GLRU:check-label-end}
\ENDFOR
\STATE {\bf return}~$e$.
\end{algorithmic}
\end{algorithm}

We demonstrate an algorithm of fast LOOCV
as an application of GRLU for instance removals.
Because we must solve numerous optimization problems for
slightly different datasets of one-instance removals in LOOCV,
we expect that GRLU works effectively.
Here we consider only the binary classification problem, and suppose that we would like to know the total classification errors for $n$ cases of removing one instance.

Given a dataset $(X, \bm{y})$ ($X\in\mathbb{R}^{n\times d}$, $\bm{y}\in\mathbb{R}^n$), let $(X^{(-i)}, \bm{y}^{(-i)})$ be the dataset of $i$th instance being removed, and $\bm{w}^{*(-i)}$ be the model parameters trained from them.
Naive LOOCV is implemented as Algorithm \ref{alg:ordinary-loocv}:
first we compute $\bm{w}^{*(-i)}$, then make a prediction for the removed instance $X_{i:}$.
However, it requires a costly operation of computing $\bm{w}^{*(-i)}$ for $n$ times (line \ref{alg:ordinary-loocv-train}).

Using GRLU, if the assumption of (i-P) (resp. (ii-P)) in Theorem \ref{th:bounds} holds, we can avoid some of the computations of $\bm{w}^{*(-i)}$:
After computing \eqref{eq:bound-instance-removal} in $O(d)$ time,
we can use the possible region of $\bm{w}^{*(-i)}$ with \eqref{eq:primal-bound} (resp. \eqref{eq:dual-primal-bound}) and thus the prediction result $X_{i:}\bm{w}^{*(-i)}$ with Primal-SCB \eqref{eq:prediction-bound-primal} (resp. Dual-SCB \eqref{eq:prediction-bound-dual}).
The entire process is described in Algorithm \ref{alg:loocv-GLRU}.
For each $i\in[n]$, we must compute $\bm{w}^{*(-i)}$ at a high computational cost (line \ref{alg:loocv-GLRU-train}) only if the sign of $X_{i:}\bm{w}^{*(-i)}$ cannot be determined by the Primal-SCB or Dual-SCB.

\begin{remark} \label{rm:stop-optimization-earlier}
Even if we must conduct the training because GLRU cannot determine the sign of $X_{i:}\bm{w}^{*(-i)}$, GLRU can accelerate this training process.
The training computation is usually stopped when the difference between the current model parameter and the optimal solution is lower than a predetermined threshold.
However, when we train of $\bm{w}^{*(-i)}$ (line \ref{alg:loocv-GLRU-train} in Algorithm \ref{alg:loocv-GLRU}), in addition to checking the difference from the optimal solution, the optimization can be stopped when the prediction result $X_{i:}\bm{w}^{*(-i)}$ is determined \cite{okumura2015quick}.

Suppose that $(\tilde{\bm{w}}, \tilde{\bm{\alpha}})$ are model parameters that are not fully optimized to $(\bm{w}^{*(-i)}, \bm{\alpha}^{*(-i)})$.
If the assumption (i-P) in Theorem \ref{th:bounds} holds (we can do the similar if the assumption (ii-P) in Theorem \ref{th:bounds} holds), that is, if $\rho_j(t)$ is $\lambda$-strongly convex ($\lambda > 0$) for all $j\in[d]$,
then we have the following by \eqref{eq:primal-bound} and \eqref{eq:prediction-bound-primal} with $\Wnew\gets\bm{w}^{*(-i)}$ as:
\begin{align}
& X_{i:}^\top\tilde{\bm{w}} - r_P^{(-i)} \|X_{i:}\|_2
	\leq X_{i:}^\top \bm{w}^{*(-i)} \nonumber\\
	& \leq X_{i:}^\top\tilde{\bm{w}} + r_P^{(-i)} \|X_{i:}\|_2, \label{eq:label-lo-instance} \\
& \text{where}\quad r_P^{(-i)} := \sqrt{\frac{2}{\lambda}\Gnew(\tilde{\bm{w}}, \tilde{\bm{\alpha}})}. \nonumber
\end{align}
Therefore, if the sign of $X_{i:}^\top \bm{w}^{*(-i)}$ is determined by \eqref{eq:label-lo-instance}, then we can stop the training.
More specifically, we periodically compute $\Gnew(\tilde{\bm{w}}, \tilde{\bm{\alpha}})$,
while optimizing $(\tilde{\bm{w}}, \tilde{\bm{\alpha}})$, and we stop the training if
\begin{align*}
r_P^{(-i)} \|X_{i:}\|_2 < |X_{i:}^\top\tilde{\bm{w}}|,
\end{align*}
or equivalently,
\begin{align*}
\Gnew(\tilde{\bm{w}}, \tilde{\bm{\alpha}}) < \frac{\lambda(X_{i:}^\top\tilde{\bm{w}})^2}{2 \|X_{i:}\|_2^2}.
\end{align*}
\end{remark}

\section{GLRU for fast stepwise feature elimination for binary classification} \label{sec:GLRU-stepwise}

\begin{algorithm}[tp]
\caption{Naive stepwise feature elimination.}
\label{alg:ordinary-stepwise}
\begin{algorithmic}[1]
\REQUIRE $X\in\mathbb{R}^{n\times d}, \bm{y}\in\{-1, +1\}^n$.
	\COMMENT{training set}
\REQUIRE $\check{X}\in\mathbb{R}^{m\times d}, \check{\bm{y}}\in\{-1, +1\}^m$.
	\COMMENT{validation set}
\STATE $S \gets [d]$. \COMMENT{set of currently selected features}
\STATE $\bm{w}^{*(S)} \gets \targmin_{\bm{w}\in\mathbb{R}^d}\Prim_X(\bm{w})$.
\STATE $E[\mathbf{null}] \gets \mathrm{err}(\bm{w}^*, S)$.
	\\ \COMMENT{$E[j]$: the validation error when the $j^\mathrm{th}$ feature}
	\\ \COMMENT{is removed.}
	\\ \COMMENT{$\mathrm{err}(\bm{w}, S) := |\{ i \mid i\in[m],~\check{y}_i\neq \mathrm{sign}(\check{X}_{i,S}\bm{w}) \}|$.}
\WHILE{$S\neq\emptyset$} \label{alg:ordinary-stepwise-elimination-step}
	\FOR{$j\in S$} \label{alg:ordinary-stepwise-each-feature}
		\STATE \label{alg:ordinary-stepwise-train} $\bm{w}^{*(S\setminus\{j\})} \gets \targmin_{\bm{w}\in\mathbb{R}^{|S|-1}}\Prim_{X_{:(S\setminus\{j\})}}(\bm{w})$.
			\\ \COMMENT{We can start training from $\bm{w}^{*(S)}_{S\setminus\{j\}}$ for faster}
			\\ \COMMENT{training (known as ``warm-start'').}
		\STATE $E[j] \gets \mathrm{err}(\bm{w}^{*(S\setminus\{j\})}, S\setminus\{j\})$.
	\ENDFOR
		\\ \COMMENT{End if no removal is the best}
	\STATE $j_\mathrm{best} \gets \targmin_{j\in\{\mathbf{null}\}\cup S}~E[j]$.
	\STATE {\bf if}~$j_\mathrm{best} = \mathbf{null}$~{\bf then}~{\bf exit while}.
		\\ ~
		\\ \COMMENT{Which feature to remove}
	\STATE $S \gets S\setminus\{j_\mathrm{best}\}$;~$E[\mathbf{null}] \gets E[j_\mathrm{best}]$.
\ENDWHILE
\STATE {\bf return}~$S$.
\end{algorithmic}
\end{algorithm}

\begin{algorithm}[tp]
\caption{Stepwise-GLRU: Fast stepwise feature elimination with GLRU.}
\label{alg:stepwise-GLRU}
\begin{algorithmic}[1]
\REQUIRE $X\in\mathbb{R}^{n\times d}, \bm{y}\in\{-1, +1\}^n$.
	\COMMENT{training set}
\REQUIRE $\check{X}\in\mathbb{R}^{m\times d}, \check{\bm{y}}\in\{-1, +1\}^m$.
	\COMMENT{validation set}
\STATE $S \gets [d]$. \COMMENT{set of currently selected features}
\STATE \label{alg:stepwise-GLRU:initial-training} $\bm{w}^{*(S)} \gets \targmin_{\bm{w}\in\mathbb{R}^d}\Prim_X(\bm{w})$.
\STATE $E[\mathbf{null}] \gets \mathrm{err}(\bm{w}^*, S)$. \COMMENT{Same as Algorithm \ref{alg:ordinary-stepwise}}
\WHILE{$S\neq\emptyset$} \label{alg:stepwise-GLRU-elimination-step}
	\FOR{$j\in S$} \label{alg:stepwise-GLRU-each-feature}
		\STATE Compute $C[j]$, $I[j]$ and $Z[j]$ in Algorithm \ref{alg:bound-stepwise-GLRU},
			where $I[j]$ is a lower bound of $E[j]$.
	\ENDFOR
		\\ ~
		\\ \COMMENT{Which feature to remove}
	\STATE $j_\mathrm{best} \gets \mathbf{null}$.
	\FOR{$j\in S$ (in the ascending order of $I[j]$)} \label{alg:stepwise-GLRU:each-feature}
		\IF{$I[j]\geq E[j_\mathrm{best}]$} \label{alg:stepwise-GLRU:end-training}
			\STATE {\bf exit for}.
				\\ \COMMENT{Removing the $j^\mathrm{th}$ feature (current or}
				\\ \COMMENT{upcoming) cannot be the best, since $j$'s are}
				\\ \COMMENT{given in the ascending order of $I[j]$.}
		\ELSE
			\STATE \label{alg:stepwise-GLRU-train2} $\bm{w}^{*(S\setminus\{j\})} \gets \targmin_{\bm{w}\in\mathbb{R}^{|S|-1}}\Prim_{X_{:(S\setminus\{j\})}}(\bm{w})$.
			\STATE $E[j]\gets\mathrm{err}(\bm{w}^{*(S\setminus\{j\})}, S\setminus\{j\})$.
			\STATE {\bf if}~$E[j] < E[j_\mathrm{best}]$~{\bf then}~$j_\mathrm{best}\gets j$.
		\ENDIF
	\ENDFOR
	\STATE {\bf if}~$j_\mathrm{best} = \mathbf{null}$~{\bf then}~{\bf exit while}.
	\STATE $S \gets S\setminus\{j_\mathrm{best}\}$;~$E[\mathbf{null}] \gets E[j_\mathrm{best}]$.
\ENDWHILE
\STATE {\bf return}~$S$.
\end{algorithmic}
\end{algorithm}

Similar to LOOCV in Section \ref{sec:GLRU-LOOCV}, GLRU for feature removal can be applied to stepwise feature elimination.
We considered only binary classification problems.

Let the training and validation datasets be $(X, \bm{y})$ ($X\in\mathbb{R}^{n\times d}$, $\bm{y}\in\mathbb{R}^n$) and $(\check{X}, \check{\bm{y}})$ ($\check{X}\in\mathbb{R}^{m\times d}$, $\check{\bm{y}}\in\mathbb{R}^m$), respectively.

Naive stepwise feature elimination is shown in Algorithm \ref{alg:ordinary-stepwise}.
For each elimination step of $S$ (line \ref{alg:ordinary-stepwise-elimination-step}) and $j\in S$ (line \ref{alg:ordinary-stepwise-each-feature}), we must compute $\bm{w}^{*(S\setminus\{j\})}$ at a high computational cost (line \ref{alg:ordinary-stepwise-train}).

Using GRLU, we can avoid some of the computations of $\bm{w}^{*(S\setminus\{j\})}$ if either assumption (i-P) or (ii-P) in Theorem \ref{th:bounds} holds.
Using GLRU we can identify the lower and the upper bounds of the validation errors for $m$ instances.

The entire process is described in Algorithm \ref{alg:stepwise-GLRU}.
First, for each feature $j$ that is not yet removed, we compute the lower bounds of the validation errors by GLRU
(line \ref{alg:stepwise-GLRU-each-feature} in Algorithm \ref{alg:stepwise-GLRU}; details are provided in Algorithm \ref{alg:bound-stepwise-GLRU}).
In contrast to naive stepwise feature elimination, which must attempt training computations for all feature removals $j\in S$,
with GLRU we can skip the training computation for a feature removal $j\in S$ if it cannot provide the lowest validation error with the strategy below:
\begin{itemize}
\item We ``train the model parameter for the feature removal and compute the true validation errors''
	in the ascending order of the lower bounds of validation errors $I[j]$ (line \ref{alg:stepwise-GLRU:each-feature} in Algorithm \ref{alg:stepwise-GLRU}).
\item If the lower bounds of validation errors $I[j]$ is equal to or larger than the lowest true validation error found so far, we can omit training the model parameter for the remaining feature removals (line \ref{alg:stepwise-GLRU:end-training} in Algorithm \ref{alg:stepwise-GLRU}).
\end{itemize}

\begin{algorithm}[tp]
\caption{Bounds of validation errors for stepwise feature elimination by GLRU. We primarily present the processes when the precondition (i-P) in Theorem \ref{th:bounds} is satisfied, and we present the different processes for the case when the precondition (ii-P) is satisfied with brackets $\langle\quad\rangle$.}
\label{alg:bound-stepwise-GLRU}
\begin{algorithmic}[1]
\REQUIRE $X\in\mathbb{R}^{n\times d}, \bm{y}\in\{-1, +1\}^n$.
\REQUIRE $\check{X}\in\mathbb{R}^{m\times d}, \check{\bm{y}}\in\{-1, +1\}^m$.
\REQUIRE $S\subseteq[d]$, $\bm{w}^{*(S)}\in\mathbb{R}^{|S|}$.
\STATE Compute $\Gnew(\Woldh, \Alphaold)$ by \eqref{eq:bound-instance-removal} with $\dold\gets|S|$, $\dnew\gets|S|-1$, and the removed feature being the $j^\mathrm{th}$.
	\\ \COMMENT{Here we assume $\Wold = \bm{w}^{*(S)}$ (computed) and}
	\\ \COMMENT{$\Wnew = \bm{w}^{*(S\setminus\{j\})}$ (not computed).}
\STATE Compute $r_P$ by \eqref{eq:primal-bound}~$\langle$by \eqref{eq:dual-primal-bound}$\rangle$.
\STATE $C[j]\gets 0$; $I[j]\gets 0$; $Z[j]\gets 0$.
	\\ \COMMENT{$C[j]$ (resp. $I[j]$, $Z[j]$) is the number of validations}
	\\ \COMMENT{instances that GLRU can identify as classified correct}
	\\ \COMMENT{(resp. classified as incorrect, or unable to identify).}
\FOR{$i\in[m]$}
	\STATE Compute the lower and the upper bounds of $\check{X}_{i,S}\bm{w}^{*(S\setminus\{j\})}$ by \eqref{eq:prediction-bound-primal} $\langle$by \eqref{eq:prediction-bound-dual}$\rangle$, denoted by $L_i$ and $U_i$, respectively.
	\IF[Classified positive]{$L_i > 0$}
		\STATE {\bf if}~{$\check{y}_i = +1$}~{\bf then}~$C[j]\gets C[j] + 1$;
		\STATE {\bf else}~$I[j]\gets I[j] + 1$.
	\ELSIF[Classified negative]{$U_i < 0$}
		\STATE {\bf if}~{$\check{y}_i = -1$}~{\bf then}~$C[j]\gets C[j] + 1$;
		\STATE {\bf else}~$I[j]\gets I[j] + 1$.
	\ELSE[Not determined]
		\STATE $Z[j]\gets Z[j] + 1$.
	\ENDIF
\ENDFOR
\end{algorithmic}
\end{algorithm}

\section{Experiments} \label{sec:experiment}

In this section we examine the experimental performance of the proposed GLRU for LOOCV (Section \ref{sec:GLRU-LOOCV}) and stepwise feature elimination (Section \ref{sec:GLRU-stepwise}).

In Sections \ref{sec:exp-loocv} and \ref{sec:exp-stepwise} we present the reduction in computational cost by GLRU in LOOCV (Section \ref{sec:GLRU-LOOCV}) and stepwise feature elimination (Section \ref{sec:GLRU-stepwise}), respectively.
For LOOCV, we also examine the cost compared to a fast but approximate LOOCV method.

In addition, in Section \ref{sec:exp-tightness} we examine the tightness of the prediction bounds for multiple additions or removals of instances or features for larger datasets.
In this experiment, we also discuss the differences between Primal-SCB \eqref{eq:prediction-bound-primal} and Dual-SCB \eqref{eq:prediction-bound-dual} and identify which method provides tighter bounds.

Because we use the datasets for binary classification, we used L2-regularized logistic regression: $\ell_y(t) = \log(1 + e^{-yt})$ and $\rho_j(t) = \frac{\lambda}{2}t^2$ (see Appendix \ref{app:losses-regularizations} for further details). Detailed training computation is provided in Appendix \ref{app:exp-training}.

\begin{table}[t]
\caption{Datasets used for the experiments of Sections \ref{sec:exp-loocv} and \ref{sec:exp-stepwise} (LOOCV and stepwise feature eliminations). All are from UCI Machine Learning Repository \cite{Dua2017UCI}. All are for binary classifications.}
\label{tb:datasets-loocv-stepwise}
See Appendix \ref{app:exp-normalization} for details of the ``normalization strategy''.
\begin{center}
\begin{tabular}{ccrr}
\hline
        & Normalization & \ccell{\#Instances} & \ccell{\#Features} \\
Dataset & strategy      & \ccell{$n$}         & \ccell{$d$} \\
\hline
{\tt arcene}  & Dense &   200 &  9,961 \\
{\tt dexter}  & Dense &   600 & 11,035 \\
{\tt gisette} & Dense & 6,000 &  4,955 \\
\hline
\end{tabular}
\end{center}
\end{table}

\begin{table}[tp]
\caption{Datasets used in the experiment of Section \ref{sec:exp-tightness}. All are for binary classifications.}
\label{tab:datasets}
\vspace*{-\baselineskip}
\begin{itemize}
\item {\bf Source}:
	The website where the datasets are distributed.
	``UCI'': UCI Machine Learning Repository \cite{Dua2017UCI},
	``LIBSVM'': LIBSVM Data \cite{libsvmDataset}.
	For the datasets {\tt covtype}, {\tt rcv1\_train} and {\tt news20},
	we used the datasets with the suffix ``.binary'',
	since they are originally for multiclass classifications and
	binary classification versions are separately distributed.
\item {\bf Normalization strategy}:
	For consistency with the theoretical results (Section \ref{sec:exp-tightness-discussion}),
	we normalized each feature $j\in[d]$ to have its L2-norm $\|X_{:j}\|_2 = \sqrt{n}$.
	``Dense'' means that each feature is normalized to mean zero and variance one,
	while ``Sparse'' means that each feature is only scaled.
	See Appendix \ref{app:exp-normalization} for details.
\item $n_\mathrm{train}$, $d$:
	``$\pm 10$'' denotes that modifications were made (one to ten removals and additions).
\end{itemize}
\begin{center}
{\scriptsize
\begin{tabular}{l|ccrr|rrr}
\hline
                           &        & Normalization & \multicolumn{2}{c|}{Size of whole dataset}            & \multicolumn{3}{c}{Size of the dataset in experiments} \\
 \multicolumn{1}{c|}{Name} & Source & strategy      & \ccell{\#Instances} & \multicolumn{1}{c|}{\#Features} & \ccell{$n_\mathrm{train}$} & \ccell{$n_\mathrm{test}$} & \ccell{$d$} \\
\hline
\multicolumn{8}{l}{Datasets for instance modifications} \\
\hline
{\tt Skin\_NonSkin} & UCI    & Dense  &   245,057 &         3 & {\sf   242,596$\pm$10} &  2,451 &      3 \\
{\tt covtype}       & LIBSVM & Dense  &   581,012 &        54 & {\sf   575,191$\pm$10} &  5,811 &     54 \\
{\tt SUSY}          & UCI    & Dense  & 5,000,000 &        18 & {\sf 4,949,990$\pm$10} & 50,000 &     18 \\
{\tt rcv1\_train}   & LIBSVM & Sparse &    20,242 &    47,236 & {\sf    20,029$\pm$10} &    203 & 47,236 \\
\hline
\multicolumn{8}{l}{Datasets for feature modifications} \\
\hline
{\tt rcv1\_train}   & LIBSVM & Sparse &    20,242 &    47,236 & 1000 & 19,242 & {\sf    47,226$\pm$10} \\
{\tt dorothea}      & UCI    & Sparse &     1,150 &    91,598 & 1000 &    150 & {\sf    91,588$\pm$10} \\
{\tt news20}        & LIBSVM & Sparse &    19,996 & 1,355,191 & 1000 & 18,996 & {\sf 1,355,181$\pm$10} \\
{\tt url}           & LIBSVM & Sparse & 2,396,130 & 3,231,961 & 1000 & 19,000 & {\sf 3,231,951$\pm$10} \\
\hline
\end{tabular}
}
\end{center}
\end{table}

\subsection{Leave-one-out cross-validation} \label{sec:exp-loocv}
\begin{table}[tp]
\caption{Computational costs of four LOOCV methods: Applying proposed method (GLRU) or not, and the training computation is conducted exactly or approximately.}
\label{tb:exp-loocv}
\begin{itemize}
\item ``Prep'' (preprocessing) denotes the cost for training for the dataset before feature removals (line \ref{alg:loocv-GLRU:initial-training} in Algorithm \ref{alg:loocv-GLRU}). In addition, for approximate training, the cost to compute the hessian $\left.\frac{\partial^2\Prim_X(\bm w)}{\partial \bm w^2}\right|_{\bm w = \bm w^*}$ is added.
\item ``(ratio)'' denotes the ratio of time with GLRU to the one without GLRU.
\item ``Number of trainings'' denotes the number of instance removals which the training computation of line \ref{alg:loocv-GLRU-train} in Algorithm \ref{alg:loocv-GLRU} are needed.
\end{itemize}
\begin{center}
{\scriptsize
\begin{tabular}{c|rrrr|rrrr|rrr}
\hline
& \multicolumn{4}{c|}{Exact LOOCV} & \multicolumn{4}{c|}{Approximate LOOCV} & \multicolumn{3}{c}{Computation}\\
\cline{2-5}
\cline{6-9}
          &              & \ccell{Exact-} & \multicolumn{2}{c|}{Exact-GLRU}  &              & \ccell{Approx-} & \multicolumn{2}{c|}{Approx-GLRU} & \multicolumn{3}{c}{of GLRU} \\
\cline{10-12}
          & \ccell{Prep} & \ccell{Naive}  & \multicolumn{2}{c|}{[sec]}       & \ccell{Prep} & \ccell{Naive}  & \multicolumn{2}{c|}{[sec]}       & Time  & \multicolumn{2}{c}{Number of} \\
$\lambda$ & \ccell{[sec]}& \ccell{[sec]}  & \multicolumn{2}{c|}{(ratio)}     & \ccell{[sec]}& \ccell{[sec]}  & \multicolumn{2}{c|}{(ratio)}     & [sec] & \multicolumn{2}{c}{trainings} \\
\hline
\multicolumn{11}{l}{Dataset {\tt arcene} ($n=200$, $d=9961$)} \\
\hline
$2^{0}$   &   2.2 & {\sf   106.7} & {\sf   37.6} & (35.2\%) &  4048.8 & {\sf  28.2} & {\sf 12.2} & (43.2\%) & {\sf 0.051} &  {\sf 86} & (43.0\%) \\
$2^{-1}$  &   2.9 & {\sf   137.5} & {\sf   48.8} & (35.5\%) &  3638.5 & {\sf  40.9} & {\sf 18.7} & (45.8\%) & {\sf 0.053} &  {\sf 91} & (45.5\%) \\
$2^{-2}$  &   3.6 & {\sf   147.3} & {\sf   55.4} & (37.6\%) &  3649.5 & {\sf  33.0} & {\sf 15.4} & (46.7\%) & {\sf 0.052} &  {\sf 93} & (46.5\%) \\
$2^{-3}$  &   3.8 & {\sf   165.0} & {\sf   65.4} & (39.6\%) &  5241.1 & {\sf  29.2} & {\sf 13.5} & (46.1\%) & {\sf 0.052} &  {\sf 92} & (46.0\%) \\
$2^{-4}$  &   5.9 & {\sf   193.6} & {\sf   74.7} & (38.6\%) &  4186.5 & {\sf  28.4} & {\sf 13.1} & (45.9\%) & {\sf 0.051} &  {\sf 92} & (46.0\%) \\
$2^{-5}$  &   5.3 & {\sf   214.2} & {\sf   84.7} & (39.5\%) &  3678.1 & {\sf  33.1} & {\sf 15.2} & (46.1\%) & {\sf 0.052} &  {\sf 92} & (46.0\%) \\
$2^{-6}$  &   7.7 & {\sf   219.8} & {\sf   84.9} & (38.6\%) &  5066.9 & {\sf  34.2} & {\sf 16.1} & (46.9\%) & {\sf 0.052} &  {\sf 93} & (46.5\%) \\
$2^{-7}$  &   6.4 & {\sf   228.3} & {\sf   86.3} & (37.8\%) &  3599.1 & {\sf  38.1} & {\sf 18.5} & (48.6\%) & {\sf 0.053} &  {\sf 96} & (48.0\%) \\
$2^{-8}$  &   9.6 & {\sf   256.6} & {\sf   97.6} & (38.0\%) &  5390.4 & {\sf  26.3} & {\sf 12.8} & (48.5\%) & {\sf 0.055} &  {\sf 96} & (48.0\%) \\
$2^{-9}$  &  10.0 & {\sf   280.5} & {\sf  102.9} & (36.7\%) &  3972.0 & {\sf  41.4} & {\sf 19.9} & (48.0\%) & {\sf 0.053} &  {\sf 96} & (48.0\%) \\
$2^{-10}$ &  10.8 & {\sf   269.0} & {\sf   97.5} & (36.2\%) &  5560.5 & {\sf  26.5} & {\sf 12.9} & (48.7\%) & {\sf 0.051} &  {\sf 97} & (48.5\%) \\
\hline
\multicolumn{11}{l}{Dataset {\tt dexter} ($n=600$, $d=11035$)} \\
\hline
$2^{0}$   &   2.8 & {\sf   401.4} & {\sf   28.3} & ( 7.1\%) & 10328.7 & {\sf 115.0} & {\sf  8.4} & ( 7.3\%) & {\sf 0.171} &  {\sf 42} & ( 7.0\%) \\
$2^{-1}$  &   3.7 & {\sf   533.9} & {\sf   38.2} & ( 7.1\%) & 11684.6 & {\sf 181.6} & {\sf 12.6} & ( 7.0\%) & {\sf 0.179} &  {\sf 43} & ( 7.2\%) \\
$2^{-2}$  &   4.1 & {\sf   543.8} & {\sf   38.2} & ( 7.0\%) & 11602.9 & {\sf 209.4} & {\sf 14.9} & ( 7.1\%) & {\sf 0.194} &  {\sf 44} & ( 7.3\%) \\
$2^{-3}$  &   5.0 & {\sf   617.1} & {\sf   42.3} & ( 6.8\%) & 13320.5 & {\sf 173.6} & {\sf 13.2} & ( 7.6\%) & {\sf 0.169} &  {\sf 44} & ( 7.3\%) \\
$2^{-4}$  &   7.2 & {\sf   625.4} & {\sf   43.0} & ( 6.9\%) & 13656.5 & {\sf 132.4} & {\sf 10.5} & ( 7.9\%) & {\sf 0.164} &  {\sf 44} & ( 7.3\%) \\
$2^{-5}$  &   8.7 & {\sf   728.3} & {\sf   49.7} & ( 6.8\%) & 14314.2 & {\sf 106.2} & {\sf  8.0} & ( 7.5\%) & {\sf 0.159} &  {\sf 44} & ( 7.3\%) \\
$2^{-6}$  &   7.1 & {\sf   731.3} & {\sf   50.0} & ( 6.8\%) & 13986.5 & {\sf 111.7} & {\sf  8.3} & ( 7.4\%) & {\sf 0.161} &  {\sf 44} & ( 7.3\%) \\
$2^{-7}$  &   7.4 & {\sf   954.3} & {\sf   64.3} & ( 6.7\%) & 11466.0 & {\sf 205.8} & {\sf 14.9} & ( 7.2\%) & {\sf 0.181} &  {\sf 44} & ( 7.3\%) \\
$2^{-8}$  &  11.3 & {\sf   847.8} & {\sf   57.7} & ( 6.8\%) & 14353.8 & {\sf 111.0} & {\sf  8.1} & ( 7.3\%) & {\sf 0.160} &  {\sf 43} & ( 7.2\%) \\
$2^{-9}$  &  12.7 & {\sf  1020.8} & {\sf   69.0} & ( 6.8\%) & 11907.4 & {\sf  92.4} & {\sf  6.8} & ( 7.4\%) & {\sf 0.168} &  {\sf 43} & ( 7.2\%) \\
$2^{-10}$ &  13.7 & {\sf  1014.5} & {\sf   69.0} & ( 6.8\%) & 11093.4 & {\sf 102.7} & {\sf  7.7} & ( 7.5\%) & {\sf 0.170} &  {\sf 43} & ( 7.2\%) \\
\hline
\multicolumn{11}{l}{Dataset {\tt gisette} ($n=6000$, $d=4955$)} \\
\hline
$2^{0}$   &  24.6 & {\sf 18839.3} & {\sf  743.8} & ( 3.9\%) & 15057.3 & {\sf 214.0} & {\sf  7.0} & ( 3.3\%) & {\sf 0.726} & {\sf 176} & ( 2.9\%) \\
$2^{-1}$  &  28.1 & {\sf 18919.4} & {\sf  915.7} & ( 4.8\%) & 14931.6 & {\sf 209.7} & {\sf  7.7} & ( 3.6\%) & {\sf 0.731} & {\sf 196} & ( 3.3\%) \\
$2^{-2}$  &  40.1 & {\sf 20135.3} & {\sf 1286.1} & ( 6.4\%) & 14517.6 & {\sf 211.2} & {\sf  8.2} & ( 3.9\%) & {\sf 0.730} & {\sf 209} & ( 3.5\%) \\
$2^{-3}$  &  62.4 & {\sf 22400.3} & {\sf 1543.6} & ( 6.9\%) & 15298.4 & {\sf 212.9} & {\sf  8.5} & ( 4.0\%) & {\sf 0.731} & {\sf 216} & ( 3.6\%) \\
$2^{-4}$  &  68.8 & {\sf 25000.8} & {\sf 2086.2} & ( 8.3\%) & 13861.5 & {\sf 209.0} & {\sf  9.1} & ( 4.3\%) & {\sf 0.747} & {\sf 240} & ( 4.0\%) \\
$2^{-5}$  & 117.8 & {\sf 27696.6} & {\sf 2492.1} & ( 9.0\%) & 14981.5 & {\sf 207.4} & {\sf  9.5} & ( 4.6\%) & {\sf 0.728} & {\sf 255} & ( 4.2\%) \\
$2^{-6}$  & 145.5 & {\sf 25897.6} & {\sf 2670.8} & (10.3\%) & 15457.4 & {\sf 202.4} & {\sf 10.1} & ( 5.0\%) & {\sf 0.727} & {\sf 278} & ( 4.6\%) \\
$2^{-7}$  & 212.4 & {\sf 29116.8} & {\sf 3220.7} & (11.1\%) & 15333.4 & {\sf 198.2} & {\sf 10.4} & ( 5.3\%) & {\sf 0.720} & {\sf 296} & ( 4.9\%) \\
$2^{-8}$  & 259.7 & {\sf 30690.8} & {\sf 3770.1} & (12.3\%) & 14851.1 & {\sf 197.7} & {\sf 11.1} & ( 5.6\%) & {\sf 0.718} & {\sf 313} & ( 5.2\%) \\
$2^{-9}$  & 327.2 & {\sf 32558.0} & {\sf 4318.2} & (13.3\%) & 14516.8 & {\sf 202.9} & {\sf 11.9} & ( 5.9\%) & {\sf 0.720} & {\sf 332} & ( 5.5\%) \\
$2^{-10}$ & 455.6 & {\sf 33874.5} & {\sf 4879.8} & (14.4\%) & 15525.9 & {\sf 198.0} & {\sf 12.1} & ( 6.1\%) & {\sf 0.712} & {\sf 345} & ( 5.7\%) \\
\hline
\end{tabular}
}
\end{center}
\end{table}

First, we present the effect of GLRU on LOOCV explained in Section \ref{sec:GLRU-LOOCV}.
We used the dataset in Table \ref{tb:datasets-loocv-stepwise}.
In the experiment, we compared the computational costs of the following four setups:
\begin{description}
\item[Exact-GLRU]: Proposed method (Algorithm \ref{alg:loocv-GLRU}, including Remark \ref{rm:stop-optimization-earlier})
\item[Exact-Naive]: We just run training computations for each instance removal (Algorithm \ref{alg:ordinary-loocv}).
\item[Approx-GLRU]: We replaced the training computation in ``Exact-GLRU'' with the approximate one.
\item[Approx-Naive]: We replaced the training computation in ``Exact-Naive'' with the approximate one.
\end{description}
Here, as stated in Section \ref{sec:related}, the approximate training computation is implemented by applying the Newton's method for only one step.
To conduct this, we need $O(d^3)$ time (a matrix inversion) for preprocessing (only once before LOOCV), and $O(d^2)$ time for the approximate training (that must be done for each instance removal). See Appendix \ref{app:exp-training-approx} for details.

The results are presented in Table \ref{tb:exp-loocv}.
First, comparing ``Exact-GLRU'' and ``Exact-Naive'', or comparing ``Approx-GLRU'' and ``Approx-Naive'',
the proposed method GLRU makes the computational cost much smaller.
We can also find that the computation time of GLRU itself is negligible compared to the total cost.
Moreover, in case $d$ is large and $n$ is small, the preprocessing computation for approximation becomes significantly large
because it requires $O(d^3)$-time computation of a matrix inversion.
In fact, the total computational cost of existing {\em approximate} method (the time of ``Approx-Naive'' plus ``Prep'')
is larger than that of proposed {\em exact} method (the time of ``Exact-GLRU'' plus ``Prep'')
for datasets {\tt arcene} and {\tt dexter}. This is another advantage of applying GLRU to LOOCV.

Here, we discuss the effect of GLRU on LOOCV in detail.
To this aim, we focus on the differences between
the reduction percentages in (i) ``Exact-GLRU'' coluumn and (ii) ``Number of trainings'' in Table \ref{tb:exp-loocv}.
Intuitively, we can expect the percentages in (i) to be smaller than those in (ii).
Because we employ the strategy in Remark \ref{rm:stop-optimization-earlier},
if the training computation times are the same for all instance removals,
and the time required for bound computations by GLRU was negligible,
then the total time (i) is reduced than the number of training computations (ii).
However, this is not the case: percentages (i) are significantly smaller than (ii) for {\tt arcene} dataset,
and significantly larger for {\tt gisette} dataset.
We believe that this is because the training computations skipped by GLRU but conducted in the naive method are likely to be faster than those that are not skipped, even only by a warm-start.
The cost of a training computation with a warm-start is expected to be small
when the amount of trained model parameter updates is small.
For such instance removals, GLRU is easier to determine the predicted label
for the removed instance; therefore its training computation is likely to be skipped by GLRU.

To examine this in more detail, we see the training computation cost per instance removal in Figure \ref{fig:ratio-warmstart-stopping}.
In each plot in Figure \ref{fig:ratio-warmstart-stopping}, line (A) denotes the average training time per instance by GLRU, that is, the training computations are conducted only for instance removals whose predicted labels are not determined by GLRU. On the other hand, line (B) denotes average training time per instance by naive method, that is, the training computations are conducted for all instance removals. Then we can see that the times (A) are larger than (B) for {\tt gisette} dataset, which implies that the training computation for instance removals that GLRU can skip were (relatively) effective. 
Decomposing the computation times of (B) in Figure \ref{fig:ratio-warmstart-stopping} (by warm-start) into instance removals whose labels are determined by GLRU or not: lines (B1) and (B2), respectively, then (B1) is much costless than (A). Note that (B2) must be more costly than (A) because the removed instances are the same and only (A) employs Remark \ref{rm:stop-optimization-earlier}.

\begin{figure}[tp]
\includegraphics[width=0.5\hsize]{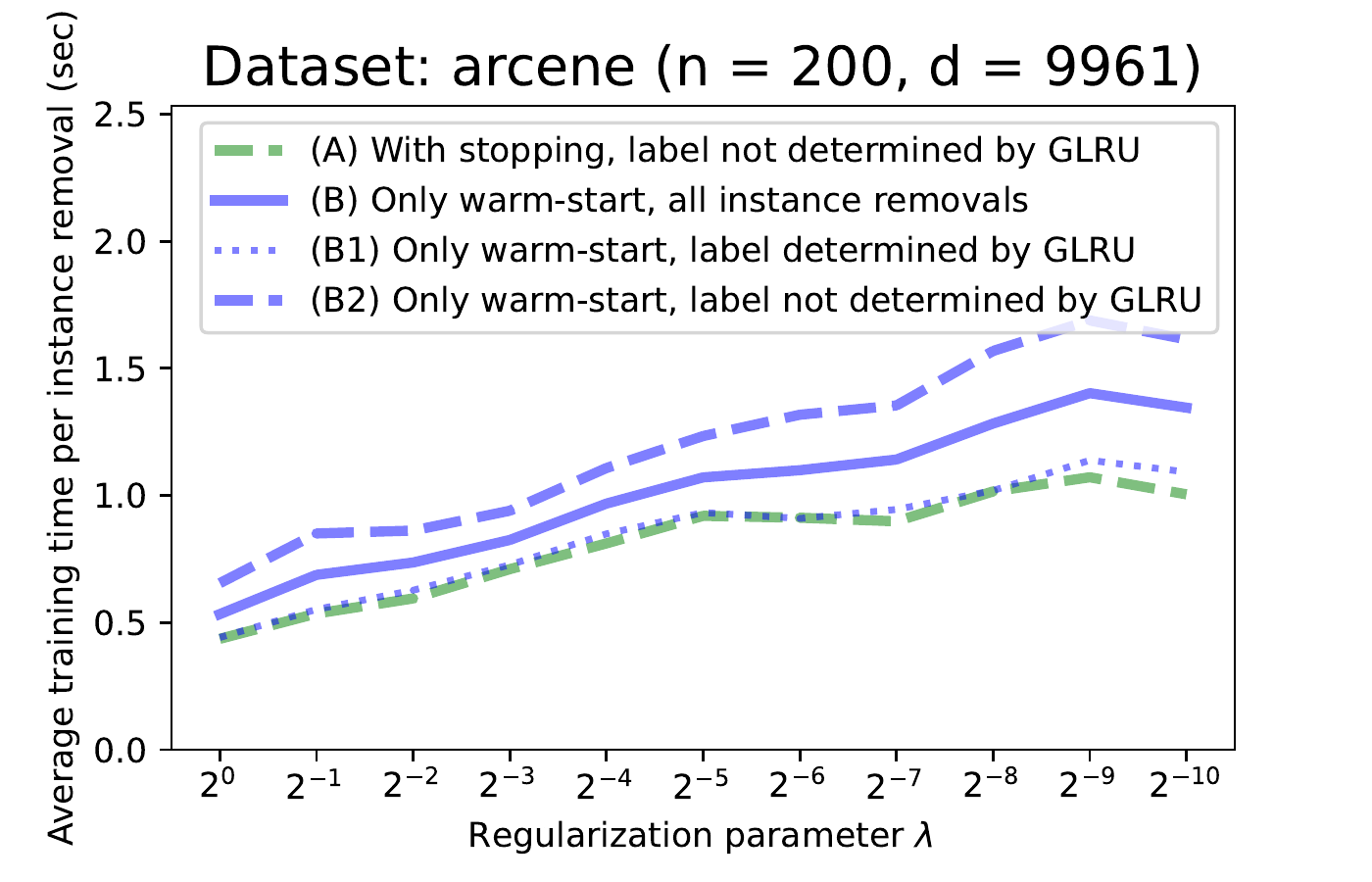}
\includegraphics[width=0.5\hsize]{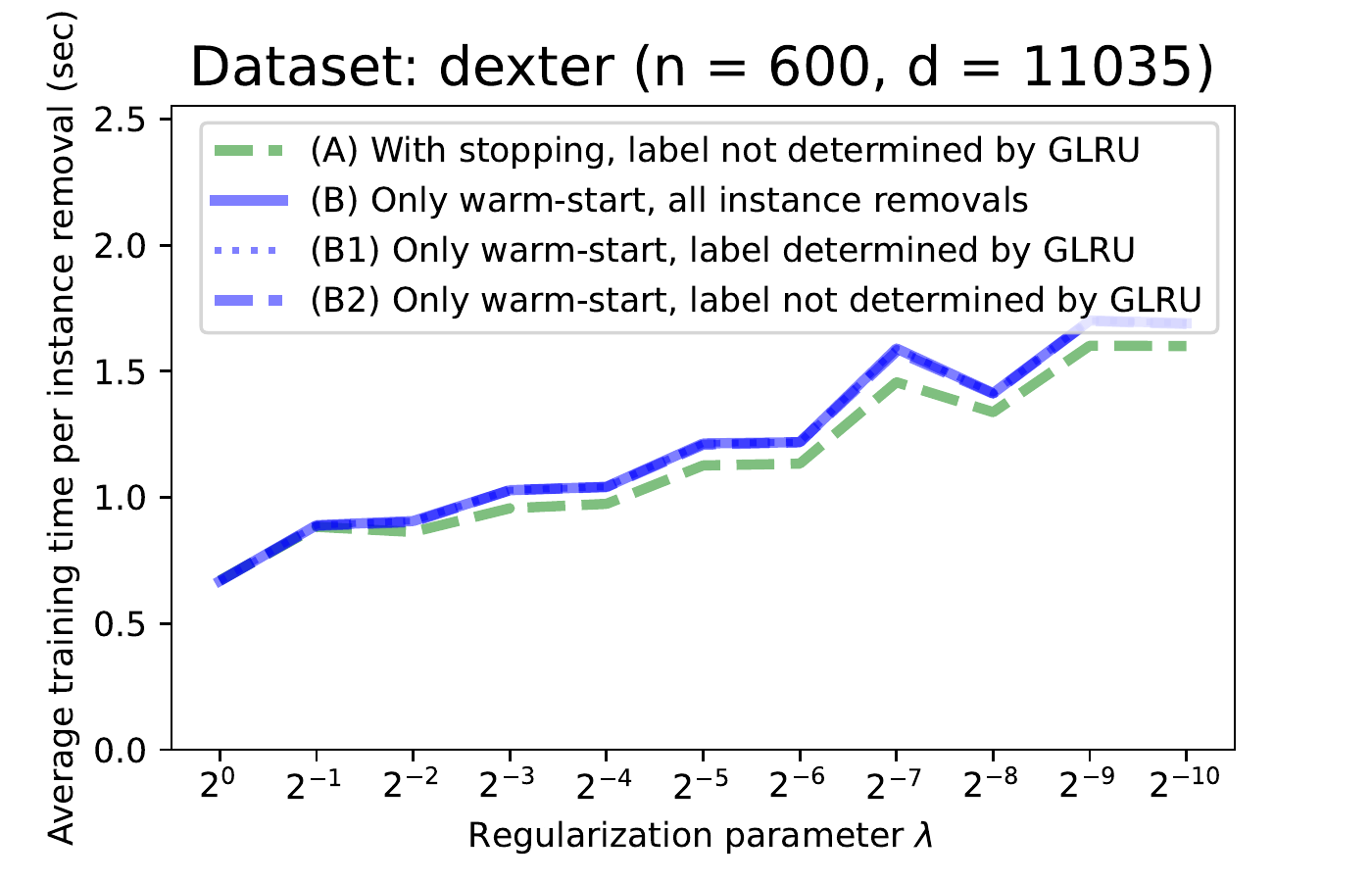}

\includegraphics[width=0.5\hsize]{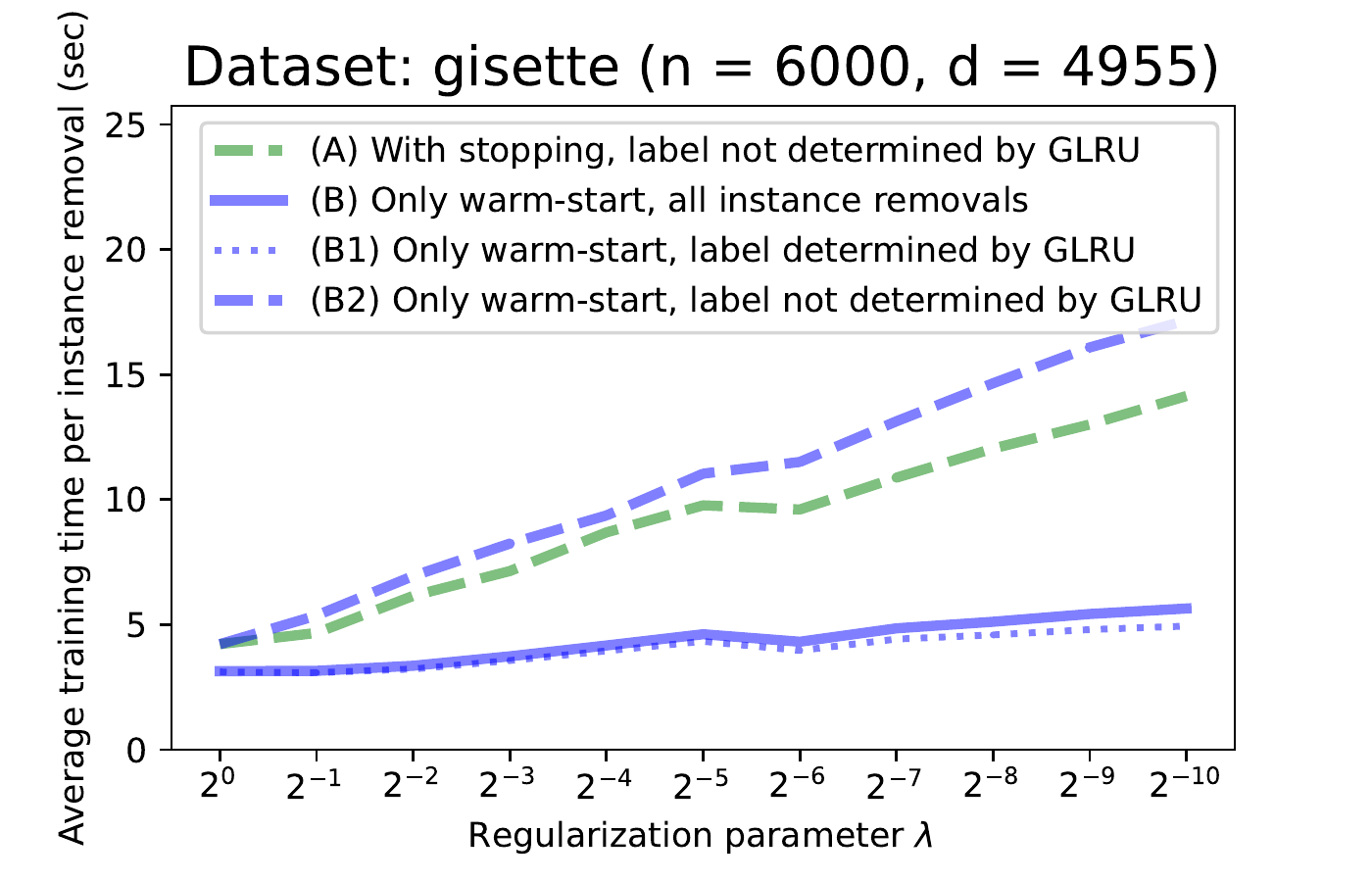}
\caption{The costs for training computations in LOOCV, on average per instance removal. As (A)/(B) becomes smaller, the ratio in ``Exact-GLRU'' (Table \ref{tb:exp-loocv}) becomes accordingly smaller than the one in ``Computation of GLRU -- Number of trainings''. Here (A) is usually smaller than (B2), but not always smaller than (B1) (and therefore (B)) depending on the datasets.}
\label{fig:ratio-warmstart-stopping}
\end{figure}

\subsection{Stepwise feature elimination} \label{sec:exp-stepwise}

\begin{table}[tp]
\caption{Computational costs of stepwise feature elimination by naive and proposed (with GLRU) methods, for the first feature removal.}
\label{tb:exp-stepwise}
\begin{itemize}
\item ``Prep'' (preprocessing) denotes the cost for training for the dataset before feature removals (line \ref{alg:stepwise-GLRU:initial-training} in Algorithm \ref{alg:stepwise-GLRU}).
\item ``(ratio)'' denotes the ratio of time with GLRU to the one without GLRU.
\item ``Number of trainings'' denotes the number of instance removals which the training computation of line \ref{alg:stepwise-GLRU-train2} in Algorithm \ref{alg:stepwise-GLRU} are needed.
\end{itemize}
\begin{center}
{\scriptsize
\begin{tabular}{c|rrrr|rrr|c}
\hline
          &              &                & &                                   & \multicolumn{3}{c|}{Computation of GLRU} \\
          & \ccell{Prep} & \ccell{Naive}  & \multicolumn{2}{c|}{With GLRU [sec]}& Time  & \multicolumn{2}{c|}{Number of} & \#Test errors\\
$\lambda$ & \ccell{[sec]}& \ccell{[sec]}  & \multicolumn{2}{c|}{(ratio)}        & [sec] & \multicolumn{2}{c|}{trainings} & (out of $n/10$) \\
\hline
\multicolumn{9}{l}{Dataset {\tt gisette} ($d=4955$, $n=6000$)} \\
\hline
$2^0$    &  19.12 & {\sf 11431.9} & {\sf  9946.7} & (87.0\%) & {\sf 6.29} & {\sf  2833} & (57.2\%) & 16 \\
$2^{-1}$ &  26.47 & {\sf 15545.5} & {\sf  6928.2} & (44.6\%) & {\sf 4.84} & {\sf  1438} & (29.0\%) & 11 \\
$2^{-2}$ &  37.11 & {\sf 19328.6} & {\sf  7211.7} & (37.3\%) & {\sf 4.72} & {\sf  1270} & (25.6\%) & 10 \\
$2^{-3}$ &  46.64 & {\sf 25486.7} & {\sf 23048.3} & (90.4\%) & {\sf 4.66} & {\sf  3473} & (70.1\%) & 10 \\ 
$2^{-4}$ &  61.52 & {\sf 31060.8} & {\sf 24000.3} & (77.3\%) & {\sf 4.42} & {\sf  2965} & (59.8\%) & 9  \\
$2^{-5}$ &  76.35 & {\sf 38363.7} & {\sf 35414.4} & (92.3\%) & {\sf 4.37} & {\sf  3797} & (76.6\%) & 10 \\
$2^{-6}$ & 104.63 & {\sf 44127.5} & {\sf 42928.4} & (97.3\%) & {\sf 4.34} & {\sf  4179} & (84.3\%) & 10 \\
$2^{-7}$ & 138.16 & {\sf 48223.7} & {\sf 46431.6} & (96.3\%) & {\sf 4.32} & {\sf  4172} & (84.2\%) & 10 \\
$2^{-8}$ & 173.68 & {\sf 55116.0} & {\sf 51843.9} & (94.1\%) & {\sf 4.30} & {\sf  4085} & (82.4\%) & 10 \\
$2^{-9}$ & 233.01 & {\sf 58589.5} & {\sf 57271.4} & (97.8\%) & {\sf 4.29} & {\sf  4331} & (87.4\%) & 11 \\
$2^{-10}$& 305.74 & {\sf 65100.8} & {\sf 63260.1} & (97.2\%) & {\sf 4.20} & {\sf  4302} & (86.8\%) & 12 \\
\hline
\multicolumn{9}{l}{Dataset {\tt arcene} ($d=9961$, $n=200$)} \\
\hline
$2^0$    &   2.08 & {\sf  3448.2} & {\sf  2574.9} & (74.7\%) & {\sf 0.32} & {\sf  7022} & (70.5\%) & 6 \\
$2^{-1}$ &   2.53 & {\sf  3783.3} & {\sf  1325.4} & (35.0\%) & {\sf 0.29} & {\sf  3309} & (33.2\%) & 6 \\
$2^{-2}$ &   3.18 & {\sf  4014.3} & {\sf   706.8} & (17.6\%) & {\sf 0.29} & {\sf  1617} & (16.2\%) & 6 \\
$2^{-3}$ &   3.28 & {\sf  4204.4} & {\sf  1609.6} & (38.3\%) & {\sf 0.29} & {\sf  3486} & (35.0\%) & 6 \\
$2^{-4}$ &   3.96 & {\sf  4389.8} & {\sf  3465.0} & (78.9\%) & {\sf 0.27} & {\sf  7417} & (74.5\%) & 6 \\
$2^{-5}$ &   4.54 & {\sf  4555.8} & {\sf   506.3} & (11.1\%) & {\sf 0.27} & {\sf   760} &  (7.6\%) & 5 \\
$2^{-6}$ &   4.74 & {\sf  4742.6} & {\sf  2175.3} & (45.9\%) & {\sf 0.27} & {\sf  4378} & (44.0\%) & 5 \\
$2^{-7}$ &   5.38 & {\sf  4912.0} & {\sf  3200.1} & (65.1\%) & {\sf 0.27} & {\sf  5737} & (57.6\%) & 5 \\
$2^{-8}$ &   6.09 & {\sf  5332.1} & {\sf  2715.5} & (50.9\%) & {\sf 0.28} & {\sf  4748} & (47.7\%) & 5 \\
$2^{-9}$ &   6.16 & {\sf  5671.4} & {\sf  3510.7} & (61.9\%) & {\sf 0.27} & {\sf  5621} & (56.4\%) & 5 \\
$2^{-10}$&   6.83 & {\sf  6095.0} & {\sf  3239.9} & (53.2\%) & {\sf 0.27} & {\sf  4893} & (49.1\%) & 5 \\
\hline
\multicolumn{9}{l}{Dataset {\tt dexter} ($d=11035$, $n=600$)} \\
\hline
$2^0$    &   2.81 & {\sf  6132.8} & {\sf  5980.0} & (97.5\%) & {\sf 1.12} & {\sf  9561} & (86.6\%) & 5 \\
$2^{-1}$ &   3.59 & {\sf  6476.2} & {\sf   352.4} &  (5.4\%) & {\sf 0.94} & {\sf   535} &  (4.8\%) & 4 \\
$2^{-2}$ &   3.83 & {\sf  5012.7} & {\sf   614.7} & (12.3\%) & {\sf 0.94} & {\sf   762} &  (6.9\%) & 4 \\
$2^{-3}$ &   4.85 & {\sf  8192.7} & {\sf   866.9} & (10.6\%) & {\sf 0.88} & {\sf  1000} &  (9.1\%) & 4 \\
$2^{-4}$ &   5.07 & {\sf  8651.2} & {\sf  1215.8} & (14.1\%) & {\sf 0.88} & {\sf  1301} & (11.8\%) & 4 \\
$2^{-5}$ &   5.99 & {\sf  9390.5} & {\sf  1579.4} & (16.8\%) & {\sf 0.89} & {\sf  1475} & (13.4\%) & 4 \\
$2^{-6}$ &   7.02 & {\sf 10663.6} & {\sf  2214.7} & (20.8\%) & {\sf 0.89} & {\sf  1930} & (17.5\%) & 4 \\
$2^{-7}$ &   6.85 & {\sf 10806.3} & {\sf  3120.7} & (28.9\%) & {\sf 0.91} & {\sf  2350} & (21.3\%) & 4 \\
$2^{-8}$ &   7.75 & {\sf 12570.8} & {\sf  3526.7} & (28.1\%) & {\sf 0.83} & {\sf  2765} & (25.1\%) & 4 \\
$2^{-9}$ &   8.77 & {\sf 13499.1} & {\sf 13415.9} & (99.4\%) & {\sf 0.82} & {\sf 10208} & (92.5\%) & 5 \\
$2^{-10}$&   8.88 & {\sf 14065.4} & {\sf 13967.1} & (99.3\%) & {\sf 0.83} & {\sf 10112} & (91.6\%) & 5 \\
\hline
\end{tabular}
}
\end{center}
\end{table}

Next we present the effect of GLRU on stepwise feature elimination, as explained in Section \ref{sec:GLRU-stepwise}.
We used the dataset in Table \ref{tb:datasets-loocv-stepwise} (the same as those in LOOCV).

In the experiment, we compared the computational costs of the proposed method (Algorithm \ref{alg:stepwise-GLRU}) and naive LOOCV (Algorithm \ref{alg:ordinary-stepwise}) employing only a warm-start (line \ref{alg:ordinary-stepwise-train}).
As stated in line \ref{alg:stepwise-GLRU:end-training} in Algorithm \ref{alg:stepwise-GLRU}, training computations can be omitted when the lower bound of the validation errors for removing a feature becomes equal to or larger than the validation errors for removing another feature.

For simplicity, we compared the computational costs of the first feature elimination.
The results are presented in Table \ref{tb:exp-stepwise}.
We can see that, in 17 out of 33 cases (three datasets $\times$ 11 $\lambda$'s)
the numbers of feature removals that require training computations are reduced to less than 50\%;
computational costs are accordingly reduced.
Furthermore, we can confirm that the computational cost of computing the bounds is significantly smaller than the training costs.

We note that, in contrast to LOOCV, the numbers of omitted training computations are not always reduced by increasing $\lambda$.
This phenomenon can be explained as follows.
As explained in Section \ref{sec:GLRU-stepwise},
we can omit the training of $\bm{w}^{*(S\setminus\{j\})}$
when there exists a feature removal whose true test error computed so far ($E[j_\mathrm{best}]$ in line \ref{alg:stepwise-GLRU:end-training}, Algorithm \ref{alg:stepwise-GLRU})
is smaller than the lower bound of that for $j$ ($I[j]$ in line \ref{alg:stepwise-GLRU:end-training} of Algorithm \ref{alg:stepwise-GLRU}).
Therefore, even if we use large $\lambda$ to provide tighter lower bound $I[j]$,
the omission is delayed if there is no small $E[j_\mathrm{best}]$.
As shown in Table \ref{tb:exp-stepwise},
the computational cost of GLRU tends to be reduced significantly when the value of ``\#Test errors'' becomes small.
An example of {\tt arcene} dataset with $\lambda = 2^{-4}$ (\#Test errors = 6, training required for 74.5\% of features) and
$\lambda = 2^{-5}$ (\#Test errors = 5, training required for 7.6\% of features) are presented in Figure \ref{fig:stepwise-behavior}.

\begin{figure}[tp]
\includegraphics[width=0.5\hsize]{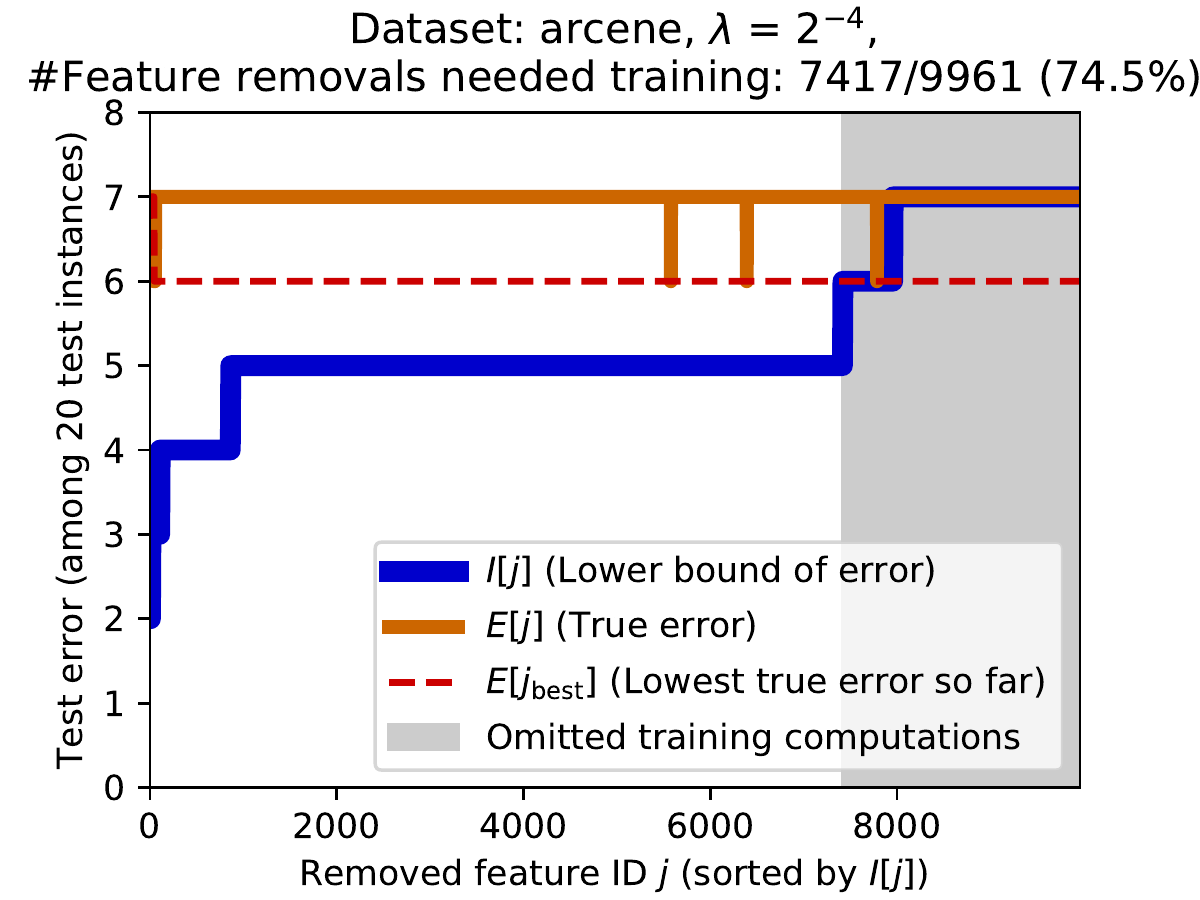}
\includegraphics[width=0.5\hsize]{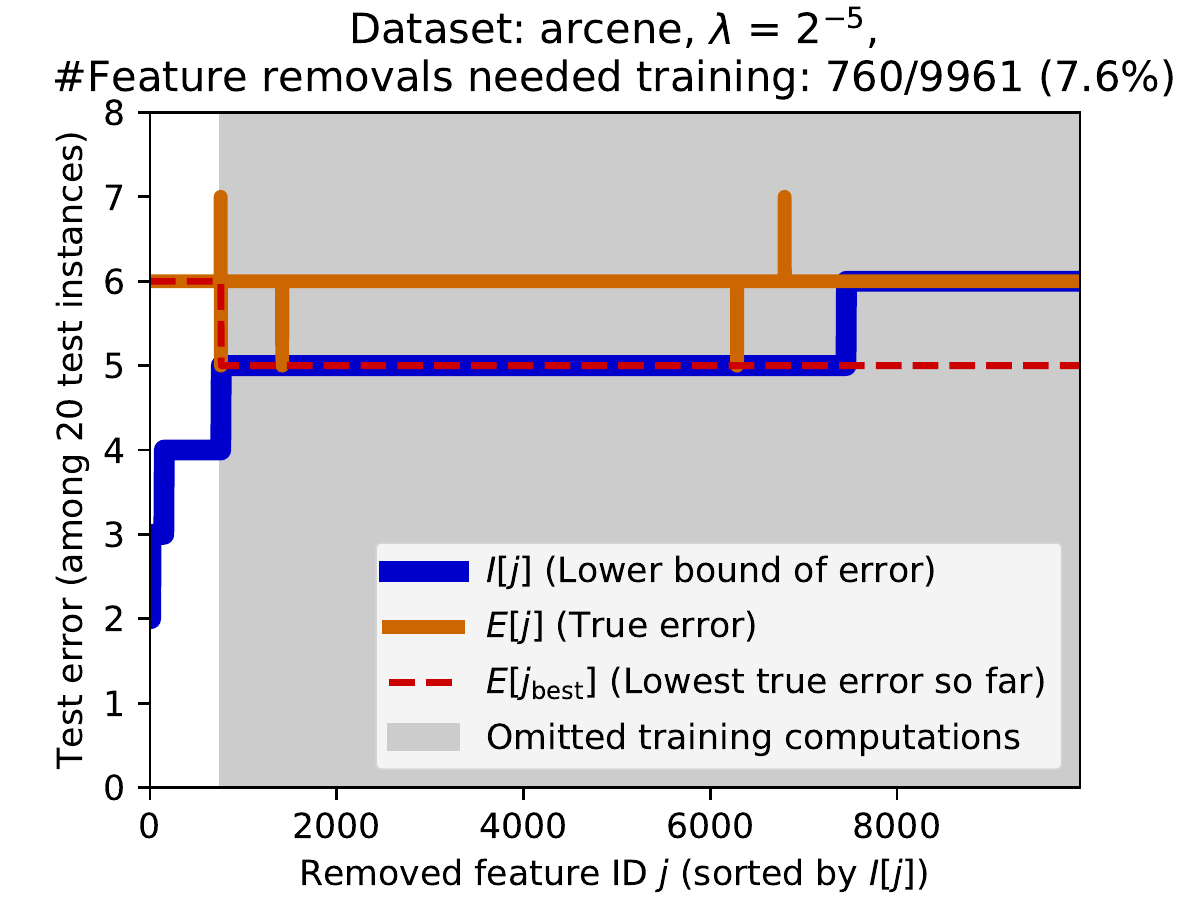}

\caption{Example of how the training computations being omitted in the stepwise feature eliminations with GLRU (Algorithm \ref{alg:stepwise-GLRU})
for the dataset {\tt arcene} and $\lambda = 2^{-4}, 2^{-5}$.
In contrast to the lower bounds of the test errors (blue lines) being not so different,
the test error computed in reality (orange lines) is lower in $\lambda=2^{-5}$ than in $\lambda=2^{-4}$.
As a result, the omission of training computations (line \ref{alg:stepwise-GLRU:end-training} in Algorithm \ref{alg:stepwise-GLRU}) starts earlier in the smaller $\lambda$ (gray area).}
\label{fig:stepwise-behavior}
\end{figure}

\subsection{Tightness of the bounds} \label{sec:exp-tightness}

\begin{figure*}[tp]
\begin{tabular}{cc}
\begin{minipage}{0.48\hsize}
\noindent
\includegraphics[width=\hsize]{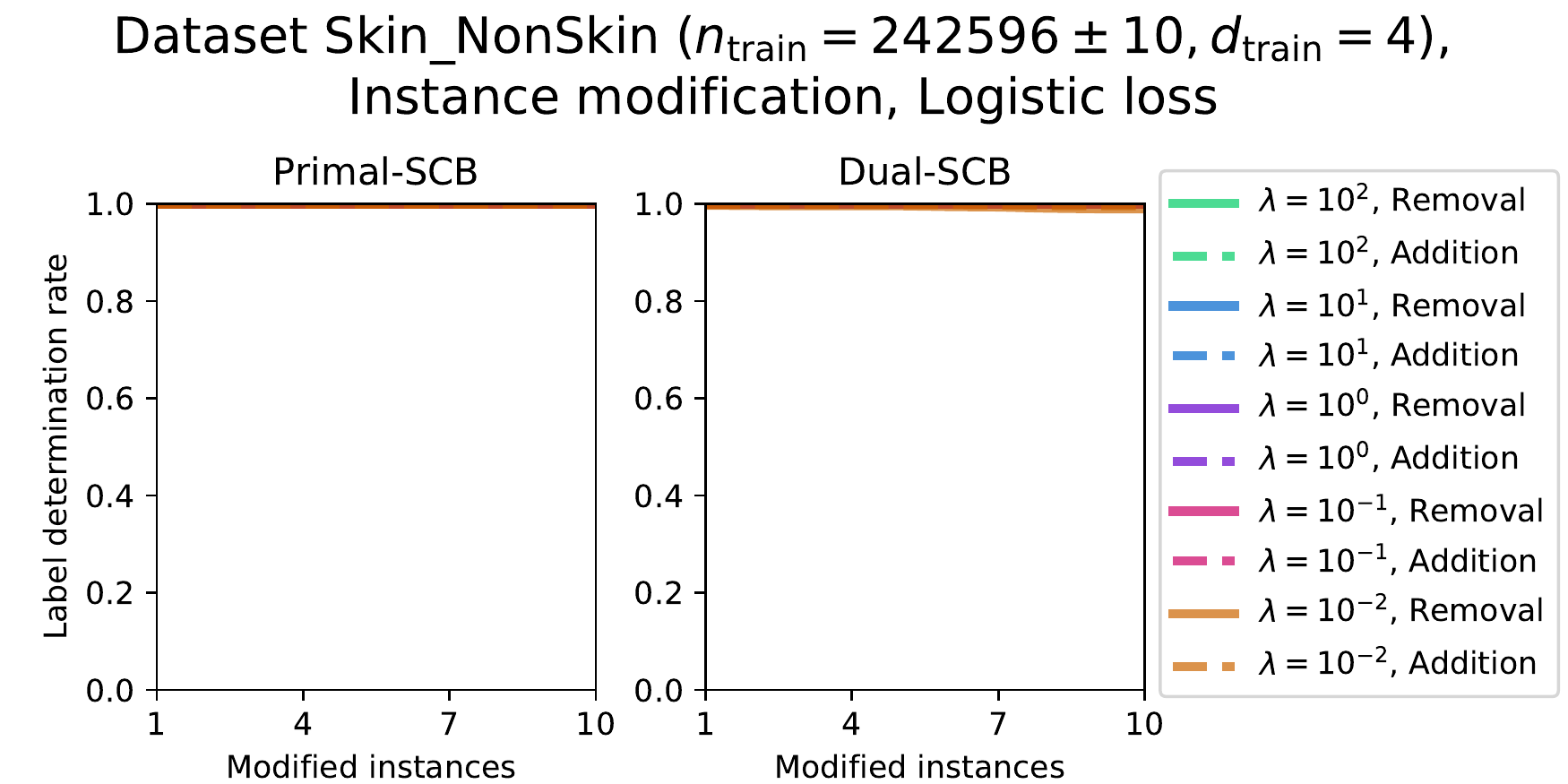}
\\
\noindent
\includegraphics[width=\hsize]{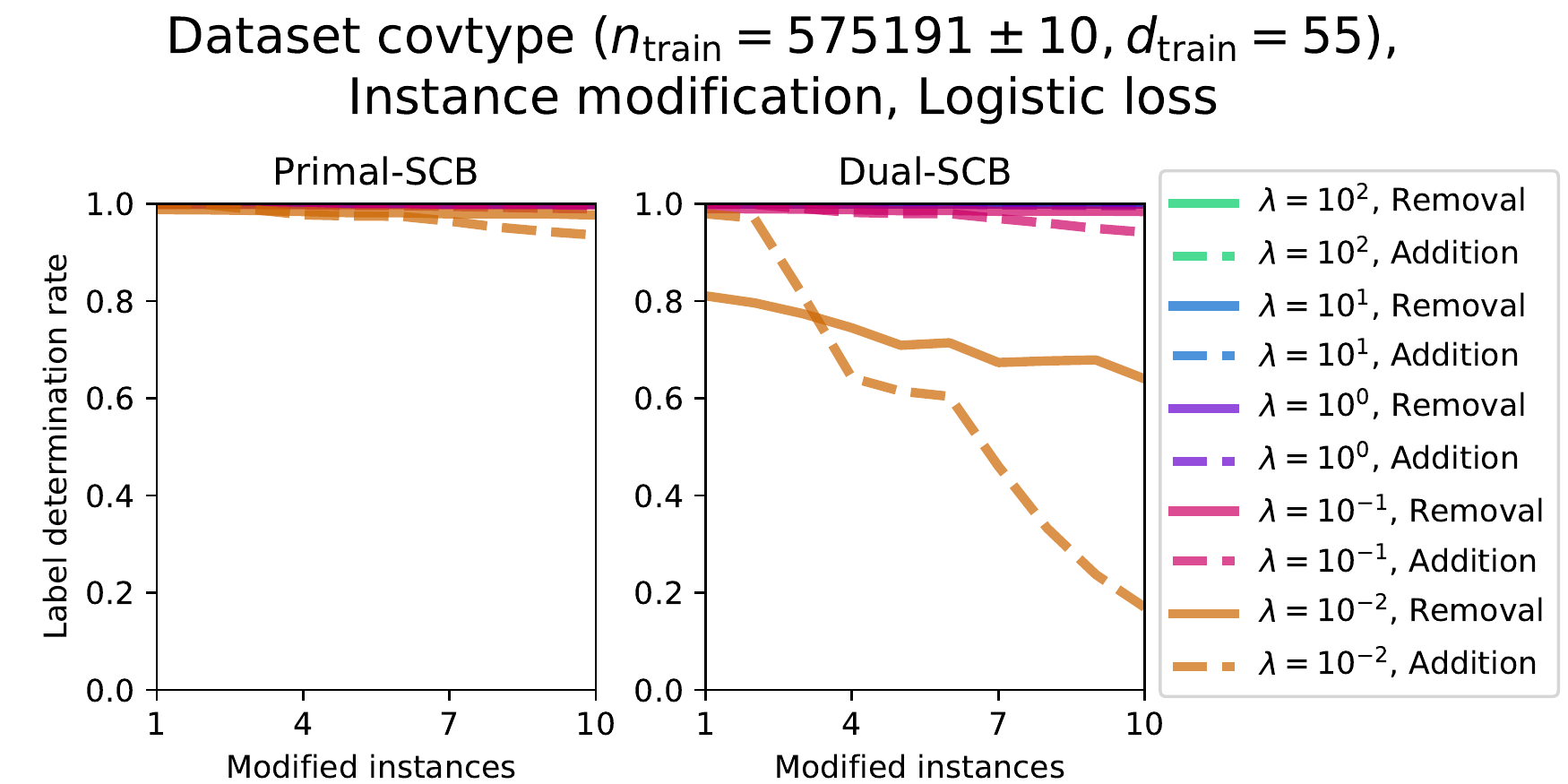}
\\
\noindent
\includegraphics[width=\hsize]{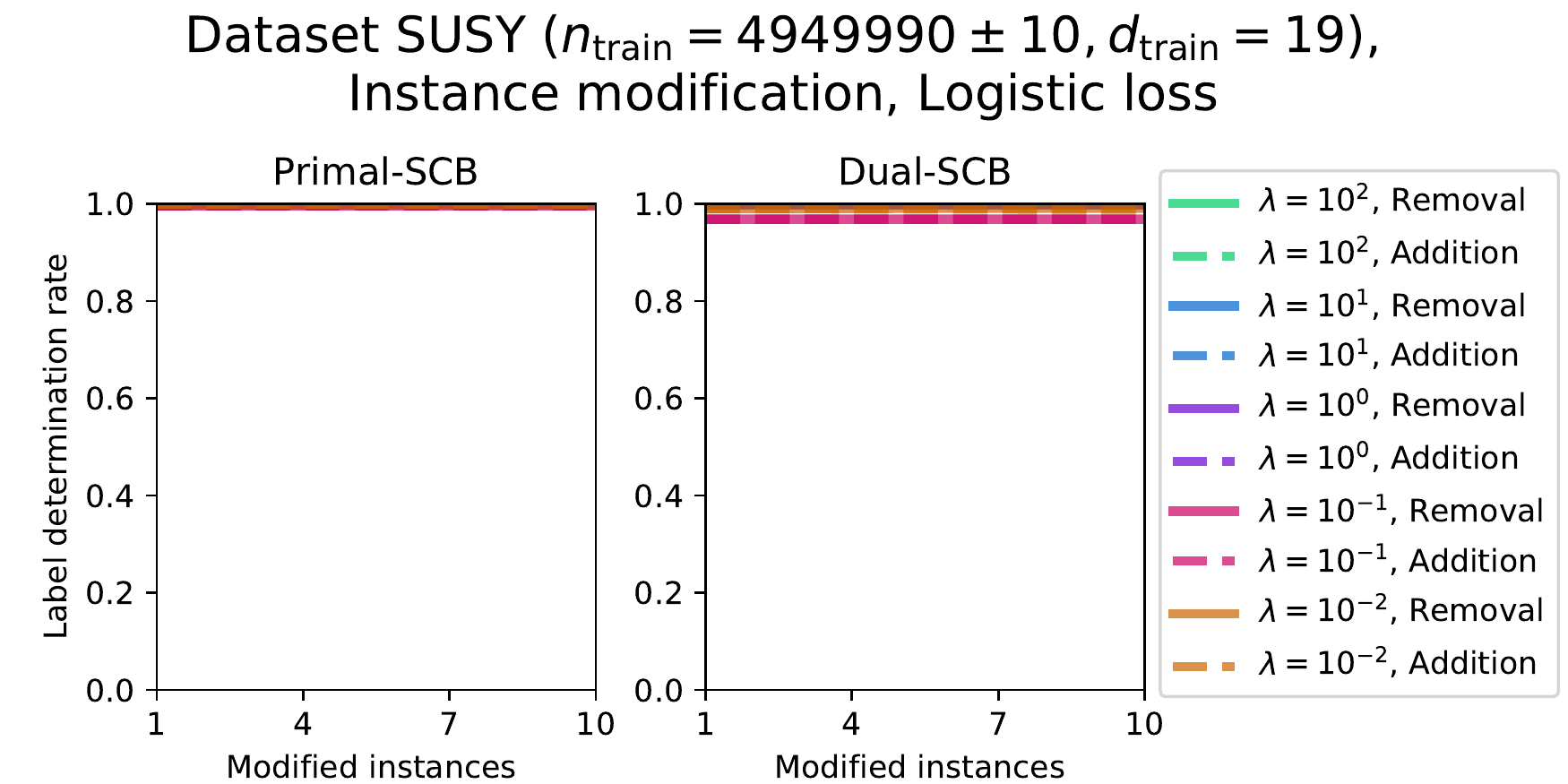}
\\
\noindent
\includegraphics[width=\hsize]{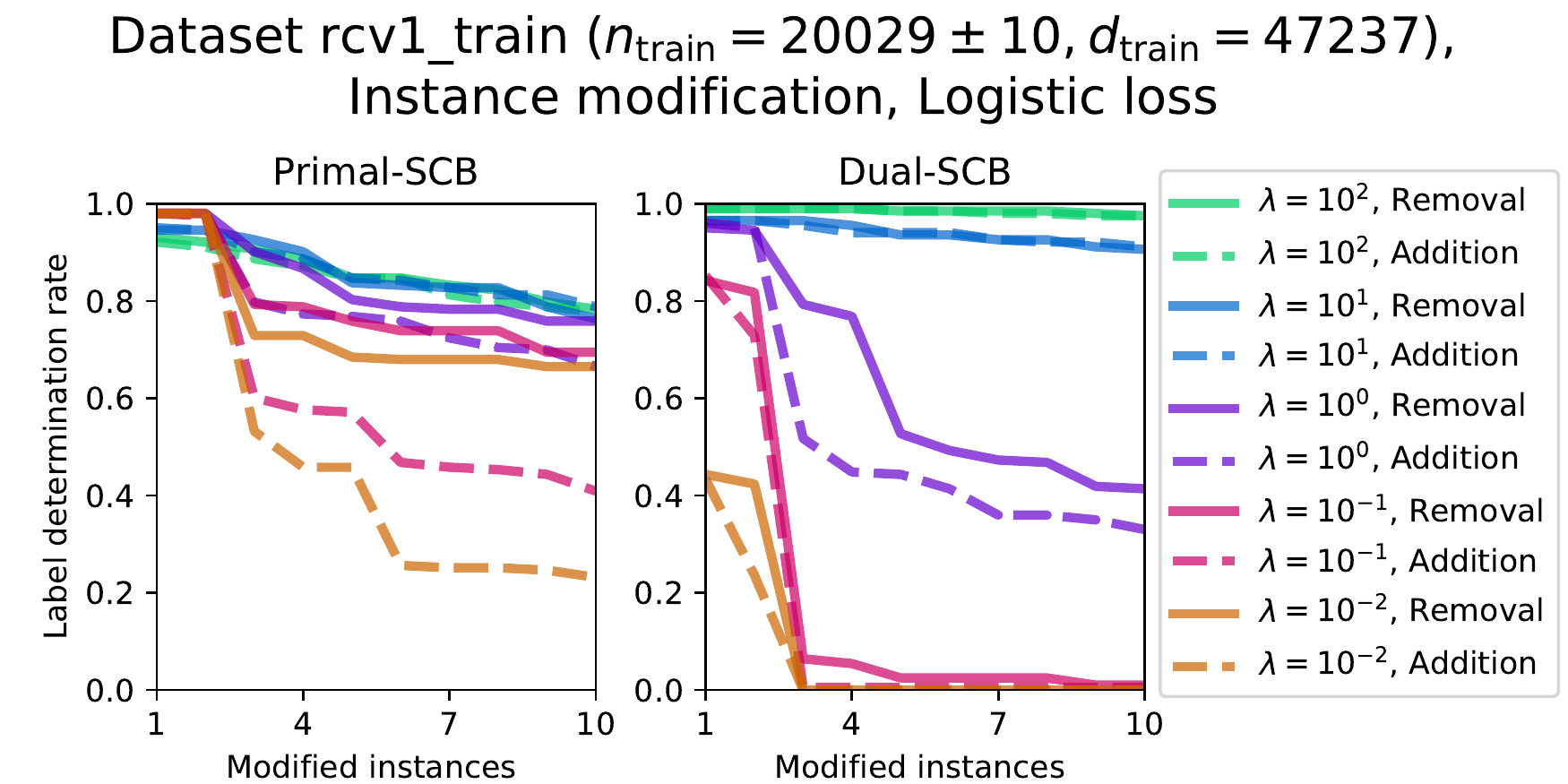}
\caption{Label determination rates for instance modifications}
\label{fig:modif-inst}
\end{minipage}
&
\begin{minipage}{0.48\hsize}
\noindent
\includegraphics[width=\hsize]{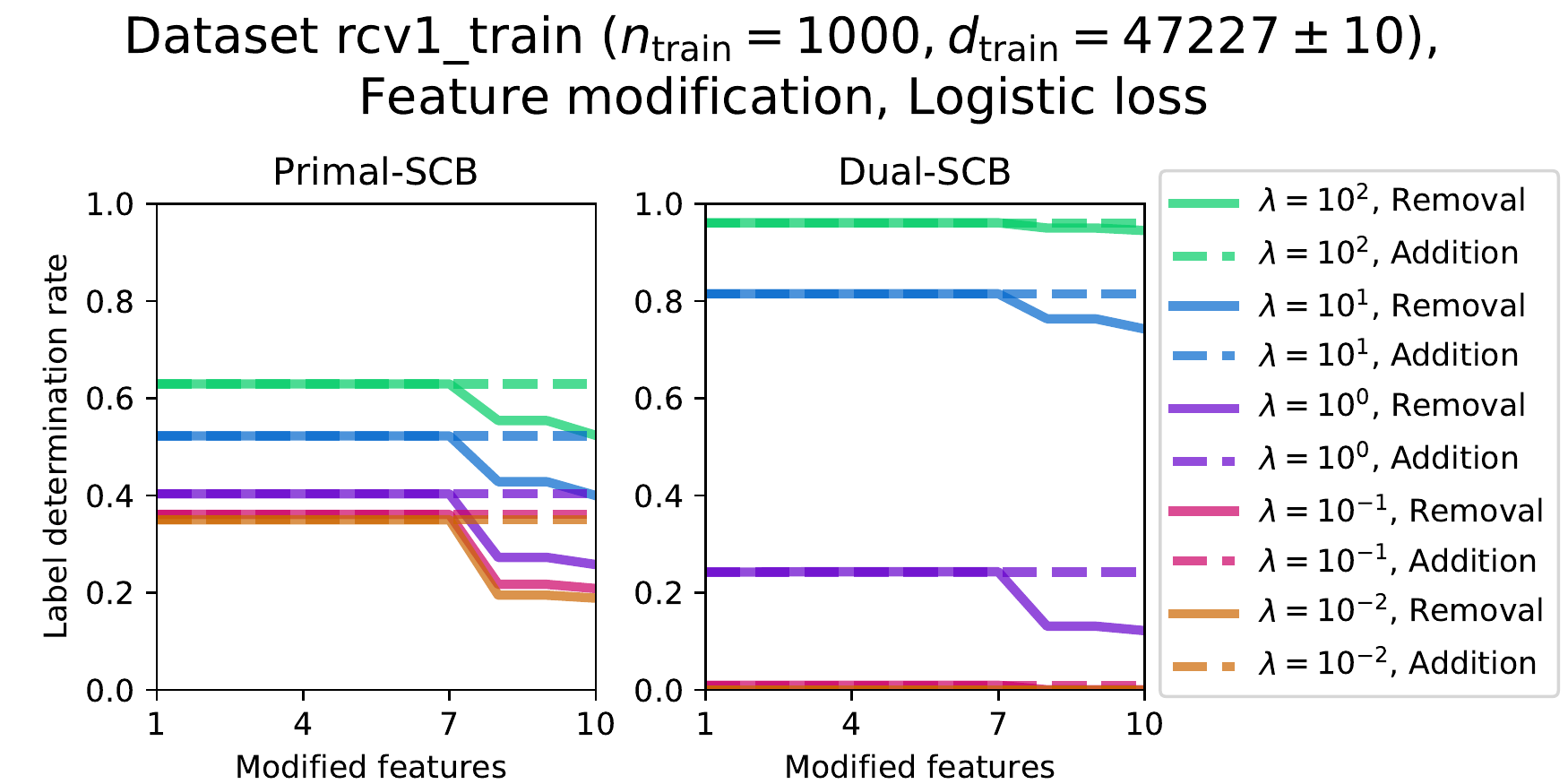}
\\
\noindent
\includegraphics[width=\hsize]{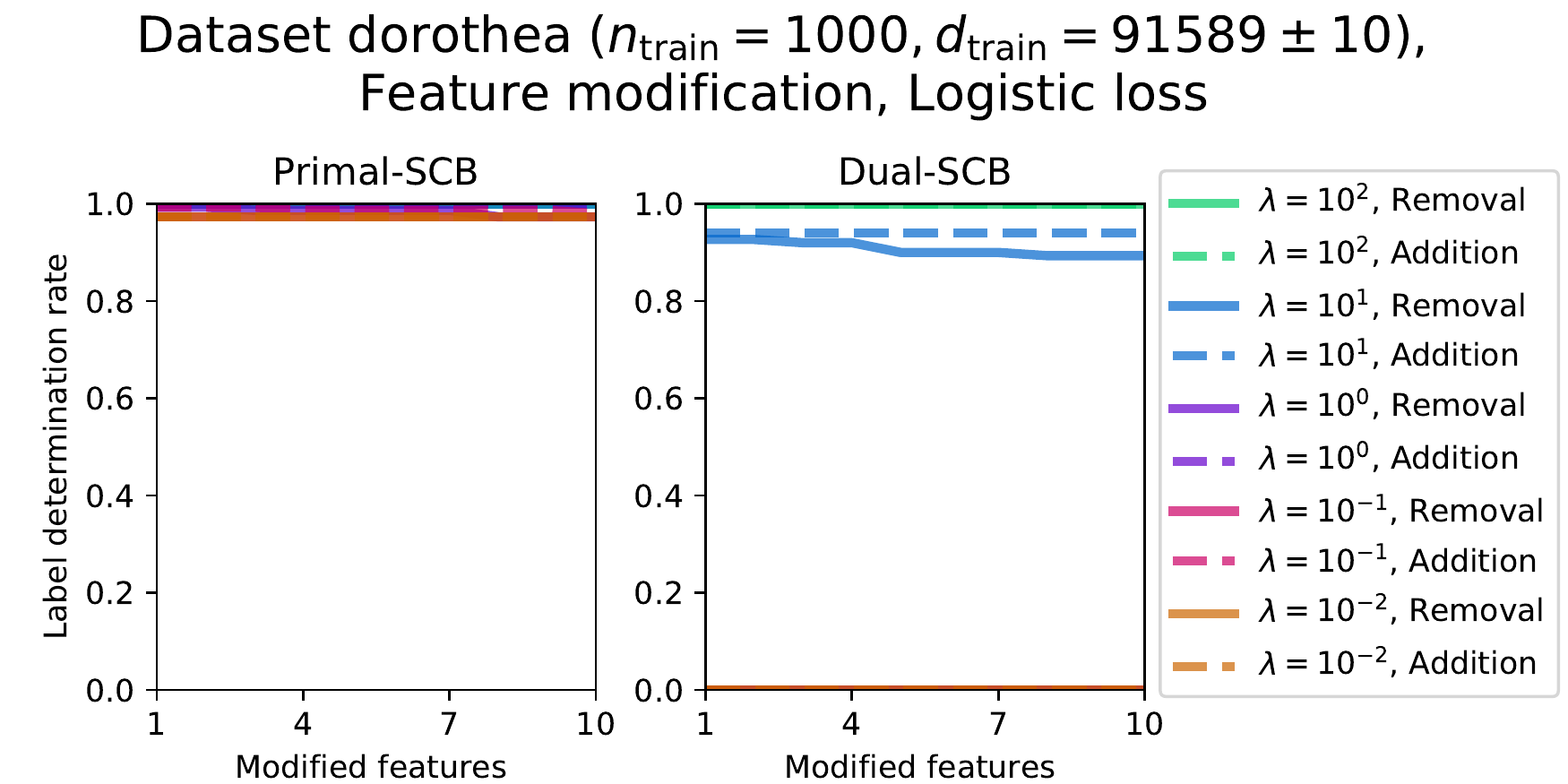}
\\
\noindent
\includegraphics[width=\hsize]{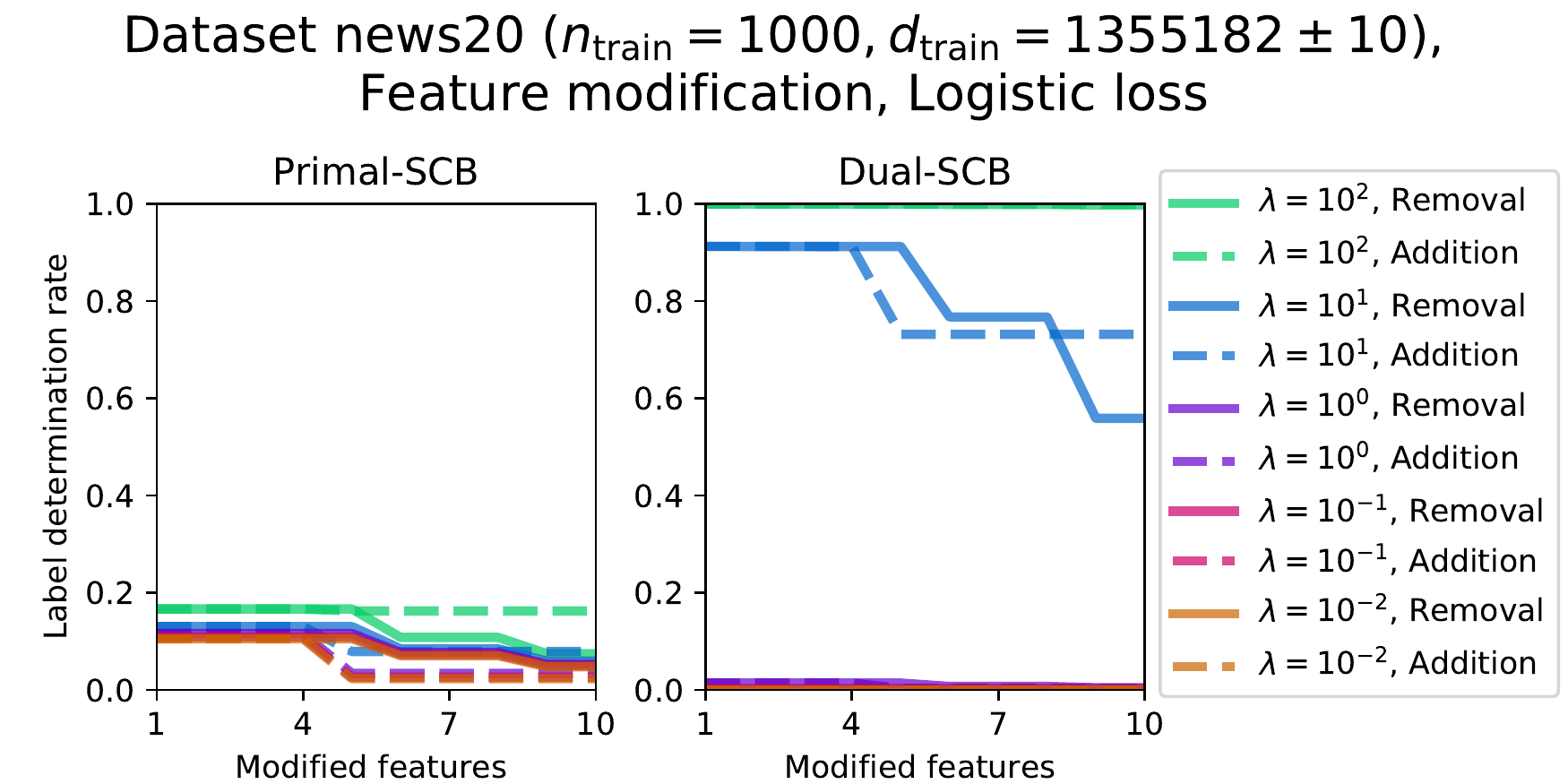}
\\
\noindent
\includegraphics[width=\hsize]{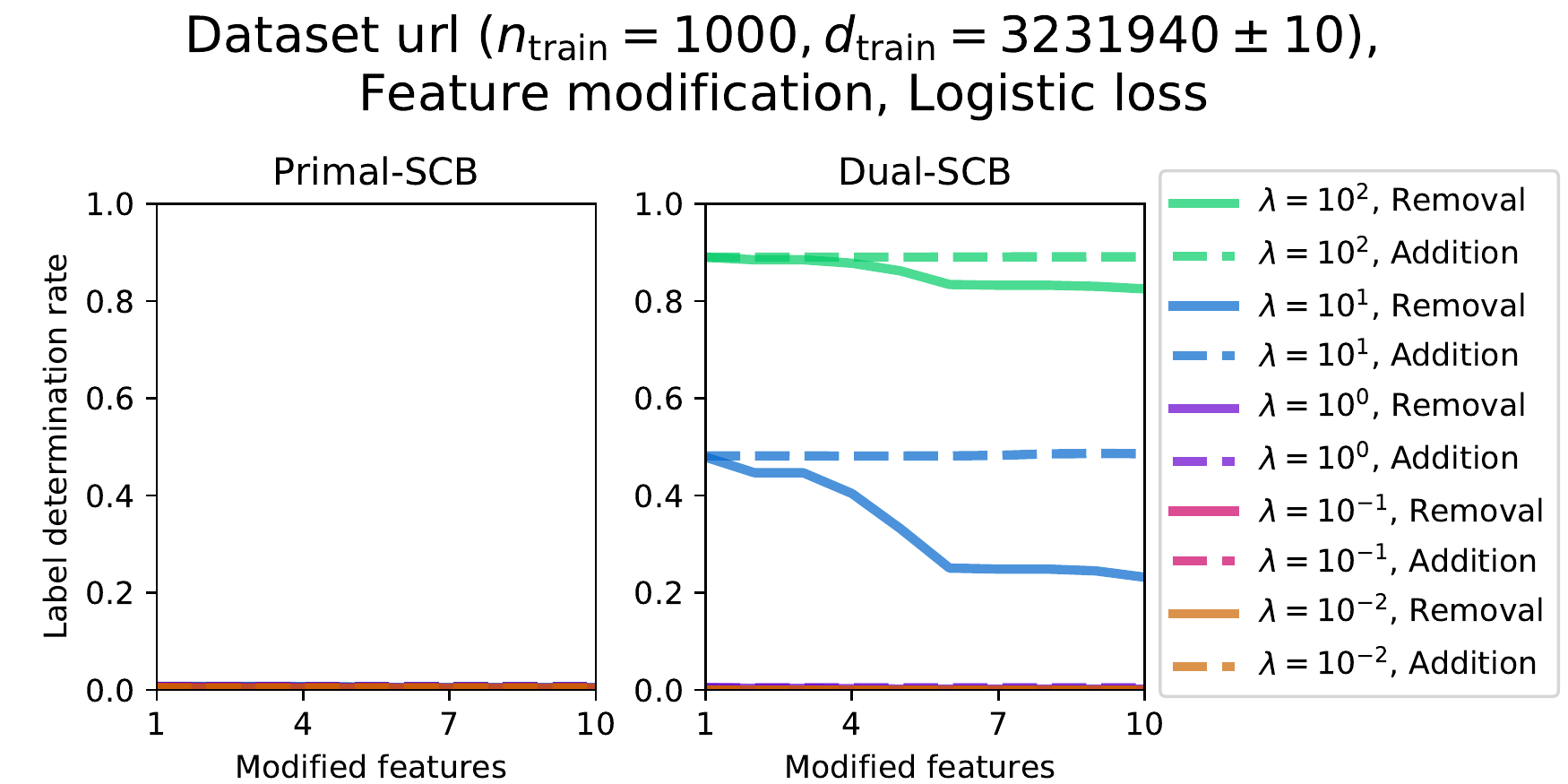}
\caption{Label determination rates for feature modifications}
\label{fig:modif-feat}
\end{minipage}
\end{tabular}
\end{figure*}

\subsubsection{Setup} \label{sec:exp-tightness-setup}

In this section, we examine the tightness of the bounds proposed in Section \ref{ch:bound-predict},
for various instance/feature modifications.
Specifically, we examine two bounds of $\Wnew$ in Corollary \ref{co:bounds-predict} (Primal-SCB and Dual-SCB) using the calculations presented in Section \ref{ch:bound-fast}.

For the experiments of instance modifications, for each dataset in Table \ref{tab:datasets} we used 99\% of instances as the training set and the rest as the test set. For the experiments of feature modifications, we used 1,000 instances as the training set and the rest as the test set. (For only the dataset {\tt url}, we used not the rest of instances but 19,000 instances as the test set since it is large.)

We then assumed the addition or removal of either instances or features from 1 to 10.
For each of these data modifications, we first conducted training, that is, we computed $\Wold$ and $\Alphaold$ by solving \eqref{eq:primal}. Then we computed the bounds of $\Wnew$ using Primal-SCB and Dual-SCB in Section \ref{ch:bound-fast} for these modifications.

To examine the tightness of the bounds of $\bm{x}^\top\Wnew$,
we use the {\em label determination rate} of the binary classification.
Given a test set $\{\bm{x}_i\}_{i\in[m]}$ ($\bm{x}_i\in\mathbb{R}^d$)
and bounds of the model parameters ${\cal W}\subset\mathbb{R}^d$ ($\Wnew\in{\cal W}$),
the label determination rate for the bound is determined by
\begin{align*}
\frac{1}{m}\sum_{i\in[m]} \mathbb{I}\left\{ \max_{\bm{w}\in{\cal W}}\bm{x}_i^\top\bm{w} < 0 \vee \min_{\bm{w}\in{\cal W}}\bm{x}_i^\top\bm{w} > 0 \right\}.
\end{align*}
This means the ratio of test instances whose binary classification results (signs of $\XtestT\Wnew$)
are determined without knowing the exact $\Wnew$.

\subsubsection{Results}

The results are shown in Figures \ref{fig:modif-inst} and \ref{fig:modif-feat}.
With all experiments, we can confirm that label determination rates decrease with the increase of modified instances or features (plots left to right), and by smaller $\lambda$ values (green to orange plots).
They are as what we expected: increasing the instance or feature modifications increases $\Gnew(\hat{\bm{w}}, \hat{\bm{\alpha}})$ in Theorem \ref{th:bounds}, and smaller $\lambda$ increases $r_P$ and $r_D$ in Theorem \ref{th:bounds}.

Furthermore, we can confirm which bounding strategy (Primal-SCB and Dual-SCB) tighten the bounds (higher label determination rate) depends not only on the datasets and $\lambda$'s:
For example, for datasets {\tt rcv1\_train} and {\tt news20}, Dual-SCB provides a higher label determination rate for larger $\lambda$ while Primal-SCB better for smaller $\lambda$.
We theoretically examined this fact in the following discussion.

\subsubsection{Theoretical discussions} \label{sec:exp-tightness-discussion}

In this section we theoretically demonstrate the tendencies of the experimental results discussed in the previous section.

To do this, we compare the size of bounds of the prediction result $\XtestT\Wnew$
between \eqref{eq:prediction-bound-primal} (Primal-SCB) and \eqref{eq:prediction-bound-dual} (Dual-SCB),
under the following conditions which the experiment used: L2-regularization $\rho(t) = (\lambda/2)t^2$,
and each feature having the L2-norm $\|X_{:j}\|_2 = \sqrt{n}$ (Table \ref{tab:datasets}).

For Primal-SCB, the size of the bound of the prediction result $\XtestT\Wnew$,
that is, the difference between the two bounds in \eqref{eq:prediction-bound-primal}, is
\begin{align}
2 r_P \|\Xtest\|_2 = \sqrt{\frac{8}{\lambda}\Gnew(\hat{\bm{w}}, \hat{\bm{\alpha}})} \|\Xtest\|_2 .
	\label{eq:bound-size-primal-SCB}
\end{align}
On the other hand, the size of the bound of the prediction result $\XtestT\Wnew$ for Dual-SCB,
that is, the difference between the two bounds in \eqref{eq:prediction-bound-dual},
is computed as follows under the approximation $\|\Xnew_{:j}\|_2 \approx \sqrt{\nnew}$:
\begin{align}
\sqrt{\frac{8\mu}{\lambda^2} \Gnew(\hat{\bm{w}}, \hat{\bm{\alpha}})} \|\Xtest\|_1.
	\label{eq:bound-size-dual-SCB-norm}
\end{align}
Proofs are given in Appendix \ref{app:exp-tightness-discussion-dual}.
This implies that the tightness of Dual-SCB is more strongly affected by $\lambda$ than Primal-SCB.

\section{Conclusion}

We assumed the case when the training dataset has been slightly modified,
and considered deriving a possible region of the trained model parameters
and then identifying the lower and upper bounds of prediction results.
The proposed GLRU method achieved this for the training methods formulated as
regularized empirical risk minimization
and derived conditions such that GLRU can provide the amount of updates (Section \ref{ch:GLRU-method}).
Moreover, we applied GLRU to LOOCV which requires a large number of training computations
by removing an instance from the dataset (Section \ref{sec:GLRU-LOOCV}),
and to stepwise feature elimination which requires a large number training computations
by removing a feature from the dataset (Section \ref{sec:GLRU-stepwise}).
Experiments demonstrated that GLRU can significantly reduce the computation time for
LOOCV (Section \ref{sec:exp-loocv}) and stepwise feature elimination (Section \ref{sec:exp-stepwise}).
Furthermore, we analyzed two GLRU strategies (Primal-SCB and Dual-SCB)
for tightness of bounds and
their dependency on the strength of regularization $\lambda$ (Section \ref{sec:exp-tightness}).

In the future work, our large interest is how much we can relax the constraint on the training method, which is currently available for regularized empirical risk minimization with linear predictions and convex losses and regularizations.
Possible extensions include handling regularization functions that are not separated (Section \ref{sec:learning}), and handling convex but not strongly convex problems (i.e., the loss function is convex but nonsmooth, and the regularization function is convex but non-strongly-convex).

\subsection*{Acknowledgments}
This work was partially supported by MEXT KAKENHI (20H00601), JST CREST (JPMJCR21D3, JPMJCR21D3), JST Moonshot R\&D (JPMJMS2033-05), JST AIP Acceleration Research (JPMJCR21U2), NEDO (JPNP18002, JPNP20006) and RIKEN Center for Advanced Intelligence Project.

\clearpage
\bibliographystyle{unsrt}
\bibliography{paper-GLRU}

\clearpage
\appendix
\section*{Appendix}

\section{Examples of loss functions and regularization functions} \label{app:losses-regularizations}

In this appendix, we present some famous examples of loss and regularization functions
for regularized ERMs (Section \ref{sec:learning}) available for the proposed method,
including the derivatives and convex conjugates required for the proposed method.

In this appendix, for notational simplicity, we describe a subgradient as a scalar if it is a singleton set.

\subsection{Loss functions for regressions}

We assume $y\in\mathbb{R}$.

\begin{itemize}
\item {\bf Squared loss} \\
	$\ell_y(t) = \frac{1}{2}(t - y)^2$ \\
	$\ell_y(t)$: 1-smooth. \\
	$\partial\ell_y(t) = t - y$,\quad
	$\ell_y^*(t) = \frac{1}{2}t(t+2y)$.
\item {\bf Huber loss} ($\gamma > 0$: hyperparameter)\\
	$\ell_y(t) = \frac{1}{2}\left\{(t - y)^2 - ([|t - y| - \gamma]_+)^2\right\}$ \\
	$\ell_y(t)$: 1-smooth. \\
	$\partial\ell_y(t) = \mathrm{sign}(t-y)\min\{\gamma, |t-y|\}$. \\
	$\displaystyle\ell_y^*(t) = \begin{cases}
		\frac{1}{2}t(t+2y) & (|t|\leq\gamma) \\
		+\infty & (\text{otherwise})
		\end{cases}$.
\end{itemize}

\subsection{Loss functions for binary classifications}

We assume $y\in\{-1, +1\}$.

\begin{itemize}
\item {\bf Squared hinge loss} \\
	$\ell_y(t) = ([1 - yt]_+)^2$. \\
	$\ell_y(t)$: 2-smooth. \\
	$\partial\ell_y(t) = -2y[1 - yt]_+$. \\
	$\displaystyle\ell_y^*(t) = \begin{cases}
		\frac{t^2 + 4yt}{4} & (yt \leq 0) \\
		+\infty & (\text{otherwise})
		\end{cases}$
\item {\bf Smoothed hinge loss} ($\gamma > 0$: hyperparameter) \\
	$\ell_y(t) := \begin{cases}
		0, & \text{(if $t \geq 1$)} \\
		\frac{1}{2\gamma}(1 - yt)^2, & \text{(if $1 - \gamma \leq yt \leq 1$)} \\
		1 - yt - \frac{\gamma}{2}. & \text{(if $yt \leq 1 - \gamma$)}
	\end{cases}$ \\
	$\ell_y(t)$: $\gamma$-smooth. \\
	$\partial\ell_y(t) = \begin{cases}
		0, & \text{(if $yt \geq 1$)} \\
		-\frac{y}{\gamma}(1 - yt), & \text{(if $1 - \gamma \leq yt \leq 1$)} \\
		-y. & \text{(if $yt \leq 1 - \gamma$)}
	\end{cases}$ \\
	$\ell_y^*(t) = \begin{cases}
		\frac{1}{2\gamma}[(\gamma yt + y)^2 - 1], & \text{(if $-1\leq yt\leq 0$)} \\
		+\infty. & (\text{otherwise})
	\end{cases}$
\item {\bf Logistic loss} \\
	$\ell_y(t) = \log(1 + e^{-yt})$. \\
	$\ell_y(t)$: (1/4)-smooth. \\
	$\displaystyle \partial\ell_y(t) = -\frac{e^{-yt}}{1 + e^{-yt}}$. \\
	$\ell_y^*(t) = \begin{cases}
		0, & \text{(if $yt\in\{-1, 0\}$)} \\
		\multicolumn{2}{l}{(yt+1)\log(yt+1) - yt\log(-yt),} \\& \text{(if $-1 < yt < 0$)} \\
		+\infty. & \text{(otherwise)} \\
	\end{cases}$
\end{itemize}

\subsection{Regularization functions} \label{ch:regfuncs}

\begin{itemize}
\item {\bf L2-regularization} ($\lambda>0$: hyperparameter) \\
	$\rho_j(t) = \frac{\lambda}{2}t^2$. \\
	$\rho_j(t)$: $\lambda$-strongly convex.\\
	$\rho_j^*(t) = \frac{t^2}{2}$,\quad
	$\partial\rho_j^*(t) = t$.
\item {\bf Elastic net regularization} ($\lambda>0$, $\kappa>0$: hyperparameters) \\
	$\rho_j(t) = \frac{\lambda}{2}t^2 + \kappa|t|$. \\
	$\rho_j(t)$: $\lambda$-strongly convex.\\
	$\rho_j^*(t) = \frac{1}{2}([|t| - \kappa]_+)^2$.\\
	$\partial\rho_j^*(t) = \mathrm{sign}(t)\max\{|t| - \kappa, 0\}$.
\item {\bf L2-regularization \& Elastic net regularization with intercept} ($\lambda>0$, $\kappa\geq 0$: hyperparameters) \\
	$\rho_j(t) = \begin{cases}
		\frac{\lambda}{2}t^2 + \kappa|t|, & (j\in[d-1]) \\
		0. & (j = d)
		\end{cases}
	$\\
	$\rho_j(t)$: $\lambda$-strongly convex for $j\in[d-1]$, not strongly convex for $j=d$.
	\begin{itemize}
	\item If $j\in[d-1]$, \\
		$\rho_j^*(t) = \frac{1}{2\lambda}([|t| - \kappa]_+)^2$,\\
		$\partial\rho_j^*(t) = \frac{1}{\lambda}\mathrm{sign}(t)\max\{|t| - \kappa, 0\}$.
	\item If $j = d$, \\
		$\rho_j^*(t) = \begin{cases}
			0, & (t = 0) \\
			+\infty, & (t\neq 0)
			\end{cases}
		$\\
		$\partial\rho_j^*(t) = \begin{cases}
			-\infty,            & (t < 0) \\
			[-\infty, +\infty], & (t = 0) \\
			+\infty.            & (t > 0) \\
			\end{cases}
		$
	\end{itemize}
\item {\bf L1-regularization with intercept} ($\lambda>0$: hyperparameter) \\
	$\rho_j(t) = \begin{cases}
		\lambda|t|, & (j\in[d-1]) \\
		0. & (j = d)
		\end{cases}
	$\\
	$\rho_j(t)$: not strongly convex.
	\begin{itemize}
	\item If $j\in[d-1]$, \\
		$\rho_j^*(t) = \begin{cases}
			0, & (|t|\leq\lambda) \\
			+\infty, & (|t| > \lambda)
			\end{cases}
		$\\
		$\partial\rho_j^*(t) = \begin{cases}
			-\infty, & (t < -\lambda) \\
			[-\infty, 0], & (t = -\lambda) \\
			0, & (|t|\leq\kappa) \\
			[0, +\infty], & (t = \lambda) \\
			+\infty. & (t > \lambda)
			\end{cases}
		$
	\item If $j = d$, \\
		$\rho_j^*(t) = \begin{cases}
			0, & (t = 0) \\
			+\infty, & (t\neq 0)
			\end{cases}
		$\\
		$\partial\rho_j^*(t) = \begin{cases}
			-\infty,            & (t < 0) \\
			[-\infty, +\infty], & (t = 0) \\
			+\infty.            & (t > 0) \\
			\end{cases}
		$
	\end{itemize}
\end{itemize}

\section{Proofs}

\subsection{Convex conjugates of $\PrimL{n}$ and $\PrimR$ in Section \ref{sec:dual-problem}} \label{app:primal-part-conjugate}

\begin{align*}
& \PrimL{n}^*(\bm{u}) = \sup_{\bm{u}^\prime\in\mathbb{R}^n}\Bigl\{\bm{u}^\top\bm{u}^\prime - \PrimL{n}(\bm{u}^\prime) \Bigr\}
	= \sup_{\bm{u}^\prime\in\mathbb{R}^n}\Bigl\{\sum_{i\in[n]} \Bigl[ u_i u^\prime_i - \frac{1}{n} \ell_{y_i}(u^\prime_i) \Bigr] \Bigr\} \\
& = \sup_{\bm{u}^\prime\in\mathbb{R}^n}\Bigl\{\frac{1}{n} \sum_{i\in[n]} [ n u_i u^\prime_i - \ell_{y_i}(u^\prime_i) ] \Bigr\}
	= \frac{1}{n} \sum_{i\in[n]} \sup_{u_i^\prime\in\mathbb{R}} \{ n u_i u^\prime_i - \ell_{y_i}(u^\prime_i) \}
	= \frac{1}{n} \sum_{i\in[n]} \ell^*_{y_i}(n u_i). \\
& \PrimR^*(\bm{v}) = \sup_{\bm{v}^\prime\in\mathbb{R}^d} \{ \bm{v}^\top\bm{v}^\prime - PR(\bm{v}) \}
	= \sup_{\bm{v}^\prime\in\mathbb{R}^d}\Bigl\{ \sum_{j\in[d]} [ v_j v^\prime_j - \rho_j(v^\prime_j) ] \Bigr\} \\
& = \sum_{j\in[d]} \sup_{v^\prime_j\in\mathbb{R}} \{ v_j v^\prime_j - \rho_j(v^\prime_j) \}
	= \sum_{j\in[d]} \rho^*_j(v_j).
\end{align*}

\subsection{Proof of Theorem \ref{th:bounds}} \label{app:proof-th:bounds}

\begin{lemma} \label{lm:distance-to-optimum}
For a $\lambda$-strongly convex function $f(\bm{v}): \mathbb{R}^k\to\mathbb{R}\cup\{+\infty\}$,
let $\bm{v}^* := \targmin_{\bm{v}\in\mathrm{dom}f} f(\bm{v})$ be the minimizer of $f$.
Then we have
\begin{align*}
& \|\bm{v} - \bm{v}^*\|_2 \leq \sqrt{\frac{2}{\lambda}[f(\bm{v}) - f(\bm{v}^*)]} \qquad \forall \bm{v}\in\mathbb{R}^k.
\end{align*}
\end{lemma}

\begin{proof}
Because $\bm{v}^*$ is a minimizer of $f$,
by the definition of the subgradient \eqref{eq:subgradient},
the zero vector $\bm{0}$ must be a subgradient of $f$ at $\bm{v}^*$ \cite{rockafellar1970convex}.
Thus we apply \eqref{eq:strongly-convex} as $\bm{g} \gets \bm{0}$ and obtain
\begin{align*}
& f(\bm{v}) - f(\bm{v}^*)
	\geq \bm{0}^\top(\bm{v} - \bm{v}^*) + \frac{\lambda}{2}\|\bm{v} - \bm{v}^*\|_2^2
	\geq \frac{\lambda}{2}\|\bm{v} - \bm{v}^*\|_2^2.
\end{align*}
\end{proof}

\begin{lemma} \label{lm:sphere-to-interval}
For any vectors $\bm{v}, \bm{v}^\prime, \bm{\xi}\in\mathbb{R}^k$ and $R > 0$,
\begin{align*}
& \|\bm{v}^\prime - \bm{v}\| \leq R \\
& \Rightarrow
	\bm{\xi}^\top\bm{v} - R\|\bm{\xi}\|_2 \leq \bm{\xi}^\top\bm{v}^\prime \leq \bm{\xi}^\top\bm{v} + R\|\bm{\xi}\|_2.
\end{align*}
\end{lemma}

\begin{proof}
From the definition of the inner product between $\bm{\xi}$ and $(\bm{v}^\prime - \bm{v})$,
\begin{align*}
& -\|\bm{\xi}\|_2 \|\bm{v}^\prime - \bm{v}\|_2 \leq \bm{\xi}^\top (\bm{v}^\prime - \bm{v}) \leq \|\bm{\xi}\|_2 \|\bm{v}^\prime - \bm{v}\|_2.
\end{align*}
By the assumption we have
\begin{align*}
& -R \|\bm{\xi}\|_2 \leq \bm{\xi}^\top (\bm{v}^\prime - \bm{v}) \leq R \|\bm{\xi}\|_2, \\
& \bm{\xi}^\top\bm{v} - R \|\bm{\xi}\|_2 \leq \bm{\xi}^\top \bm{v}^\prime \leq \bm{\xi}^\top\bm{v} + R \|\bm{\xi}\|_2.
\end{align*}
\end{proof}

\begin{proof}[Proof of Theorem \ref{th:bounds} (i-P) and (i-D)]
\eqref{eq:primal-bound} is proved as follows. Because $\rho_j(t)$ is $\lambda$-strongly convex and $\ell$ is convex, $\Pnew$ is $\lambda$-strongly convex. Thus
\begin{align*}
& \|\Wnew - \hat{\bm{w}}\|_2 \\
& \leq \sqrt{(2/\lambda)[\Pnew(\hat{\bm{w}}) - \Pnew(\Wnew)]} \quad (\because \text{Lemma \ref{lm:distance-to-optimum}})\\
& \leq \sqrt{(2/\lambda)[\Pnew(\hat{\bm{w}}) - \Dnew(\Alphanew)]} \quad (\because \eqref{eq:strong-duality})\\
& \leq \sqrt{(2/\lambda)[\Pnew(\hat{\bm{w}}) - \Dnew(\hat{\bm{\alpha}})]} := r_P. \quad (\because \text{$\Alphanew$ is the maximizer of $\Dnew$})
\end{align*}

\eqref{eq:primal-dual-bound} is proved as follows. First, we evaluate $\Xnew_{i:}\Wnew$ using Lemma \ref{lm:sphere-to-interval} as
\begin{align*}
& \Xnew_{i:}\hat{\bm{w}} - r_P \|\Xnew_{i:}\|_2 \leq \Xnew_{i:}\Wnew \leq \Xnew_{i:}\hat{\bm{w}} + r_P \|\Xnew_{i:}\|_2.
\end{align*}
Moreover, because $\ell_{y_i}$ is assumed to be convex, $\partial \ell_{y_i}$ monotonically increases\footnote{
Here we must carefully define the monotonicity because $\partial \ell_{y_i}$ is a multivalued function. Formally, $\partial \ell_{y_i}$ is defined as monotonically increasing if $p \leq q$ holds for any $s < t$ and for any choice of $p\in\partial \ell_{y_i}(s)$ and $q\in\partial \ell_{y_i}(t)$.
}. Thus
\begin{align*}
\minpartial \ell_{y_i}(\Xnew_{i:}\hat{\bm{w}} - r_P \|\Xnew_{i:}\|_2)
& \leq \minpartial \ell_{y_i}(\Xnew_{i:}\Wnew)
	\leq \maxpartial \ell_{y_i}(\Xnew_{i:}\Wnew) \\
& \leq \maxpartial \ell_{y_i}(\Xnew_{i:}\hat{\bm{w}} + r_P \|\Xnew_{i:}\|_2).
\end{align*}
Finally, we apply the KKT condition \eqref{eq:KKT-primal2dual} to $\Pnew$ and $\Dnew$, that is, $-\alpha_i \in \partial \ell_{y_i}(\Xnew_{i:}\Wnew)$. This proves that \eqref{eq:primal-dual-bound} holds.
\end{proof}

\begin{proof}[Proof of Theorem \ref{th:bounds} (ii-D)]
Because $\ell$ is $\mu$-smooth and $\rho_j$ is convex, $\ell^*$ is $(1/\mu)$-strongly convex and $\Dnew$ is $(1/\nnew\mu)$-strongly convex. Thus, with a similar discussion to proving \eqref{eq:primal-bound} (See ``Proof of Theorem \ref{th:bounds} (i-P) and (i-D)'' above) we have $\|\Alphanew - \hat{\bm{\alpha}}\|_2 \leq r_D := \sqrt{2\nnew\mu[\Pnew(\hat{\bm{w}}) - \Dnew(\hat{\bm{\alpha}})]}$.
\end{proof}

\begin{proof}[Proof of Theorem \ref{th:bounds} (ii-P)]
First we evaluate $\Xnew_{:j}\Alphanew$ using Lemma \ref{lm:sphere-to-interval} as
\begin{align*}
& {\Xnew_{:j}}^\top\hat{\bm{\alpha}} - r_D \|\Xnew_{:j}\|_2 \leq {\Xnew_{:j}}^\top\Alphanew \leq {\Xnew_{:j}}^\top\hat{\bm{\alpha}} + r_D \|\Xnew_{:j}\|_2.
\end{align*}
Moreover, because $\rho_j$ is assumed to be convex, $\rho^*_j$ is also convex and thus $\partial \rho^*_j$ is monotonically increasing. Thus
\begin{align*}
& \minpartial \rho^*_j\Bigl(\frac{1}{\nnew} \Bigl[ {\Xnew_{:j}}^\top\hat{\bm{\alpha}} - r_D \|\Xnew_{:j}\|_2 \Bigr] \Bigr) \\
& \leq \minpartial \rho^*_j\Bigl(\frac{1}{\nnew} {\Xnew_{:j}}^\top\Alphanew \Bigr)
	\leq \maxpartial \rho^*_j\Bigl(\frac{1}{\nnew} {\Xnew_{:j}}^\top\Alphanew \Bigr) \\
& \leq \maxpartial \rho^*_j\Bigl(\frac{1}{\nnew} \Bigl[ {\Xnew_{:j}}^\top\hat{\bm{\alpha}} + r_D \|\Xnew_{:j}\|_2 \Bigr] \Bigr).
\end{align*}
Finally, we apply the KKT condition \eqref{eq:KKT-dual2primal} to $\Pnew$ and $\Dnew$, that is, $\wnew{j}\in\partial\rho_j^*\Bigl(\frac{1}{\nnew}{\Xnew_{:j}}^\top\Alphanew\Bigr)$. This proves case (ii-P).
\end{proof}

\subsection{Calculation of Remark \ref{rm:dual2primal-input-feasible}} \label{app:dual2primal-input-feasible}

In the proof in Appendix \ref{app:proof-th:bounds} for (ii-P) of Theorem \ref{th:bounds},
to calculate the lower and upper bounds of $\wnew{j}$,
we used only Lemma \ref{lm:sphere-to-interval} to calculate $\underline{F}_j(\hat{\bm{\alpha}})$ and $\overline{F}_j(\hat{\bm{\alpha}})$.
However, we can derive tighter bounds of $\wnew{j}$ by additionally assuming that $\hat{\bm{\alpha}}$ is in the domain of $\Dual_{\Xnew}$.

First, we compute $\Dual_{\Xnew}$. From \eqref{eq:conjugate-finite} and \eqref{eq:dual} we have:
\begin{align}
& \mathrm{dom} D_X = \{ \bm{\alpha}\in\mathbb{R}^d \mid \nonumber\\
& \phantom{\mathrm{dom} D_X = \{}
	-\maxranpartial{\ell_{y_i}} \leq \alpha_i \leq -\minranpartial{\ell_{y_i}} ~\forall i\in[n], \label{eq:loss-conj-feasible}\\
& \phantom{\mathrm{dom} D_X = \{}
	\minranpartial{\rho_j} \leq \frac{1}{n}X_{:j}^\top\bm{\alpha} \leq \maxranpartial{\rho_j}~\forall j\in[d]
	\}. \label{eq:reg-conj-feasible}
\end{align}

Then, we can redefine $\underline{F}_j$ and $\overline{F}_j$ as:
\begin{subequations}
\label{eq:dual2primal-input-feasible}
\begin{align}
\underline{F}_j(\hat{\bm{\alpha}}) & := \max\{ \XnewT_{:j}\hat{\bm{\alpha}} - r_D \|\XnewT_{:j}\|_2, \\
	& \phantom{:= \max\{} \min_{\bm{\alpha}\in\mathrm{dom}\Dual_{\Xnew}} \XnewT_{:j}\bm{\alpha} \} \\
& = \max\{ \XnewT_{:j}\hat{\bm{\alpha}} - r_D \|\XnewT_{:j}\|_2, \\
	& \phantom{:= \max\{} \min_{\bm{\alpha}:~\text{satisfying}~\eqref{eq:loss-conj-feasible}} \XnewT_{:j}\bm{\alpha}, \\
	& \phantom{:= \max\{} \min_{\bm{\alpha}:~\text{satisfying}~\eqref{eq:reg-conj-feasible}} \XnewT_{:j}\bm{\alpha} \} \\
& = \max\{ \XnewT_{:j}\hat{\bm{\alpha}} - r_D \|\XnewT_{:j}\|_2, \\
	& \phantom{:= \max\{} \MinLin{\{-\maxranpartial{\ell_{y_i}}\}_i}{\{-\minranpartial{\ell_{y_i}}\}_i}(\Xnew_{:j}), \\
	& \phantom{:= \max\{} \nnew\cdot\minranpartial{\rho_j} \} \\
\overline{F}_j(\hat{\bm{\alpha}})
	& := \min\{ \XnewT_{:j}\hat{\bm{\alpha}} + r_D \|\XnewT_{:j}\|_2, \\
	& \phantom{:= \max\{} \MaxLin{\{-\maxranpartial{\ell_{y_i}}\}_i}{\{-\minranpartial{\ell_{y_i}}\}_i}(\Xnew_{:j}), \\
	& \phantom{:= \max\{} \nnew\cdot\maxranpartial{\rho_j} \}.
\end{align}
\end{subequations}

\section{Proofs of calculations of duality gaps} \label{app:computation-dualitygap}

In this appendix, we present the proofs for the results given in Section \ref{ch:bound-fast}, that is,
the calculation of $\Gnew(\hat{\bm{w}}, \hat{\bm{\alpha}})$
when $\hat{\bm{w}}$ and $\hat{\bm{\alpha}}$ were obtained from $\Wold$ and $\Wnew$, respectively.

\subsection{For instance removals: $\nnew < \nold$}

\begin{align*}
& \nnew\Pnew(\Wold) - \nold\PrimL{\nold}(\Xold\Wold) \\
= & \sum_{i\in[\nnew]} \ell_{y_i}(\Xnew_{i:}\Wold) - \sum_{i\in[\nold]} \ell_{y_i}(\Xold_{i:}\Wold) + \nnew\PrimR(\Wold) \\
= & -\sum_{i=\nnew+1}^{\nold} \ell_{y_i}(\Xold_{i:}\Wold) + \nnew\PrimR(\Wold).
\end{align*}
\begin{align*}
& \nnew\Dnew(\Alphaoldh) + \nold\PrimL{\nold}^*\Bigl( -\frac{1}{\nold}\Alphaold \Bigr) \\
= & -\sum_{i\in[\nnew]} \ell^*_{y_i}(-\alphaoldh{i}) + \sum_{i\in[\nold]} \ell^*_{y_i}(-\alphaold{i}) - \nnew \PrimR^*\Bigl(\frac{1}{\nnew}\XnewT\Alphaoldh\Bigr) \\
= & \sum_{i=\nnew+1}^{\nold} \ell^*_{y_i}(-\alphaold{i}) - \nnew \PrimR^*\Bigl(\frac{1}{\nnew}\Bigl[\XoldT\Alphaold - \sum_{i=\nnew+1}^{\nold}\alphaold{i}\Xold_{i:}\Bigr]\Bigr).
\end{align*}
\begin{align}
\therefore & \Gnew(\Wold, \Alphaoldh) := \Pnew(\Wold) - \Dnew(\Alphaoldh) \nonumber\\
= & \frac{1}{\nnew} \biggl[ -\sum_{i=\nnew+1}^{\nold} [\ell_{y_i}(\Xold_{i:}\Wold) + \ell^*_{y_i}(-\alphaold{i})] \nonumber\\
	& + \nold\PrimL{\nold}(\Xold\Wold) + \nnew\PrimR(\Wold) \nonumber\\
	& + \nold\PrimL{\nold}^*\Bigl( -\frac{1}{\nold}\Alphaold \Bigr) \nonumber\\
	& + \nnew \PrimR^*\Bigl(\frac{1}{\nnew}\Bigl[\XoldT\Alphaold - \sum_{i=\nnew+1}^{\nold}\alphaold{i}\Xold_{i:}\Bigr]\Bigr)
	\biggr].
	\tag{\eqref{eq:bound-instance-removal} restated}
\end{align}

\subsection{For instance additions: $\nnew > \nold$}

\begin{align*}
& \nnew\Pnew(\Wold) - \nold\PrimL{\nold}(\Xold\Wold) \\
= & \sum_{i\in[\nnew]} \ell_{y_i}(\Xnew_{i:}\Wold) - \sum_{i\in[\nold]} \ell_{y_i}(\Xold_{i:}\Wold) + \nnew\PrimR(\Wold) \\
= & \sum_{i=\nold+1}^{\nnew} \ell_{y_i}(\Xnew_{i:}\Wold) + \nnew\PrimR(\Wold).
\end{align*}
\begin{align*}
& \nnew\Dnew(\Alphaoldh) + \nold\PrimL{\nold}^*\Bigl( -\frac{1}{\nold}\Alphaold \Bigr) \\
= & -\sum_{i\in[\nnew]} \ell^*_{y_i}(-\alphaoldh{i}) + \sum_{i\in[\nold]} \ell^*_{y_i}(-\alphaold{i}) - \nnew \PrimR^*\Bigl(\frac{1}{\nnew}\XnewT\Alphaoldh\Bigr) \\
= & -\sum_{i=\nold+1}^{\nnew} \ell^*_{y_i}(-\alphaoldh{i}) - \nnew \PrimR^*\Bigl(\frac{1}{\nnew}\Bigl[\XoldT\Alphaold + \sum_{i=\nold+1}^{\nnew}\alphaoldh{i}\Xnew_{i:}\Bigr]\Bigr).
\end{align*}
\begin{align}
\therefore & \Gnew(\Wold, \Alphaoldh) := \Pnew(\Wold) - \Dnew(\Alphaoldh) \nonumber\\
= & \frac{1}{\nnew} \biggl[ \sum_{i=\nold+1}^{\nnew} [\ell_{y_i}(\Xnew_{i:}\Wold) + \ell^*_{y_i}(-\alphaoldh{i})] \nonumber\\
	& + \nold\PrimL{\nold}(\Xold\Wold) + \nnew\PrimR(\Wold) \nonumber\\
	& + \nold\PrimL{\nold}^*\Bigl( -\frac{1}{\nold}\Alphaold \Bigr) \nonumber\\
	& + \nnew \PrimR^*\Bigl(\frac{1}{\nnew}\Bigl[\XoldT\Alphaold + \sum_{i=\nold+1}^{\nnew}\alphaoldh{i}\Xnew_{i:}\Bigr]\Bigr)
	\biggr].
	\tag{\eqref{eq:bound-instance-addition} restated}
\end{align}

\subsection{For feature removals: $\dnew < \dold$}

\begin{align*}
& \Pnew(\Woldh) - \PrimR(\Wold) \\
= & \PrimL{n}(\Xnew\Woldh) + \sum_{j\in[\dnew]} \rho_j(\woldh{j}) - \sum_{j\in[\dold]} \rho_j(\wold{j}) \\
= & \PrimL{n}\Bigl(\Xold\Wold - \sum_{j = \dnew+1}^{\dold} \wold{j}\Xold_{:j} \Bigr) - \sum_{j = \dnew+1}^{\dold} \rho_j(\wold{j})
\end{align*}
\begin{align*}
& \Dnew(\Alphaold) - \Dold(\Alphaold) \\
= & -\PrimR^*\Bigl( \frac{1}{n}\XnewT\Alphaold \Bigr) - \PrimR^*\Bigl( \frac{1}{n}\XoldT\Alphaold \Bigr) 
= & \sum_{j = \dnew+1}^{\dold}\rho_j^*\Bigl( \frac{1}{n}\XoldT_{:j}\Alphaold \Bigr)
\end{align*}
\begin{align}
\therefore & \Gnew(\Woldh, \Alphaold) := \Pnew(\Woldh) - \Dnew(\Alphaold) \nonumber\\
= & - \sum_{j = \dnew+1}^{\dold} \Bigl[ \rho_j(\wold{j}) + \rho_j^*\Bigl( \frac{1}{n}\XoldT_{:j}\Alphaold \Bigr) \Bigr] \nonumber\\
	& + \PrimL{n}\Bigl(\Xold\Wold - \sum_{j = \dnew+1}^{\dold} \wold{j}\Xold_{:j} \Bigr) + \PrimR(\Wold) - \Dold(\Alphaold)
	\tag{\eqref{eq:bound-feature-removal} restated}
\end{align}

\subsection{For feature additions: $\dnew > \dold$}

\begin{align*}
& \Pnew(\Woldh) - \PrimR(\Wold) \\
= & \PrimL{n}(\Xnew\Woldh) + \sum_{j\in[\dnew]} \rho_j(\woldh{j}) - \sum_{j\in[\dold]} \rho_j(\wold{j}) \\
= & \PrimL{n}\Bigl(\Xold\Wold + \sum_{j = \dold+1}^{\dnew} \woldh{j}\Xnew_{:j} \Bigr) + \sum_{j = \dold+1}^{\dnew} \rho_j(\woldh{j})
\end{align*}
\begin{align*}
& \Dnew(\Alphaold) - \Dold(\Alphaold) \\
= & -\PrimR^*\Bigl( \frac{1}{n}\XnewT\Alphaold \Bigr) + \PrimR^*\Bigl( \frac{1}{n}\XoldT\Alphaold \Bigr) \\
= & -\sum_{j = \dold+1}^{\dnew}\rho_j^*\Bigl( \frac{1}{n}\XnewT_{:j}\Alphaold \Bigr)
\end{align*}
\begin{align}
\therefore & \Gnew(\Woldh, \Alphaold) := \Pnew(\Woldh) - \Dnew(\Alphaold) \nonumber\\
= & \sum_{j = \dold+1}^{\dnew} \Bigl[ \rho_j(\woldh{j}) + \rho_j^*\Bigl( \frac{1}{n}\XnewT_{:j}\Alphaold \Bigr) \Bigr] \nonumber\\
	& + \PrimL{n}\Bigl(\Xold\Wold + \sum_{j = \dold+1}^{\dnew} \woldh{j}\Xnew_{:j} \Bigr) + \PrimR(\Wold) -  \Dold(\Alphaold)
	\tag{\eqref{eq:bound-feature-addition} restated}
\end{align}

\section{GLRU for the implementations of experiments}

In this section, we describe the expressions used for the implementation of the experiment.

For LOOCV (Section \ref{sec:GLRU-LOOCV}), we require GLRU for one-instance removal
\eqref{eq:bound-instance-removal} in Section \ref{sec:bound-inst-change}.
Let $n$ be the number of instances in the entire dataset; therefore $\nold = n$ and $\nnew = n - 1$.
Additionally, let $(\bm{w}^*, \bm{\alpha}^*)$ be the optimal solution for the entire dataset
(see Algorithm \ref{alg:loocv-GLRU}).
For the index of the removed instance $i\in[n]$,
\eqref{eq:bound-instance-removal} is computed as:
\begin{align}
 & \Gnew(\bm{w}^*, \hat{\bm{\alpha}}^*) \nonumber\\
= & \frac{1}{n-1} \biggl[ - \ell_{y_i}(X_{i:}\bm{w}^*) - \ell^*_{y_i}(-\alpha_i) \nonumber\\
	& + n \PrimL{n}(X \bm{w}^*) + (n-1)\PrimR(\bm{w}^*) + n\PrimL{n}^*\Bigl( -\frac{1}{n}\bm{\alpha}^* \Bigr) \nonumber\\
	& + (n-1)\PrimR^*\Bigl(\frac{1}{n-1}\Bigl[X^\top\bm{\alpha}^* - \alpha_i X_{i:}\Bigr]\Bigr)
	\biggr].
	\label{eq:bound-1instance-removal}
\end{align}

For stepwise feature elimination (Section \ref{sec:GLRU-stepwise}),
we require GLRU for one-feature removal \eqref{eq:bound-feature-removal} in Section \ref{sec:bound-feat-change}.
For the set of currently remaining features $S\subseteq[d]$,
let $(\bm{w}^{*(S)}, \bm{\alpha}^{*(S)})$ be the optimal solution trained from $(X_{:S}, \bm{y})$
(see Algorithm \ref{alg:stepwise-GLRU}).
Given the index of the removed feature $j\in S$,
\eqref{eq:bound-feature-removal} is computed as:
\begin{align}
 & \Gnew(\hat{\bm{w}}^{*(S)}, \bm{\alpha}^{*(S)}) \nonumber\\
= & - \Bigl[ \rho(w^{*(S)}_j) + \rho^*\Bigl( \frac{1}{n}X_{:j}^\top\bm{\alpha}^{*(S)} \Bigr) \Bigr] \nonumber\\
	& + \PrimL{n}\Bigl(X_{:S}\bm{w}^{*(S)} - w^{*(S)}_j X_{:j} \Bigr) + \PrimR(\bm{w}^{*(S)}) - \Dual_{X_{:S}}(\bm{\alpha}^{*(S)}).
	\label{eq:bound-1feature-removal}
\end{align}

In addition, because we use $L_2$-regularization $\rho(t) := (\lambda/2)t^2$ in the experiment,
we can compute $\PrimR(\bm{v}) = (\lambda/2)\|\bm{v}\|_2^2$ and
$\PrimR^*(\bm{v}) = (1/2\lambda)\|\bm{v}\|_2^2$.
Thus, \eqref{eq:bound-1instance-removal} and \eqref{eq:bound-1feature-removal} are
respectively computed as follows:
\begin{align}
 & \Gnew(\bm{w}^*, \hat{\bm{\alpha}}^*) \quad\text{(for \eqref{eq:bound-1instance-removal})} \nonumber\\
= & \frac{1}{n-1} \biggl[ - \ell_{y_i}(X_{i:}\bm{w}^*) - \ell^*_{y_i}(-\alpha_i) \nonumber\\
	& + n \PrimL{n}(X \bm{w}^*) + (n-1)\PrimR(\bm{w}^*) + n\PrimL{n}^*\Bigl( -\frac{1}{n}\bm{\alpha}^* \Bigr) \nonumber\\
	& + \frac{1}{2(n-1)\lambda}\Bigl\| X^\top\bm{\alpha}^* - \alpha_i X_{i:} \Bigr\|_2^2
	\biggr].
	\label{eq:bound-1instance-removal-l2} \\
 & \Gnew(\hat{\bm{w}}^{*(S)}, \bm{\alpha}^{*(S)}) \quad\text{(for \eqref{eq:bound-1feature-removal})} \nonumber\\
= & - \Bigl[ \frac{\lambda}{2}(w^{*(S)}_j)^2 + \frac{1}{2 n^2 \lambda} \Bigl( X_{:j}^\top\bm{\alpha}^{*(S)} \Bigr)^2 \Bigr] \nonumber\\
	& + \PrimL{n}\Bigl(X_{:S}\bm{w}^{*(S)} - w^{*(S)}_j X_{:j} \Bigr) + \PrimR(\bm{w}^{*(S)}) - \Dual_{X_{:S}}(\bm{\alpha}^{*(S)}).
	\label{eq:bound-1feature-removal-l2}
\end{align}

\section{Detailed experimental setups}

In this appendix, we present the detailed setups of the experiments in Section \ref{sec:experiment}.

\subsection{(Exact) training computation} \label{app:exp-training}

In the training procedure (computing $\bm{w}^*$ in \eqref{eq:primal}), we employed the trust region Newton method implemented in LIBLINEAR (named ``TRON'') \cite{Fan08b}. The trust region Newton method runs Newton methods in a limited domain for multiple times to achieve faster convergence.
In our implementation, the convergence criterion was different from that in LIBLINEAR. We stop the training procedure if the ``relative'' duality gap $[\Prim_X(\bm{w}) - \Dual_X(\bm{\alpha})]/\Prim_X(\bm{w})$ becomes smaller than threshold $10^{-6}$.

\subsection{Approximate training computation} \label{app:exp-training-approx}

Following \cite{rad2020scalable}, we implemented an approximate training computation for LOOCV by applying the Newton's method for only one step as follows.

Suppose the same notations as Section \ref{sec:GLRU-LOOCV}, and suppose that $\bm{w}^*$ (the model parameter trained with all instances) is already computed.
Then $\bm{w}^{*(-i)}$ can be approximated as:
\begin{align}
\bm{w}^{*(-i)} \approx \bm{w}^* - \Bigl.\Bigl\{\Bigl[\frac{\partial^2}{\partial \bm w^2}\Prim_{X^{(-i)}}(\bm w)\Bigr]^{-1} \frac{\partial}{\partial \bm w}\Prim_{X^{(-i)}}(\bm w)\Bigr\}\Bigr|_{\bm w = \bm{w}^*}. \label{eq:approx-new}
\end{align}
Then, the gradient in \eqref{eq:approx-new} can be computed as follows:
\begin{align}
& \frac{\partial}{\partial \bm w}\Prim_{X^{(-i)}}(\bm w)
	= \frac{1}{n-1}\sum_{k\in[n]\setminus\{i\}} X^\top_{k:} \ell_{y_k}(X_{k:}\bm w) + \lambda \bm w \nonumber\\
& = \frac{n}{n-1}\frac{\partial}{\partial \bm w}\Prim_X(\bm w) - \frac{\ell^\prime_{y_i}(X_{i:}\bm w)}{n-1}X^\top_{i:} - \frac{\lambda}{n-1} \bm w. \label{eq:new-gradient} \\
& \therefore\quad \Bigl. \frac{\partial}{\partial \bm w}\Prim_{X^{(-i)}}(\bm w)\Bigr|_{\bm w = \bm{w}^*}
	= - \frac{1}{n-1}\Bigl[ X^\top_{i:} \ell^\prime_{y_i}(X_{i:}\bm{w}^*) + \lambda \bm{w}^* \Bigr].
\end{align}
Moreover, the hessian in \eqref{eq:approx-new} can be computed as follows:
\begin{align*}
& \Bigl. \frac{\partial^2}{\partial \bm w^2}\Prim_{X^{(-i)}}(\bm w) \Bigr|_{\bm w = \Wold} = \tilde{H} - s X^\top_{i:}X_{i:}, \quad\text{where}\\
& \tilde{H} := \frac{n}{n-1}\Bigl.\Bigl\{\frac{\partial^2}{\partial \bm w^2}\Pold(\bm w)\Bigr\}\Bigr|_{\bm w = \Wold} - \frac{\lambda}{n-1} I ~~(\in\mathbb{R}^{d\times d}), \\
& s := \frac{\ell^{\prime\prime}_{y_i}(X_{i:}\Wold)}{n-1} ~~(\in\mathbb{R}).
\end{align*}
So, the inverse of the hessian can be computed by Sherman-Morrison's formula as:
\begin{align}
\left[\tilde{H} - s X^\top_{i:}X_{i:} \right]^{-1} = \tilde{H}^{-1} + \frac{s \tilde{H}^{-1} X^\top_{i:}X_{i:} \tilde{H}^{-1}}{1 - s X_{i:} \tilde{H}^{-1} X^\top_{i:}}. \label{eq:inverse-hessian-loocv}
\end{align}
This implies that, if we compute $\tilde{H}^{-1}$ beforehand with $O(d^3)$ time, we can compute \eqref{eq:inverse-hessian-loocv} for each instance removal in LOOCV with a relatively small computational cost of $O(d^2)$ time.

In the implementation, for the computation of the inverse we used the method {\tt pseudoInverse} of the class {\tt CompleteOrthogonalDecomposition} in {\em Eigen} 3.4.0 (\url{https://eigen.tuxfamily.org/}), a matrix library for C++.

\subsection{Data preprocesses and normalizations} \label{app:exp-normalization}

First, for each dataset, we removed all features with unique value. Specifically, the datasets {\tt dexter} and {\tt dorothea} originally had 20,000 and 100,000 features, respectively, but as a result of the removal, they had 11,035 and 91,598 features, respectively.

In addition, we conducted data normalizations for the benchmark datasets in Tables \ref{tb:datasets-loocv-stepwise} and \ref{tab:datasets}.
To normalize a benchmark dataset, we used one of the following schemes. Let $Z\in\mathbb{R}^{n\times d}$ be the original data values and $X\in\mathbb{R}^{n\times d}$ be the normalized values. In addition, let $\mathbb{E}[\cdot]$ and $\mathbb{V}[\cdot]$ be the mean and the variance\footnote{Not the ``unbiased variance'' but the ``sample variance''; we define $\mathbb{V}[\bm{u}]$ ($\bm{u}\in\mathbb{R}^n$) as $(1/n)\sum_{i\in[n]}(u_i - \mathbb{E}(\bm{u}))^2$.}.
\begin{description}
\item[Strategy ``Dense'' (for dense or small datasets)] ~\\
	For each feature, we conducted a linear transformation to make its mean zero and variance one. The values in $Z$ are normalized by $x_{ij} = (z_{ij} - \mathbb{E}[Z_{:j}])/\sqrt{\mathbb{V}[Z_{:j}]}$.
\item[Strategy ``Sparse'' (For sparse and large datasets)] ~\\
	For each feature, we scaled (multiplication, no addition) to obtain the L2-norm $\|X_{:j}\|_2 = \sqrt{n}$. We used this to maintain the sparsity of the original dataset. We normalized to $\sqrt{n}$ since the value is also $\sqrt{n}$ by the ``Dense'' strategy above. The values in $Z$ are normalized by $x_{ij} = \sqrt{n}\cdot z_{ij}/\|Z_{:j}\|_2$.
\end{description}

\section{Size of Bounds by Primal-SCB and Dual-SCB} \label{app:exp-tightness-discussion-dual}

In this appendix, we show detailed calculations for the theoretical results in Section \ref{sec:exp-tightness-discussion}.

The size of the bounds of the prediction result $\XtestT\Wnew$ for Dual-SCB \eqref{eq:bound-size-dual-SCB-norm},
that is, the difference of the two bounds in \eqref{eq:prediction-bound-dual},
is computed as follows.

First, since $\rho(t) = (\lambda/2)t^2$ in the setup of Section \ref{sec:exp-tightness},
$\rho^*(t) := (1/2\lambda)t^2$ and $\partial\rho^*(t) = \{(1/\lambda)t\}$.
Thus, the difference between the two bounds in \eqref{eq:prediction-bound-dual} is calculated as
\begin{align}
	& \frac{1}{\lambda\nnew}\Biggl[ \sum_{j\in[\dnew],~\xtest{j} < 0} \xtest{j} (\underline{F}_j(\hat{\bm{\alpha}}) - \overline{F}_j(\hat{\bm{\alpha}}) )
	+ \sum_{j\in[\dnew],~\xtest{j} > 0} \xtest{j} (\overline{F}_j(\hat{\bm{\alpha}}) - \underline{F}_j(\hat{\bm{\alpha}}) ) \Biggr] \nonumber\\
= & \frac{2}{\lambda\nnew}\Biggl[ \sum_{j\in[\dnew],~\xtest{j} < 0} \xtest{j} (-r_D \|\XnewT_{:j}\|_2)
	+ \sum_{j\in[\dnew],~\xtest{j} > 0} \xtest{j} (r_D \|\XnewT_{:j}\|_2) \Biggr] \nonumber\\
= & \frac{2 r_D}{\lambda\nnew} \sum_{j\in[\dnew]} |\xtest{j}| \|\XnewT_{:j}\|_2 \nonumber\\
= & \sqrt{\frac{8\mu}{\lambda^2\nnew} \Gnew(\hat{\bm{w}}, \hat{\bm{\alpha}})} \sum_{j\in[\dnew]} |\xtest{j}| \|\XnewT_{:j}\|_2,
	\label{eq:bound-size-dual-SCB}
\end{align}
where $\underline{F}_j$ and $\overline{F}_j$ in Theorem \ref{th:bounds} are defined not by \eqref{eq:dual2primal-input-feasible} but by \eqref{eq:dual2primal-input-naive}.

In addition, because we assume that the dataset is normalized as
$\|X_{:j}\|_2 = \sqrt{n}$ (Section \ref{sec:exp-tightness-setup}),
we approximate $\|\Xnew_{:j}\|_2 \approx \sqrt{\nnew}$.
Thus, \eqref{eq:bound-size-dual-SCB} can be rewritten as
\begin{align}
\sqrt{\frac{8\mu}{\lambda^2} \Gnew(\hat{\bm{w}}, \hat{\bm{\alpha}})} \|\Xtest\|_1.
	\tag{\eqref{eq:bound-size-dual-SCB-norm} restated}
\end{align}

\begin{remark}
The result \eqref{eq:bound-size-dual-SCB} seemingly implies that we can make the size of bounds by Dual-SCB tighter compared to Primal-SCB by scaling down $X$.
However, it is not true because, when $X$ is scaled down, we should scale down $\lambda$ accordingly.

If we scale down $X$ and $\Xtest$ to $\eta X$ and $\eta \Xtest$, respectively ($0 < \eta < 1$), 
then the size of the bounds by Dual-SCB is scaled down by $\eta^2$ times while Primal-SCB by $\eta$ times.
This is because the size of the bounds by Dual-SCB \eqref{eq:bound-size-dual-SCB} has the term $|\xtest{j}|\|\Xnew_{:j}\|_2$,
whereas that of Primal-SCB \eqref{eq:bound-size-primal-SCB} has only $\|\Xtest\|_2$.

However, if $X$ is scaled down to $\eta X$, then $\lambda$ should be scaled down to $\eta^2 \lambda$ to obtain the same training result in the sense that the prediction result remains the same.
Given a test instance $\Xtest$ and training result $\bm{w}^*$, we make a prediction using $\XtestT\bm{w}^*$: Thus, if $\Xtest$ is scaled down to $\eta\Xtest$, $\bm{w}^*$ must be scaled up to $\bm{w}^*/\eta$. As a result, to maintain the regularization term $(\lambda/2)\|\bm{w}^*\|_2^2$, we scale down $\lambda$ to $\eta^2 \lambda$. Therefore, the result of \eqref{eq:bound-size-dual-SCB} and \eqref{eq:bound-size-primal-SCB} are consistent for the scaling of $X$ as long as we scale $\lambda$ accordingly.
\end{remark}

\end{document}